%% file: tree-structured_RL_main.tex
\documentclass[nohyperref]{article}

\usepackage{microtype}
\usepackage{graphicx}
\usepackage{subfigure}
\usepackage{booktabs} 

\usepackage{hyperref}

\usepackage[accepted]{icml2022}

\usepackage{amsmath}
\usepackage{amssymb}
\usepackage{mathtools}
\usepackage{amsthm}

\usepackage[capitalize,noabbrev]{cleveref}

\theoremstyle{plain}
\newtheorem{theorem}{Theorem} 

\newtheorem{lemma}[theorem]{Lemma}

\theoremstyle{definition}

\newtheorem{assumption}[theorem]{Assumption}
\theoremstyle{remark}

\usepackage[textsize=tiny]{todonotes}

\usepackage{appendix}
\usepackage{titletoc}
\usepackage{braket}
\usepackage{ifthen}
\usepackage{tikz}
\usetikzlibrary{positioning,chains,fit,shapes,calc}
\usepackage{eqparbox,array}
\usepackage{amsmath}
\usepackage{amssymb}
\usepackage{amsfonts}
\usepackage{enumerate}
\usepackage{flushend}
\usepackage{mathrsfs}
\usepackage{subfigure}
\usepackage{booktabs}
\usepackage{makecell}
\usepackage{setspace}
\usepackage{xcolor}
\usepackage{wrapfig}
\usepackage{bm}
\usepackage{multicol}
\usepackage{bbm}

\newcommand{\A}{\mathbb{A}}

\newcommand{\R}{\mathbb{R}}
\newcommand{\N}{\mathbb{N}}

\newcommand{\bw}{\boldsymbol{w}}

\newcommand{\cA}{\mathcal{A}}

\newcommand{\cE}{\mathcal{E}}

\newcommand{\cG}{\mathcal{G}}

\newcommand{\cI}{\mathcal{I}}

\newcommand{\cM}{\mathcal{M}}

\newcommand{\cS}{\mathcal{S}}

\mathchardef\mhyphen="2D
\newcommand{\argmax}{\operatornamewithlimits{argmax}}

\newcommand{\ex}{\mathbb{E}}

\newcommand{\var}{\textup{Var}}

\newcommand{\kl}{\textup{KL}}
\newcommand{\sbr}[1]{\left( #1 \right)}
\newcommand{\mbr}[1]{\left[ #1 \right]}
\newcommand{\lbr}[1]{\left\{ #1 \right\}}
\newcommand{\abr}[1]{\left| #1 \right|}
\newcommand{\term}[1]{\mathtt{Term\ #1}}
\newcommand{\regret}{\mathtt{Regret}}
\newcommand{\reward}{\mathtt{Reward}}
\newcommand{\univ}{\mathtt{univ}}
\newcommand{\reg}{\mathtt{reg}}
\newcommand{\sumtext}{\mathtt{sum}}
\newcommand{\unif}{\mathtt{unif}}
\newcommand{\indicator}[1]{\mathbbm{1}\lbr{ #1 }}
\newcommand{\binomial}{\mathtt{Binomial}}
\newcommand{\bernoulli}{\mathtt{Bernoulli}}
\newcommand{\aug}{\mathtt{aug}}
\newcommand{\algtteuler}{\mathtt{BranchVI}}
\newcommand{\algbonusleader}{\mathtt{BranchRFE}}

\newcommand{\compilehidecomments}{false}
\ifthenelse{ \equal{\compilehidecomments}{true} }{%
	\newcommand{\wei}[1]{}
	\newcommand{\longbo}[1]{}
	\newcommand{\yihan}[1]{}
}{
	\newcommand{\wei}[1]{{\color{red}  [\text{Wei:} #1]}}
	\newcommand{\longbo}[1]{{\color{blue} [\text{Longbo:} #1]}}
	\newcommand{\yihan}[1]{{\color{teal} [\text{Yihan:} #1]}}
}

\newcommand{\compilefullversion}{true} 
\ifthenelse{\equal{\compilefullversion}{false}}{%
	\newcommand{\OnlyInFull}[1]{}
	\newcommand{\OnlyInShort}[1]{#1}
}{%
	\newcommand{\OnlyInFull}[1]{#1}%
	\newcommand{\OnlyInShort}[1]{}%
}%

\allowdisplaybreaks
\icmltitlerunning{Branching Reinforcement Learning}

\begin{document}

\twocolumn[
\icmltitle{Branching Reinforcement Learning}



\icmlsetsymbol{equal}{*}

\begin{icmlauthorlist}
\icmlauthor{Yihan Du}{thu}
\icmlauthor{Wei Chen}{msra}
\end{icmlauthorlist}

\icmlaffiliation{thu}{IIIS, Tsinghua University, Beijing, China}
\icmlaffiliation{msra}{Microsoft Research}

\icmlcorrespondingauthor{Yihan Du}{duyh18@mails.tsinghua.edu.cn}
\icmlcorrespondingauthor{Wei Chen}{weic@microsoft.com}

\icmlkeywords{Branching reinforcement learning, branching Markov decision process, tree-structured trajectory}

\vskip 0.3in
]



\printAffiliationsAndNotice{}  

\begin{abstract}
In this paper, we propose a novel Branching Reinforcement Learning (Branching RL) model, and investigate both Regret Minimization (RM) and Reward-Free Exploration (RFE) metrics for this model.
Unlike standard RL where the trajectory of each episode is a single $H$-step path, branching RL allows an agent to take multiple base actions in a state such that transitions branch out to multiple successor states correspondingly, and thus it generates a tree-structured trajectory.
This model finds important applications in hierarchical recommendation systems and online advertising.
For branching RL, we establish new Bellman equations and key lemmas, i.e., branching value difference lemma and branching law of total variance, and also bound the total variance by only $O(H^2)$ under an exponentially-large trajectory.
For RM and RFE metrics, we propose computationally efficient algorithms $\algtteuler$ and $\algbonusleader$, respectively, and derive nearly matching upper and lower bounds. 
Our regret and sample complexity results are polynomial in all problem parameters despite exponentially-large trajectories.
\end{abstract}

\section{Introduction}
Reinforcement Learning (RL)~\cite{burnetas1997optimal,sutton2018reinforcement} models a fundamental sequential decision making problem, where an agent interacts with the environment over time in order to maximize the obtained rewards. 
Standard RL~\cite{jaksch2010near,agrawal2017posterior,azar2017minimax,jin2018q,zanette2019tighter} considers taking only a single action in a state and formulates a single $H$-step path model. 
However, in many real-world applications such as recommendation systems~\cite{drl_recommendation} and online advertising~\cite{tree_based_recommendation}, we often need to select multiple options at a time, and each option can trigger a corresponding successor state. For example, in category-based shopping recommendation~\cite{drl_recommendation}, the recommendation system often displays a list of main categories at the first step, where each one has a probability to be clicked. If a main category is clicked,  at the second step, the system further provides a list of sub-categories according to the clicked main category. By analogy, at the last step, the system provides a list of items according to the chosen category path. In this process, users can select (trigger) more than one category-item paths, e.g., one may buy IT accessories-printers-laser printers and IT accessories-scanners-document scanners at once.

To handle such scenarios involving multiple actions and successor states, we propose a novel Branching Reinforcement Learning (Branching RL) framework, which is an episodic tree-structured forward model.
In each episode, an agent starts from an initial state and takes a \emph{super action} that contains multiple \emph{base actions}, where each base action in this state has a probability to be triggered. 
For each state-base action pair, if triggered successfully, the agent receives a reward and transitions to a next state; 
Otherwise, if it is not triggered, the agent receives zero reward and transitions to an absorbing state associated with zero reward. 
Thus, the transitions branch out to multiple successor states. At the second step, for each branched-out state, the agent also selects a super action that contains multiple base actions with trigger probabilities. 
She only obtains rewards from the triggered state-base action pairs, and each state-base action pair transitions to a corresponding next state. 
Then, the transitions at the second step branch out to more successor states. 
By analogy, till the last step, she traverses an $H$-layer tree-structured trajectory, and only collects rewards at the triggered state-base action pairs. 

Different from standard episodic RL~\cite{azar2017minimax,jin2018q,zanette2019tighter} where each trajectory is a single $H$-step path, the trajectory of branching RL is an $H$-layer triggered tree with exponentially increasing states and actions in each layer. 
This model allows an agent to take multiple base actions at once and handle multiple successor states. It can be applied to many hierarchical decision making scenarios, such as category-based recommendation systems~\cite{drl_recommendation} and online advertising~\cite{tree_based_recommendation}.

Under the branching RL model, we investigate two popular metrics in the RL literature, i.e., Regret Minimization (RM) and Reward-Free Exploration (RFE). In regret minimization~\cite{jaksch2010near,azar2017minimax,zanette2019tighter}, the agent aims to minimize the gap between the obtained reward and the reward that can obtained by always taking the optimal policy. In reward-free exploration~\cite{jin2020reward,kaufmann2021adaptive,menard2021fast}, the agent explores the unknown environment (model) without observation of rewards, in order to estimate the model accurately such that for any given reward function, she can plan a near-optimal policy using the estimated model.
The performance in RFE is measured by the number of episodes used during exploration (i.e., sample complexity).

Our work faces several unique challenges: (i) Since branching RL is a tree-structured forward model which greatly differs from standard RL,  existing analytical tools for standard RL, e.g., Bellman equations, value difference lemma and law of total variance, cannot be directly applied to our problem. (ii) With exponentially-large trajectories, it is challenging to analyze the total variance and derive tight (polynomial) regret and sample complexity guarantees. (iii) Since the number of possible super actions can be combinatorially large, how to design a computationally efficient algorithm that avoids naive enumeration over all super actions is another challenge.

To tackle the above challenges, we establish novel analytical tools, including branching Bellman equations, branching value difference lemma and branching law of total variance, and bound the total variance by only $O(H^2)$ under exponentially-large trajectories. 
We also propose computationally efficient algorithms for both RM and RFE metrics, and provide nearly matching upper and lower bounds, which are polynomial in all problem parameters despite exponentially-large trajectories. 

To sum up, our contributions in this paper are as follows:

\begin{itemize}
	\item We propose a novel Branching Reinforcement Learning (Branching RL) framework, which is an episodic $H$-layer tree-structured forward model and finds important applications in hierarchical recommendation systems and online advertising.  Under branching RL, we investigate two popular metrics, i.e., Regret Minimization (RM) and  Reward-Free Exploration (RFE).
	\item We establish new techniques for branching RL, including branching Bellman equations, branching value difference lemma and  branching law of total variance, and bound the total variance by only $O(H^2)$ despite exponentially-large trajectories.
	\item For both RM and RFE metrics, we design computationally efficient algorithms $\algtteuler$ and $\algbonusleader$, respectively, and build near-optimal upper and lower bounds, which are polynomial in all problem parameters even with exponentially-large trajectories. When our problem reduces to standard RL, our results match the state-of-the-arts.
\end{itemize}

Due to space limit, we defer all proofs to Appendix.

\section{Related Work}

Below we review the literature of standard (episodic and tabular) RL with regret minimization (RM) and reward-free exploration (RFE) metrics.

\textbf{Standard RL-RM.}
For the regret minimization (RM) metric, \citet{jaksch2010near} propose an algorithm that adds optimistic bonuses on transition probabilities, and achieves a regret bound with a gap in factors $H,S$ compared to the lower bound~\cite{jaksch2010near,osband2016lower}. Here $H$ is the length of an episode, and $S$ is the number of states.
\citet{agrawal2017posterior} use posterior sampling and obtain an improved regret bound. 
\citet{azar2017minimax} build confidence intervals directly for value functions rather than transition probabilities, and provide the first optimal regret.
\citet{zanette2019tighter} design an algorithm based on both optimistic and pessimistic value functions, and achieve a tighter problem-dependent regret bounds without requiring domain knowledge.
The above works focus on model-based RL algorithms. There are also other works~\cite{jin2018q,zhang2020almost} studying model-free algorithms based on Q-learning with exploration bonus or advantage functions.

\textbf{Standard RL-RFE.} \citet{jin2020reward} introduce the reward-free-exploration (RFE) metric and design an algorithm that runs multiple instances of existing RM algorithm~\cite{zanette2019tighter}, and their sample complexity has a gap to the lower bound~\cite{jin2020reward,domingues2021episodic} in factors $H,S$.
\citet{kaufmann2021adaptive} propose an algorithm which builds upper confidence bounds for the estimation error of value functions, and improve the sample complexity of \cite{jin2020reward}. \citet{menard2021fast} achieve a near-optimal sample complexity by applying an empirical Bernstein inequality and upper bounding the overall estimation error.

There are huge differences between standard RL and our branching RL. 
The exponentially-large trajectory of branching RL brings unique challenges in developing Bellman equations and key lemmas, designing computationally efficient algorithms and deriving optimal (polynomial) bounds. Existing RL algorithms and analysis cannot be applied to solve our challenges.

\section{Problem Formulation}
In this section, we present the formal formulation of Branching Reinforcement Learning (Branching RL).

\textbf{Branching Markov Decision Process (Branching MDP).}
We consider an episodic branching MDP defined by a tuple $\cM=(\cS,A^{\univ},\cA,m,H,q,p,r)$. Here $\cS=\cS^{\reg}\cup \{s_{\perp}\}$ is the state space with cardinality $S$. $\cS^{\reg}$ is the set of \emph{regular states}, and $s_{\perp}$ is an \emph{ending state}, which is an absorbing state with zero reward. 
$A^{\univ}$ is the set of \emph{base actions}, which represents the set of all feasible items in recommendation. Let $N:=|A^{\univ}|$ denote the number of base actions.
A \emph{super action} $A \subset A^{\univ}$ consists of $m$ ($m \leq N$) base actions, which stands for a recommended list.
$\cA$ is the collection of all feasible \emph{super actions} and can be combinatorially large. 
$H$ is the length of an episode.
Throughout the paper, we call a super action an action for short, call $(s,a) \in \cS \times A^{\univ}$ a \emph{state-base action pair}, and call  $(s,a) \in \cS \setminus \{s_{\perp}\} \times A^{\univ}$ a \emph{regular state-base action pair}.

$q(s,a)$ is the trigger probability of state-base action pair $(s,a) \in \cS \times A^{\univ}$. 
$p(s'|s,a)$ is the probability of transitioning to state $s'$ on state-base action pair $(s,a)$, for any $(s',s,a) \in \cS \times \cS \times A^{\univ}$.
$r(s,a) \in [0,1]$ is the reward of pair $(s,a) \in \cS \times A^{\univ}$.
We assume that reward function $r$ is deterministic as many prior RL works~\cite{azar2017minimax,jin2018q,zhang2020almost}, and our work can generalize to stochastic rewards easily. 
Parameters $q,p,r$ are time-homogeneous, i.e., have the same definitions for different step $h \in [H]$. 
The ending state $s_{\perp}$ has zero reward and always transitions back to itself, i.e., $q(s_{\perp},a)=0$, $p(s_{\perp}|s_{\perp},a)=1$ and $r(s_{\perp},a)=0$ for all $a \in A^{\univ}$.
We define a policy $\pi$ as a collection of $H$ functions $\{\pi_h\!:\!\cS \!\mapsto\! \cA\}_{h \in [H]}$, and $K$ as the number of episodes.

\textbf{String-based Notations.}
As shown in Figure~\ref{fig:illustrating_example}, in branching RL, the trajectory of each episode is an $m$-ary tree, where there are $H$ layers (steps), and each layer $h \in [H]$ has $m^{h-1}$ states (nodes) and $m^{h}$ state-base action pairs (edges).
We use the following string-based notations to denote a trajectory:
Each tree node in layer $h$ has a string index $\langle i_1, \ldots, i_{h-1} \rangle$, with the root node for layer $1$ having the empty string $\emptyset$,
and $i_1, \ldots, i_{h-1} \in \{1, 2, \ldots, m\}$.
The $m$ children of this node have indices that concatenate $i_h \in [m]$ to the string, making it  $\langle i_1, \ldots, i_{h-1}, i_h \rangle$, where $i_h$ stands for that
this node is the $i_h$-th child of the node $\langle i_1, \ldots, i_{h-1} \rangle$.
Operator $\oplus$ is the concatenation operation for strings, and $i^{\oplus h}$ denotes the concatenation of $h$ strings $\Braket{i}$ for any $i\in [m], h \in [H]$. For any string $\sigma$, $|\sigma|$ denotes its length, and thus state $s_{\sigma}$ is at step $|\sigma|+1$.

\textbf{Online Game.}
In each episode $k \in [K]$, an agent selects a policy $\pi^k$ at the beginning, and starts from an initial state $s_{\emptyset}$. 
At step $1$, she chooses an action $A_{\emptyset}=\{a_{\Braket{1}},\dots,a_{\Braket{m}}\}$ according to $\pi^k_1$. 
Each state-base action pair $(s_{\emptyset},a_{\Braket{i}})$ for $i \in [m]$ has probability $q(s_{\emptyset},a_{\Braket{i}})$ to be triggered. If triggered successfully, the agent obtains reward $r(s_{\emptyset},a_{\Braket{i}})$ and this state-base action pair transitions to a next state $s_{\Braket{i}} \sim p(\cdot|s_{\emptyset},a_{\Braket{i}})$; 
Otherwise, if not triggered successfully, she obtains zero reward and this state-base action pair transitions to the ending state $s_{\perp}$. Hence, the transitions at step $1$ branch out to $m$ successor states $s_{\Braket{1}},\dots,s_{\Braket{m}}$.
At step $2$, for each state $s_{\Braket{i}}$ ($i \in [m]$), she chooses an action $A_{\Braket{i}}=\{a_{\Braket{i,1}},\dots,a_{\Braket{i,m}} \}$ according to $\pi^k_2$. 
Then, there are $m^2$ state-base action pairs $\{(s_{\Braket{i}},a_{\Braket{i,j}})\}_{i,j \in [m]}$ at step $2$, and each of them
is triggered with probability $q(s_{\Braket{i}},a_{\Braket{i,j}})$. If triggered successfully, the agent receives reward $r(s_{\Braket{i}},a_{\Braket{i,j}})$ and this pair transitions to a next state $s_{\Braket{i,j}} \sim p(\cdot|s_{\Braket{i}},a_{\Braket{i,j}})$; Otherwise, she receives zero reward and this pair transitions to $s_{\perp}$. Then, the transitions at step 2 branch out to $m^2$ successor states $\{s_{\Braket{i,j}}\}_{i,j\in [m]}$. The episode proceeds by analogy at the following steps $3,\dots,H$.
In the trajectory tree, once the agent reaches $s_{\perp}$ at some node, she obtains no reward throughout this branch (This is so-called ``ending state'').

\begin{figure}[t] 
	\centering    
	\includegraphics[width=0.99\columnwidth]{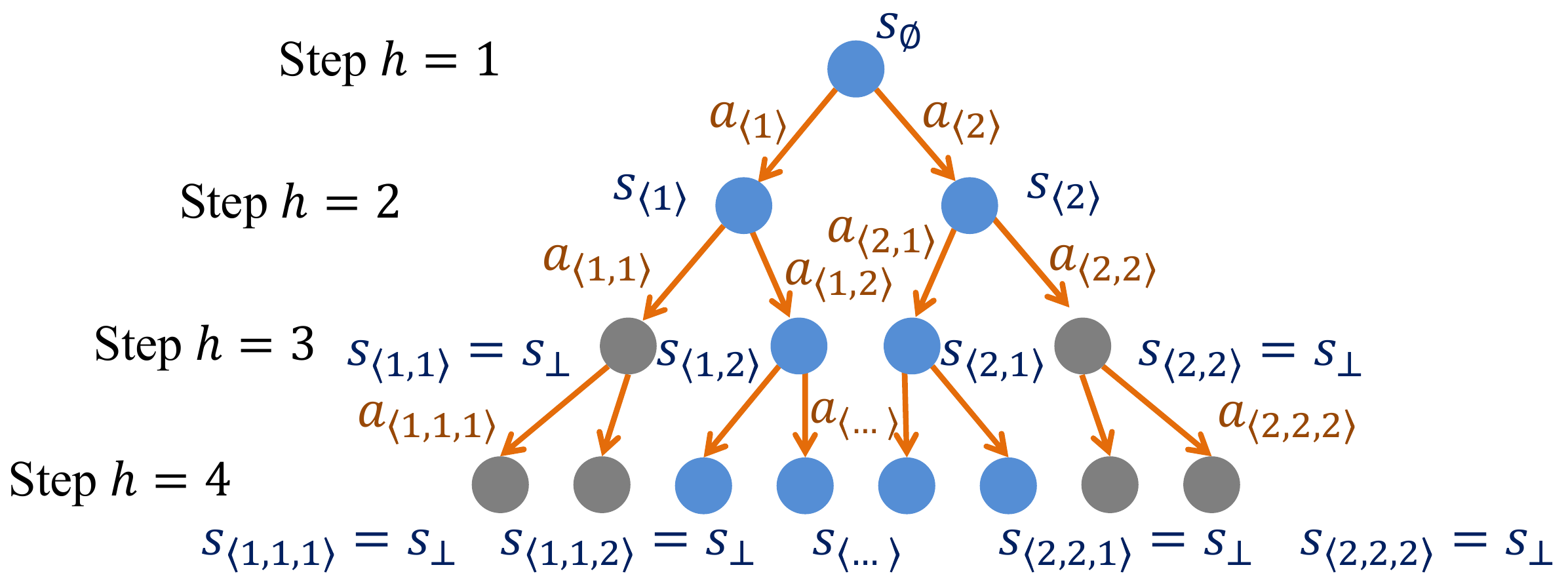} 
	\caption{Illustrating example with $m=2$ for branching RL.} 
	\label{fig:illustrating_example}     
\end{figure}

\textbf{Branching Value functions and Bellman Equations.}
For any policy $\pi$, we define value function $V_h^{\pi}:\cS \mapsto \R$, so that 
\begin{align}
V^{\pi}_h(s)  = 
& \ex_{q,p,\pi} \Big[ \sum_{\sigma'=\emptyset}^{m^{\oplus(H-h)}}
\sum_{\ell=1}^{m} q(s_{\sigma \oplus \sigma'}, a_{\sigma \oplus \sigma'\oplus \ell} ) \cdot  
\nonumber\\
& \quad \quad \quad  r(s_{\sigma \oplus \sigma'}, a_{\sigma \oplus \sigma'\oplus \ell} ) | s_{\sigma}=s \Big] \label{eq:V_function2}
\end{align}
gives the expected cumulative reward starting from some state $s$ at step $h$ till the end of this branch, under policy $\pi$. 
Here $\sigma$ is the index string for an arbitrary state at step $h$, and thus $\sigma \!\in\! \{1^{\oplus(h-1)},\dots,m^{\oplus(h-1)} \}$. $\sum_{\sigma'=\emptyset}^{m^{\oplus(H-h)}}$ denotes the summation over strings $\sigma'=\emptyset,\Braket{1},\dots,\Braket{m},\Braket{1,1},\dots,\Braket{m,m}, \dots, m^{\oplus(H-h)}$, which effectively enumerates all tree nodes of $H\!-\!h\!+\!1$ layers.
The expectation is taken with respect to the trajectory, which is dependent on trigger distribution $q$, 
transition distribution $p$ and policy $\pi$.

Accordingly, we also define Q-value function $Q_h^{\pi}:\cS \times \cA \mapsto \R$, so that
\begin{align*}
Q^{\pi}_h(s,A) =& 
\ex_{q,p,\pi} \Big[ \sum_{\sigma'=\emptyset}^{m^{\oplus(H-h)}} 
\sum_{\ell=1}^{m} q(s_{\sigma \oplus \sigma'}, a_{\sigma \oplus \sigma'\oplus \ell} ) \cdot  
\\
& \quad \quad \quad  r(s_{\sigma \oplus \sigma'}, a_{\sigma \oplus \sigma'\oplus \ell} ) | s_{\sigma}=s, A_{\sigma}=A \Big]
\end{align*}
denotes the expected cumulative reward starting from some state-action pair $(s,A)$ at step $h$ till the end of this branch, under policy $\pi$. From the definitions of $r,p,q$ for ending state $s_{\perp}$, we have $V^{\pi}_h(s_{\perp})=Q^{\pi}_h(s_{\perp},A)=0$ for any $A \in \cA, h \in [H], \pi$.

Since $\cS$, $\cA$ and $H$ are all finite, there exists a deterministic optimal policy $\pi^*$ which has the optimal value $V_h^*(s)=\sup_{\pi} V_h^{\pi}(s)$ for any $s \in \cS$ and $h \in [H]$.
Then, we can establish the Bellman (optimality) equations as follows:

\begin{equation*}
\left\{
\begin{aligned}
Q^{\pi}_h(s,A)=&\sum_{a \in A} q(s,a) \sbr{r(s,a)+p(\cdot|s,a)^\top V^{\pi}_{h+1}}
\\
V^{\pi}_h(s)=&Q^{\pi}_h(s,\pi_h(s))
\\
V^{\pi}_{H+1}(s)=&0, \ \forall s \in \cS ,
\end{aligned}
\right.
\end{equation*}
\begin{equation*}
\left\{
\begin{aligned}
Q^{*}_h(s,A)=&\sum_{a \in A} q(s,a) \sbr{r(s,a)+p(\cdot|s,a)^\top V^{*}_{h+1}}
\\
V^{*}_h(s)=& \max_{A \in \cA} Q^{*}_h(s,A)
\\
V^{*}_{H+1}(s)=&0, \ \forall s \in \cS .
\end{aligned}
\right.
\end{equation*}

Under the framework of branching RL, we consider two important RL settings, i.e., regret minimization (branching RL-RM) and reward-free exploration (branching RL-RFE).

\textbf{Regret Minimization (RM).} 
In branching RL-RM, the agent plays the branching RL game for $K$ episodes, and the goal is to minimize the following regret

\begin{align*}
	\regret(K)=\sum_{k=1}^{K} \sbr{V_1^*(s_{\emptyset}^k)-V_1^{\pi_k}(s_{\emptyset}^k) } .
\end{align*}

\textbf{Reward-Free Exploration (RFE).} 
Branching RL-RFE consists of two phases, i.e., exploration and planning. (i) In the exploration phase, given a fixed initial state $s_{\emptyset}$, the agent plays the branching RL game \emph{without} the observation of reward function $r$, and estimates a trigger and transition model $(\hat{q},\hat{p})$. (ii) In the planning phase, the agent is given reward function $r$, and computes the optimal policy $\hat{\pi}^{*}$ under her estimated model $(\hat{q},\hat{p})$ with respect to $r$.
Given an accuracy parameter $\varepsilon$ and a confidence parameter $\delta$, she needs to guarantee that for \emph{any} given reward function $r$, the policy $\hat{\pi}^{*}$ with respect to $r$ is \emph{$\varepsilon$-optimal}, i.e.,

\begin{align*}
V_1^{\hat{\pi}^{*}}(s_{\emptyset};r) \geq V_1^{\pi^{*}}(s_{\emptyset};r)-\varepsilon ,
\end{align*}

with probability at least $1-\delta$. We measure the performance by \emph{sample complexity}, i.e., the number of episodes used in the exploration phase to guarantee an $\varepsilon$-optimal policy for any given $r$.

In order to ensure that the number of triggered state-base action pairs will not increase exponentially and designing sample efficient algorithms is possible in branching RL, we introduce the following assumption. 

\begin{assumption}[Bounded Trigger Probability] \label{assumption:tri_prob}
	For any $(s,a) \in \cS \times A^{\univ}$, we have
	$ q(s,a) \leq \frac{1}{m}$.
\end{assumption}

To justify the necessity of Assumption~\ref{assumption:tri_prob}, we provide a rigorous lower bound to show that once relaxing the threshold of trigger probability, any branching RL algorithm must suffer an exponential regret.

\begin{theorem}\label{thm:relax_assumption}
	Suppose that for any $(s,a) \in \cS \times A^{\univ}$, $q(s,a) \leq \bar{q}$ for some threshold parameter $\bar{q} > \frac{1}{m}$. Then, there exists an instance of branching RL with $H>1$, where the regret of any algorithm is bounded by $\Omega ( \frac{m \bar{q} ( (m \bar{q})^{H-1}-1 )}{m \bar{q}-1} \sqrt{SNK} ) $.
\end{theorem}

We describe the intuition behind this lower bound, and defer the full proof to Appendix~\ref{apx:proof_lb_relax_assump}.
Consider a branching MDP, where at an early step the agent has to distinguish the optimal action that has trigger probability $\bar{q}$, from the sub-optimal actions that have trigger probabilities only $\bar{q}-\eta$. 
Once the agent takes a sub-optimal action at the early step, such trigger sub-optimality will impact exponentially many states in the following steps, and 
she will suffer a regret of $m\eta \cdot \sbr{ m\bar{q}+m^2\bar{q}^2+\dots+m^{H-1}\bar{q}^{H-1} }$ in this episode, which is the sum of a geometric progression with common ratio $m\bar{q}$. If $\bar{q}>\frac{1}{m}$, summing over all episodes, the total regret is exponentially large with respect to $H$.


Besides this lower bound, Assumption~\ref{assumption:tri_prob} is also mild in practice, since in real-world applications such as recommendation systems~\cite{drl_recommendation} and online advertising~\cite{tree_based_recommendation}, it is often the case that users are only attracted to and click on a few items in a recommended list. In addition, in multi-step (e.g., category-based) recommendation, the interests of users usually converge to a single branch in the end (in expectation).

When $m=1$, our branching RL reduces to standard episodic RL~\cite{azar2017minimax,jin2018q,zanette2019tighter} with transition probability $p^{\aug}$, such that $p^{\aug}(s_{\perp}|s,a)=1-q(s,a)$ and $p^{\aug}(s'|s,a)=q(s,a)p(s'|s,a)$ for any $s' \neq s_{\perp}$. In this case, our results match the state-of-the-art results for standard RL in both RM~\cite{azar2017minimax,zanette2019tighter} and RFE~\cite{menard2021fast,zhang2021nearly} settings.

\section{Properties of the Branching Markov Decision Process}\label{sec:property_study}

Before introducing our algorithms for branching RL, in this section, we first investigate special structural properties of branching MDP, which are critical to deriving tight (polynomial) regret and sample complexity guarantees.

\subsection{Branching Value Difference Lemma and Law of Total Variance} \label{sec:key_lemmas}

Different from standard episodic MDP~\cite{azar2017minimax,zanette2019tighter} where a trajectory is an $H$-step path, the trajectory of branching MDP is an $m$-ary tree with each node a state and each edge a state-base action pair. 
Thus, many analytical tools in standard MDP, e.g., value difference lemma~\cite{dann2017unifying} and law of total variance~\cite{jin2018q,zanette2019tighter}, cannot be directly applied to branching MDP. To handle this problem, we establish new fundamental techniques for branching MDP, including branching value difference lemma and branching law of total variance.
 
First, we present a branching value difference lemma.

\begin{lemma}[Branching Value Difference Lemma]\label{lemma:value_diff_main_text}
	For any two branching MDP $\cM'(\cS,A^{\univ},\cA,m,H,q',p',r)$ and $\cM''(\cS,A^{\univ},\cA,m,H,q'',p'',r)$, the difference in values under the same policy $\pi$ satisfies that
	\begin{align*}
	& V'^{\pi}_h(s)-V''^{\pi}_h(s) =
	\!\!\! \sum_{\sigma'=\emptyset}^{m^{\oplus(H-h)}}
	\! \sum_{\ell=1}^{m}
	\ex_{q'',p'',\pi} \Big[  \big( q'(s_{\tau},a_{\tau \oplus \ell}) \\& \!-\! q''(s_{\tau},a_{\tau \oplus \ell}) \big) \!\cdot\! r(s_{\tau},a_{\tau \oplus \ell}) + \big(q'(s_{\tau},a_{\tau \oplus \ell}) p'(s_{\tau},a_{\tau \oplus \ell}) \\& \!-\! q''(s_{\tau},a_{\tau \oplus \ell}) p''(s_{\tau},a_{\tau \oplus \ell}) \big)^{\top}  V'^{\pi}_{|\tau \oplus \ell|+1} \big| s_{\sigma}=s \Big],
	\end{align*}
	where $\tau:=\sigma \oplus \sigma'$.
\end{lemma}

Using Lemma~\ref{lemma:value_diff_main_text} with $\cM'$ and $\cM''$ being the optimistic and true models, respectively, we can bound the difference between optimistic and true values by the deviations between optimistic and true trigger and transition probabilities, in expectation with respect to the true model.

Next, we provide a branching law of total variance, which is critical to analyzing the estimation error of transition.

\begin{lemma}[Branching Law of Total Variance]\label{lemma:LTV_main_text}
	For any policy $\pi$, 
	\begin{align}
	& \!\!\! \ex_{q,p,\pi} \Big[ \sum_{\sigma=\emptyset}^{m^{\oplus(H-1)}} \!\! \sum_{\ell=1}^{m} \var_{q,p}\sbr{V^{\pi}_{|\sigma \oplus \ell|+1}(s_{\sigma \oplus \ell}) | s_{\sigma},a_{\sigma \oplus \ell} } \Big] 
	\nonumber\\
	= & \ex_{q,p,\pi} \! \Big[ \! \Big( \!\!\!\!\!\!\sum_{\sigma=\emptyset}^{m^{\oplus (\! H-1 \!)}} \!\!\!\!\sum_{\ell=1}^{m} 
	q(s_{\sigma},a_{\sigma \oplus \ell}) r(s_{\sigma},a_{\sigma \oplus \ell}) \!-\!  V^{\pi}_{1}\!(s_{\emptyset}) \! \Big)^{\!\!2} \Big] \!\!\!\!\label{eq:LTV}
	\\
	\leq & \ex_{q,p,\pi} \Big[ \Big( \sum_{\sigma=\emptyset}^{m^{\oplus(H-1)}}  \indicator{ s_{\sigma} \neq s_{\perp}} \Big)^2 \Big] . \label{eq:LTV_ub}
	\end{align}
\end{lemma}

Here $\var_{q,p}\sbr{V^{\pi}_{|\sigma \oplus \ell|+1}(s_{\sigma \oplus \ell}) | s_{\sigma},a_{\sigma \oplus \ell} }$ denotes the variance of value $V^{\pi}_{|\sigma \oplus \ell|+1}(s_{\sigma \oplus \ell})$ with respect to $s_{\sigma \oplus \ell}$, which depends on trigger probability $q(s_{\sigma},a_{\sigma \oplus \ell})$ and transition probability $p(\cdot|s_{\sigma},a_{\sigma \oplus \ell})$, conditioning on $(s_{\sigma},a_{\sigma \oplus \ell})$.

\textbf{Remark 1.}
Lemma~\ref{lemma:LTV_main_text} exhibits that under branching MDP, the sum of conditional variances over all state-base action pairs is equal to the overall variance considering the whole trajectory, shown by Eq.~\eqref{eq:LTV}. 
Furthermore, the overall variance can be bounded by the total number of regular (triggered) states, revealed by Eq.~\eqref{eq:LTV_ub}.

From Lemma~\ref{lemma:LTV_main_text}, we have that to bound the estimation error of transition, which is related to the sum of conditional variances, it suffices to bound the total number of triggered states in a trajectory tree (discussed in the following).


\subsection{The Number of Triggered States}

In this subsection, we show that with Assumption~\ref{assumption:tri_prob} that only constrains the first moment of trigger distribution, we can bound both the first and second moments of the number of triggered states in a trajectory tree.


\begin{lemma}[The Number of Triggered States]\label{lemma:triggered_nodes_main_text}
	For any policy $\pi$, 
	\begin{align}
	&\ex_{q,p,\pi} \Big[  \sum_{\sigma=\emptyset}^{m^{\oplus(H-1)}}  \indicator{ s_{\sigma} \neq s_{\perp}} \Big] \leq H , \label{eq:triggered_nodes_ex}
	\\
	&\ex_{q,p,\pi} \Big[\Big( \sum_{\sigma=\emptyset}^{m^{\oplus(H-1)}}  \indicator{ s_{\sigma} \neq s_{\perp}} \Big)^2 \Big] \leq 3 H^2 . \label{eq:triggered_nodes_sec_moment}
	\end{align}
\end{lemma}

\textbf{Remark 2.}
Eq.~\eqref{eq:triggered_nodes_ex} gives a universal upper bound of value function as
$
V^{\pi}_h(s) \leq \mbr{ \sum_{\sigma=\emptyset}^{m^{\oplus(H-1)}}  \indicator{ s_{\sigma} \neq s_{\perp}}  } \leq H
$ for any $s\in\cS,h \in [H],\pi$.
Moreover, Eq.~\eqref{eq:triggered_nodes_sec_moment} provides a sharp upper bound for overall variance, as well as the sum of conditional variances of transition (by plugging Eq.~\eqref{eq:triggered_nodes_sec_moment} into Lemma~\ref{lemma:LTV_main_text}). 
To our best knowledge, this second moment result is novel.

Lemma~\ref{lemma:triggered_nodes_main_text} shows that despite the exponentially increasing nodes in branching MDP, its value and overall variance (estimation error) will not explode. This critical property enables us to avoid an exponentially-large regret or sample complexity.

\textbf{Novel Analysis for Triggered States.}
The analysis of Lemma~\ref{lemma:triggered_nodes_main_text} is highly non-trivial. We first relax all regular trigger probabilities to $\frac{1}{m}$, and then investigate the number of triggered states for each step individually.
While we can show that the number of triggered states at each step is a conditional Binomial random variable, the distribution of their sum is too complex to express.
This incurs a non-trivial challenge on analyzing the second moment of the total number of triggered states.
To tackle this challenge, we investigate the correlation of triggered states between any two steps, by exploiting the structure of branching MDP.

 
\emph{Proof sketch.}
Under Assumption~\ref{assumption:tri_prob}, to bound the total number of triggered (regular) states for any branching MDP and policy $\pi$, it suffices to bound it under a relaxed model $\cM^*$ with $q(s,a)=q^*:=\frac{1}{m}$ for all $(s,a) \in \cS \setminus \{s_{\perp}\} \times A^{\univ}$. 
Let $\omega_h$ denote the number of triggered states at each step $h$ under $\cM^*$, and $\omega:=\sum_{h=1}^{H} \omega_h$. 
Below we prove that $\ex[\omega] \leq H$ and $\ex[\omega^2] \leq 3H^2$.

For $h=1$, $\omega_h=1$ deterministically. 
For $h \geq 2$, $\omega_h | \omega_{h-1} \sim \binomial(m \omega_{h-1} ,q^*)$. According to the properties of Binomial distribution and $q^*:=\frac{1}{m}$, for $h \geq 2$,
\begin{align*}
	&\ex \mbr{\omega_h} \hspace*{0.5em} =  mq^* \ex \mbr{\omega_{h-1}} = 1 ,
	\\
	&\ex \mbr{(\omega_h)^2} \!=\!  m q^* (1-q^*) \ex \mbr{\omega_{h-1}}  \!+\! m^2 (q^*)^2 \ex \mbr{(\omega_{h-1})^2} 
	\\
	& \hspace*{3.85em} = (1-q^*) + \ex \mbr{(\omega_{h-1})^2} \leq h .
\end{align*}

Hence, we have that $\ex[\omega]=\sum_{h=1}^{H} \ex [\omega_h]=H$, and 

\begin{align}
\ex[\omega^2] = & \sum_{h=1}^{H} \ex [(\omega_h)^2] + 2 \sum_{1<i,j<H} \ex [\omega_i \omega_j] 
\nonumber\\
\leq & \frac{H(H+1)}{2} + 2 \sum_{1<i,j<H} \ex [\omega_i \omega_j] . \label{eq:sec_moment_omega}
\end{align}

Now, the challenge falls on how to bound $\ex [\omega_i \omega_j]$ for any $1<i,j<H$.
Let $W_{\sigma}$ be a Bernoulli random variable denoting whether state $s_{\sigma}$ is triggered for any index string $\sigma$. Then, we can write $\ex [\omega_i \omega_j]$ as 
\begin{align}
	& \ex \Big[  \big( \sum_{\sigma=1^{\oplus(i-1)}}^{m^{\oplus(i-1)}} W_{\sigma} \big) \cdot  \big(\sum_{\sigma'=1^{\oplus(j-1)}}^{m^{\oplus(j-1)}} W_{\sigma'} \big) \Big]
	\nonumber\\
	\overset{\textup{(a)}}{=} & m^{i-1} \ex \Big[   W_{1^{\oplus(i-1)}}  \big(\sum_{\sigma'=1^{\oplus(j-1)}}^{m^{\oplus(j-1)}} W_{\sigma'} \big) \Big]
	\nonumber\\
	= & m^{i-1} \ex \Big[ W_{1^{\oplus(i-1)}} \big( \hspace*{-1.5em} \sum_{\begin{subarray}{c}
	\sigma'=1^{\oplus(j-1)}
	\\ \sigma' \textup{starts with } 1^{\oplus(i-1)}
	\end{subarray} }^{m^{\oplus(j-1)}} \hspace*{-2em} W_{\sigma'} + \hspace*{-1em} \sum_{\begin{subarray}{c}
	\sigma'=1^{\oplus(j-1)}
	\\ \sigma' \textup{does not start with } 1^{\oplus(i-1)}
	\end{subarray} }^{m^{\oplus(j-1)}} \hspace*{-3.5em} W_{\sigma'} \big) \Big] 
	\nonumber\\
	\overset{\textup{(b)}}{\leq} & m^{i-1}  \sbr{m^{j-i} (q^{*})^{i-1+j-i} + m^{j-1} (q^{*})^{i-1+j-1} }
	\nonumber\\
	= & 2 . \label{eq:omega_ij}
\end{align}

Here (a) comes from the symmetry of trajectory tree. (b) is due to that at step $j$, the children states of $s_{1^{\oplus(i-1)}}$ have dependency on it and the other states are independent of it, and $\ex [W_{\sigma}W_{\sigma'}] =\Pr[W_{\sigma}=1, W_{\sigma'}=1]$ for any $\sigma,\sigma'$.
By plugging Eq.~\eqref{eq:omega_ij} into Eq.~\eqref{eq:sec_moment_omega}, we have 
$
\ex[\omega^2]
\leq \frac{H(H+1)}{2} + 4 \frac{H(H-1)}{2} \leq 3H^2
$. 
Thus, we obtain Lemma~\ref{lemma:triggered_nodes_main_text}.
\hfill $\square$

\section{Branching Reinforcement Learning with Regret Minimization}

In this section, we study branching RL-RM, and propose an efficient algorithm $\algtteuler$ with a near-optimal regret guarantee for large enough $K$. A lower bound is also established to validate the optimality of $\algtteuler$.

\begin{algorithm}[t]
	\caption{$\algtteuler$} \label{alg:tt_euler}
	\begin{algorithmic}[1]
		\STATE {\bfseries Input:}  confidence parameter $\delta$,
		$\delta' := \frac{1}{6}\delta$, $L:=\log(\frac{SNH(m^{H} \vee K)}{\delta'})$. Initialize $\bar{V}^k_h(s_{\perp})=\underline{V}^k_h(s_{\perp}) = 0 ,\ \forall h \in [H],k $.
		\FOR{$k=1,2,\dots$}
		\FOR{$h=H,H-1,\dots,1$}
		\FOR{$s \in \cS \setminus \{s_{\perp}\}$}
		\FOR{$a \in A^{\univ}$}
			\STATE $\hat{q}^k(s,a) \!\leftarrow\! \frac{J^k_{\sumtext}(s,a)}{n^k(s,a)}$.  $b^{k,q}(s,a) \!\leftarrow\! 4\sqrt{\frac{L}{n^k(s,a)}}$; \label{line:rm_hat_q_bonus_q}
			\STATE $\hat{p}^k(s'|s,a) \leftarrow \frac{P^k_{\sumtext}(s'|s,a)}{J^k_{\sumtext}(s,a)}, \ \forall s' \in \cS $; \label{line:rm_hat_p}
			\STATE $b^{k,qpV}(s,a) \leftarrow 4\sqrt{\frac{\var_{s'}\sbr{\bar{V}^{k}_{h+1}(s')} L }{n^k(s,a)}} 
			+  4\sqrt{\frac{\!\ex_{s'} \!\mbr{\sbr{ \bar{V}^{k}_{h+1}(s')-\underline{V}^{k}_{h+1}(s') }^2 } \!L }{n^k\!(s,a)}} \!+\! \frac{36 H L}{n^k \!(s,a)}$;
			\label{line:rm_bonus_qpV_ub}
			\STATE $f^k_h(s,a) \leftarrow (\hat{q}^k(s,a)+b^{k,q}(s,a)) r(s,a) + \hat{q}^k(s,a) \hat{p}^k(\cdot|s,a)^\top \bar{V}^k_{h+1}+b^{k,qpV}(s,a)$; \label{line:rm_component_function}
		\ENDFOR
		\STATE $\bar{V}^k_h(s) \!\leftarrow\! \min\{\max_{A \in \cA} \sum_{a \in A} \! f^k_h(s,a),\  H \}$; 
		\label{line:rm_V_function_ub}
		\STATE $\pi^k_h(s) \leftarrow \argmax_{A \in \cA} \sum_{a \in A} f^k_h(s,a)$; \label{line:rm_pi}
		\STATE $\underline{V}^k_h(s) \leftarrow \max\{\sum_{a \in \pi^k_h(s)} ((\hat{q}^k(s,a)-b^{k,q}(s,a)) r(s,a) + \hat{q}^k(s,a) \hat{p}^k(\cdot|s,a)^\top \underline{V}^k_{h+1}-b^{k,qpV}(s,a)),\  0\}$; \label{line:rm_V_function_lb}
		\ENDFOR
		\ENDFOR
		\STATE Take policy $\pi^k$ and observe the trajectory; \label{line:rm_execute_policy}
		\ENDFOR
	\end{algorithmic}
\end{algorithm}

\subsection{Algorithm $\algtteuler$}

The algorithm design for branching RL faces two unique \emph{challenges}: (i) Computation efficiency. Since the action space $\cA$ can be combinatorially large, it is inefficient to explicitly maintain $Q$ function as in standard RL~\cite{azar2017minimax,zanette2019tighter}; 
(ii) Tight optimistic estimator. Naively adapting standard RL algorithms~\cite{azar2017minimax,zanette2019tighter} and adding optimistic bonuses on trigger and transition probabilities, respectively, will lead to a loose regret bound (see Appendix~\ref{apx:discussion_bonus} for details).   
To handle these challenges, $\algtteuler$ only maintains the components of $Q$ function, and uses a maximization oracle to directly calculate $V$ function. In addition, $\algtteuler$ considers trigger and transition distributions as a whole and adds a composite bonus.

The procedure of $\algtteuler$ (Algorithm~\ref{alg:tt_euler}) is as follows.
In each episode $k$, we first calculate the empirical trigger and transition probabilities $\hat{q}^k,\hat{p}^k$, a bonus for trigger probability $b^{k,q}$, and a composite bonus for trigger and transition probabilities $b^{k,qpV}$ (Lines~\ref{line:rm_hat_q_bonus_q}-\ref{line:rm_bonus_qpV_ub}). Here $n^k(s,a)$, $J^k_{\sumtext}(s,a)$ and $P^k_{\sumtext}(s'|s,a)$ denote the number of times $(s,a)$ was visited, the number of times $(s,a)$ was successfully triggered, and the number of times the agent transitioned to $s'$ from $(s,a)$ up to episode $k$, respectively. 
Then, we calculate a component function $f^k_h(s,a)$, which represents the contribution to value function from each $(s,a)$ (Line~\ref{line:rm_component_function}). 

We allow $\algtteuler$ to access a maximization oracle which can efficiently calculate $\max_{A \in \cA} \sum_{a \in A} w(a)$ and $\argmax_{A \in \cA} \sum_{a \in A} w(a)$ for any vector $\bw \in \R^{N}$ ($N:=|A^{\univ}|$). Since the objective function $\sum_{a \in A} w(a)$ is linear, such oracle exists for many combinatorial decision classes, e.g., all $m$-cardinality subsets and $m$-cardinality matchings.
%
%
By utilizing this oracle with $f^k_h(s,a)$, we can efficiently calculate the optimistic value function $\bar{V}^k_h(s)$
and policy $\pi^k_h$, and further compute the pessimistic value function $\underline{V}^k_{h}$  (Lines~\ref{line:rm_V_function_ub}-\ref{line:rm_V_function_lb}).
After figuring out $\bar{V}^k_h(s),\underline{V}^k_{h},\pi^k_h(s)$ for all $s,h$, we execute policy $\pi^k$ in episode $k$ (Line~\ref{line:rm_execute_policy}).

\textbf{Computation Efficiency.}
We remark that the computation complexity of $\algtteuler$ is $O(S^2 N)$, instead of expensive $O(S^2 |\cA|)$ as one may suffer by naively adapting standard RL algorithms~\cite{azar2017minimax,zanette2019tighter}.
This advantage is due to that $\algtteuler$ only maintains the component function instead of the $Q$ function, and utilizes a maximization oracle to directly compute $V$ function. 


\subsection{Regret Upper Bound for $\algtteuler$}

Now we provide the regret guarantee for $\algtteuler$.

\begin{theorem}[Regret Upper Bound] \label{thm:regret_ub}
	With probability at least $1-\delta$, for any episode $K>0$, the regret of algorithm $\algtteuler$ is bounded by
	$
	O ( H \sqrt{SNK} \log (\frac{SNH(m^{H} \vee K)}{\delta} ) ) .
	$
	In particular, when $K\geq m^H$, the regret is bounded by
	\begin{align*}
		O \sbr{ H \sqrt{SNK} \log\sbr{\frac{SNHK}{\delta} } }
	\end{align*}
\end{theorem}

\textbf{Remark 3.}
Theorem~\ref{thm:regret_ub} shows that, despite the exponentially-large trajectory of branching RL, $\algtteuler$ enjoys a regret only polynomial in problem parameters. For large enough $K$ such that $K\geq m^H$, Theorem~\ref{thm:regret_ub} matches the lower bound (presented in Section~\ref{sec:regret_lb}) up to logarithmic factors. In addition, when branching RL reduces to standard RL (i.e., $m=1$), our result also matches the state-of-the-arts~\cite{azar2017minimax,zanette2019tighter}.

\textbf{Branching Regret Analysis.}
In contrast to standard RL~\cite{azar2017minimax,zanette2019tighter}, we derive a tree-structured regret analysis based on special structural properties of branching RL in Section~\ref{sec:property_study}.

Using the branching value difference lemma (Lemma~\ref{lemma:value_diff_main_text}), we can decompose the regret into the estimation error from each regular state-base action pair as follows: 
\begin{align*}
&\regret(K) \leq \sum_{k=1}^{K}  \sum_{\sigma=\emptyset}^{m^{\oplus(H-1)}} \sum_{\ell=1}^{m} \sum_{(s,a),s\neq s_{\perp}} w_{k,\sigma,\ell}(s,a) \cdot
\nonumber\\
& \Big[ \underbrace{ \sbr{ \tilde{q}^k(s,a) - q(s,a) } r(s,a) }_{\term{1:\ Triggered\ reward}} 
\nonumber\\ 
&+ \underbrace{\sbr{ \tilde{q}^k(s,a) \tilde{p}^k(\cdot|s,a) - \hat{q}^k(s,a) \hat{p}^k(\cdot|s,a) }^\top \bar{V}^{k}_{|\sigma \oplus \ell|+1}}_{\term{2:\ Triggered\ transition\ optimism}} 
\nonumber\\ 
&+ \underbrace{\sbr{ \hat{q}^k(s,a) \hat{p}^k(\cdot|s,a) - q(s,a) p(\cdot|s,a) }^\top V^{*}_{|\sigma \oplus \ell|+1}}_{\term{3:\ Triggered\ transition\ estimation}} 
\nonumber\\
& \!\!+\!\! \underbrace{ \sbr{ \hat{q}^k\!(s,a) \hat{p}^k\!(\cdot|s,a) \!\!-\!\! q(s,a) p(\!\cdot|s,a\!) }^{\!\!\!\top} \!\!\! \sbr{\! \bar{V}^{k}_{|\sigma \oplus \ell|+1} \!\!-\!\! V^{*}_{|\sigma \oplus \ell|+1} \!} \!\!}_{\term{4:\ Lower\ order\ term}}  \Big] \!.
\end{align*}
Here $w_{k,\sigma,\ell}(s,a)$ denotes the probability that $(s_{\sigma},a_{\sigma \oplus \ell})=(s,a)$ in episode $k$, and $\tilde{q}^k$ and $\tilde{p}^k$ represent optimistic trigger and transition probabilities, respectively.
In this decomposition, we address the estimation error for triggered reward ($\term{1}$) and triggered transition ($\mathtt{Terms\ 2,3}$) separately, and consider trigger and transition probabilities as a whole distribution (in $\mathtt{Terms\ 2,3}$).
The dominant terms are $\mathtt{Terms\ 2,3}$, which stand for the estimation error for triggered transition and depend on the sum of conditional variances. Using branching law of total variance (Lemma~\ref{lemma:LTV_main_text}) and the second moment bound of triggered states (Lemma~\ref{lemma:triggered_nodes_main_text}), we can bound $\mathtt{Terms\ 2,3}$ by only $O(H^2)$ despite  exponential state-base action pairs.

We note that it is necessary to separately address triggered reward and triggered transition, and consider trigger and transition as a whole distribution in bonus design (Line~\ref{line:rm_bonus_qpV_ub}) and regret decomposition. 
A naive adaption of standard RL algorithm~\cite{zanette2019tighter}, which adds bonuses on trigger and transition probabilities respectively, i.e., replacing Line~\ref{line:rm_component_function} with
$
f^k_h(s,a) \leftarrow  \sbr{\hat{q}^k(s,a)+b^{k,q}(s,a)} \cdot (r(s,a) +  \hat{p}^k(\cdot|s,a)^\top \bar{V}^k_{h+1} + b^{k,pV}(s,a)) ,
$ 
will suffer an extra factor $\sqrt{H}$ in the regret bound. Please see  \OnlyInFull{Appendix~\ref{apx:discussion_bonus}}\OnlyInShort{the supplementary material} for more discussion.

\subsection{Regret Lower Bound} \label{sec:regret_lb}

In this subsection, we provide a lower bound for branching RL-RM which is polynomial in problem parameters and demonstrates the optimality of  $\algtteuler$.

\begin{theorem}[Regret Lower Bound] \label{thm:regret_lb}
	There exists an instance of branching RL-RM, where any algorithm must have
	$
	\Omega ( H\sqrt{SNK} )
	$ regret.
\end{theorem}

\textbf{Remark 4.}\! Theorem~\ref{thm:regret_lb} validates that the regret of $\algtteuler$ (Theorem~\ref{thm:regret_ub}) is near-optimal for large enough $K$, and reveals that, a polynomial regret is achievable and tight even with exponentially-large trajectories in branching RL.

\textbf{Branching Regret Lower Bound Analysis.}
Unlike prior standard episodic RL works~\cite{azar2017minimax,jin2018q} which directly adapt the diameter-based lower bound~\cite{jaksch2010near} to the episodic setting, we derive a new lower bound analysis for branching RL-RM.
We construct a hard instance, where an agent uniformly enters one of $\Theta(S)$ ``bandit states'', i.e., states with an optimal action and sub-optimal actions, and hereafter always transitions to a ``homogeneous state'', i.e., a state with homogeneous actions. Then, if the agent makes a mistake in a bandit state, she will suffer $\Omega(H)$ regret in this episode. By bounding the KL-divergence between this hard instance and uniform instance, we can derive a desired lower bound.

\section{Branching Reinforcement Learning with Reward-free Exploration}

\begin{algorithm}[t]
	\caption{$\algbonusleader$} \label{alg:tt_bonusleader}
	\begin{algorithmic}[1]
		\STATE {\bfseries Input:} $s_{\emptyset}$, $\varepsilon$, $\delta$, $\beta(t,\kappa):=\log(SN/\kappa)+S\log(8e(t+1))$ for any $t \in \N$ and $\kappa \in (0,1)$. Initialize $ B^k_h(s_{\perp}) \!=\! 0, \forall h \in [H],k $ and $B^k_{H+1}(s) \!=\! 0,  \forall s \!\in\! \cS, k $.
		\FOR{$k=1,2,\dots$}
		\FOR{$h=H,H-1,\dots,1$}
		\FOR{$s \in \cS \setminus \{s_{\perp}\}$}
		\FOR{$a \in A^{\univ}$}
		\STATE $\hat{q}^k(s,a) \leftarrow \frac{J^k_{\sumtext}(s,a)}{n^k(s,a)}$; \label{line:rfe_hat_q}
		\STATE $\hat{p}^k(s'|s,a) \leftarrow \frac{P^k_{\sumtext}(s'|s,a)}{J^k_{\sumtext}(s,a)}, \ \forall s' \in \cS $; \label{line:rfe_hat_p}
		\STATE $g^{k}_h(s,a) \leftarrow 12 H^2 \frac{\beta(n^k(s,a),\delta)}{n^k(s,a)} + \sbr{1+\frac{1}{H}} \! \hat{q}^k(s,a) \hat{p}^k(\cdot|s,a)^\top \! B^{k}_{h+1}(s)$; \label{line:rfe_component_function}
		\ENDFOR
		\STATE $\pi^k_h(s) \leftarrow \argmax_{A \in \cA} \sum_{a \in A} g^k_h(s,a) $; \label{line:rfe_pi}
		\STATE $B^k_h(s) \leftarrow \min \{\max_{A \in \cA} \sum_{a \in A} g^k_h(s,a),\  H\}$; \label{line:rfe_B_kh}
		\ENDFOR
		\ENDFOR
		\STATE {\bfseries if} $4e \sqrt{ B^{k}_1(s_{\emptyset}) } + B^{k}_1(s_{\emptyset}) \leq \frac{\varepsilon}{2}$, 
		{\bfseries return} $(\hat{q}^k, \hat{p}^k)$;
		\label{line:rfe_stop_rule}
		\STATE {\bfseries else} Take policy $\pi^k$ and observe the trajectory;
		\label{line:rfe_execute_policy}
		\ENDFOR
	\end{algorithmic}
\end{algorithm}

In this section, we investigate branching RL-RFE, and develop an efficient algorithm $\algbonusleader$ and nearly matching upper and lower bounds of sample complexity.

\subsection{Algorithm $\algbonusleader$}

In each episode, $\algbonusleader$ estimates the trigger and transition probabilities, and computes the estimation error $B^k_h(s)$, which stands for the cumulative variance of trigger and transition from step $h$ to $H$.
Once the total estimation error $B^k_1(s_{\emptyset})$ is shrunk below the required accuracy, $\algbonusleader$ returns the estimated model $(\hat{q}^k,\hat{p}^k)$. Given any reward function $r$, the optimal policy $\hat{\pi}^*$ under  $(\hat{q}^k,\hat{p}^k)$  with respect to $r$ is $\varepsilon$-optimal, i.e., $V^{\hat{\pi}^*}_1(s_{\emptyset};r)\geq V^{*}_1(s_{\emptyset};r)-\varepsilon$, with probability at least $1-\delta$.

We describe $\algbonusleader$ (Algorithm~\ref{alg:tt_bonusleader}) as follows:
In each episode $k$, for each step $h$, $\algbonusleader$ first estimates the trigger and transition probabilities $\hat{q}^k, \hat{p}^k$ (Lines~\ref{line:rfe_hat_q},\ref{line:rfe_hat_p}), and calculates the component estimation error $g^{k}_h(s,a)$ for each state-base action pair (Line~\ref{line:rfe_component_function}).
Then, similar to $\algtteuler$, we utilize a maximization oracle to efficiently find the action with the maximum estimation error (the most necessary action for exploration) to be $\pi^k_h(s)$, and assign such maximum error to $B^k_h(s)$ (Lines~\ref{line:rfe_pi},\ref{line:rfe_B_kh}).
If the square root of total estimation error  $\sqrt{B^k_1(s_{\emptyset})}$, which represents the standard deviation of trigger and transition for whole trajectory, is smaller than $\varepsilon/2$, $\algbonusleader$ stops and outputs the estimated model $(\hat{q}^k,\hat{p}^k)$ (Line~\ref{line:rfe_stop_rule}); Otherwise, it continues to explore according to the computed policy $\pi^k$ (Line~\ref{line:rfe_execute_policy}).

The computation complexity of $\algbonusleader$ is also $O(S^2N)$ instead of $O(S^2|\cA|)$, since it only computes the component estimation error for state-base action pairs rather than enumerating super actions, and utilizes a maximization oracle to calculate $\pi^k_h(s)$ and $B^k_h(s)$.

\subsection{Sample Complexity Upper Bound for $\algbonusleader$}

Now we present the sample complexity for $\algbonusleader$.
%
%
We say an algorithm for branching RL-RFE is $(\delta,\varepsilon)$-correct, if it returns an estimated model $(\hat{q},\hat{p})$ such that given any reward function $r$, the optimal policy under $(\hat{q},\hat{p})$ with respect to $r$ is $\varepsilon$-optimal with probability at least $1-\delta$.

\begin{theorem}[Sample Complexity Upper Bound] \label{thm:sample_complexity_ub}
	For any $\varepsilon>0$ and $\delta \in (0,1)$, algorithm $\algbonusleader$ is $(\delta,\varepsilon)$-correct. Moreover, with probability $1-\delta$, the number of episodes used in $\algbonusleader$ is bounded by
	$$
	O \sbr{ \frac{H^2 SN}{\varepsilon^2} \sbr{ \log \sbr{\frac{SN}{\delta}} + S \log\sbr{e\cdot m^{H}} } \cdot C^2} ,
	$$
	where $C=\log\sbr{ \sbr{\log \sbr{\frac{SN}{\delta}} + S \log\sbr{e\cdot m^{H}}} \cdot \frac{H SN }{\varepsilon} }$.
\end{theorem}

\textbf{Remark 5.}
Theorem~\ref{thm:regret_ub} exhibits that even with exponentially-large trajectories in branching RL,  $\algbonusleader$ only needs polynomial episodes to ensure an $\varepsilon$-optimal policy for any reward function. 
This sample complexity is optimal for small enough $\delta$ within logarithmic factors (compared to the lower bound presented in Section~\ref{sec:sample_complexity_lb}).
In addition, when degenerating to standard RL (i.e., $m=1$), our result also matches the state-of-the-arts~\cite{menard2021fast,zhang2021nearly}.

\textbf{Branching Sample Complexity Analysis.}
Unlike standard RL~\cite{menard2021fast,zhang2021nearly} which only bounds the estimation error in a single $H$-step path, in branching RL we need to unfold and analyze the estimation error for all state-base action pairs in a trajectory tree. In our analysis, we utilize special structural properties of branching MDP, e.g., branching law of total variance (Lemmas~\ref{lemma:LTV_main_text}) and the second moment bound of triggered states (Lemmas~\ref{lemma:triggered_nodes_main_text}), to skillfully bound the total estimation error throughout the  trajectory tree. Despite exponentially-large trajectories, we obtain sample complexity only polynomial in problem parameters.

\subsection{Sample Complexity Lower Bound} \label{sec:sample_complexity_lb}

\begin{theorem}[Sample Complexity Lower Bound] \label{thm:sample_complexity_lb}
	There exists an instance of branching RL-RFE, where any $(\delta,\varepsilon)$-correct algorithm requires
	$
	\Omega (\frac{H^2 SN}{\varepsilon^2}  \log \delta^{-1}) 
	$
	trajectories.
\end{theorem}

\textbf{Remark 6.} 
Theorem~\ref{thm:sample_complexity_lb} demonstrates that $\algbonusleader$ (Theorem~\ref{thm:sample_complexity_ub}) achieves a near-optimal sample complexity for small enough $\delta$, and also, a polynomial sample complexity is achievable and sharp for branching RL-RFE.


\section{Experiments}

In this section, we conduct experiments for branching RL.
We set $K=5000$, $\delta=0.005$, $H=6$, $m=2$, $N\in \{10,15\}$, $\cS=\{s_{\perp},s_1,\dots,s_5\}$.
$\cA$ is the collection of all $m$-cardinality subsets of $A^{\univ}=\{ a_1,\dots,a_{N} \}$, and thus $|\cA|= \binom{N}{m} \in \{45,105\}$.  
The reward function $r(s,a)=1$ for any $(s,a) \in \cS \times \cA$.
The trigger probability $q(s,a)=\frac{1}{m}$ for any $(s,a) \in \cS \times \{a_{N-1},a_{N}\}$, and $q(s,a)=\frac{1}{2m}$ for any $(s,a) \in \cS \times A^{\univ} \setminus \{a_{N-1},a_{N}\}$.
We set $s_1$ as the initial state for each episode.
Under all actions $a \in A^{\univ}$, the transition probability $q(s'|s_1,a)=0.5$ for any $s' \in \{s_2,s_3\}$, and $q(s'|s,a)=0.5$ for any $(s,s') \in \{s_2,s_3\} \times \{s_4,s_5\}$ or $(s,s') \in \{s_4,s_5\} \times \{s_2,s_3\}$.
We perform $50$ independent runs, and report the average regrets and running times (in legends) across runs.

Since we study a new problem and there is no existing algorithm for branching RL, we compare our algorithm $\mathtt{BranchVI}$ with two adaptations from standard RL, i.e., $\mathtt{Euler\mbox{-}Adaptation}$~\cite{zanette2019tighter} and $\varepsilon\mbox{-}\mathtt{Greedy}$ ($\varepsilon=0.01$). The former uses individual exploration bonuses for trigger and transition, the latter uses $\varepsilon$-greedy in exploration, and both of them explicitly maintain Q-functions. 
As shown in Figure~\ref{fig:experiments}, $\mathtt{BranchVI}$ enjoys a lower regret and a faster running time than the baselines, which demonstrates the effectiveness of our exploration strategy and the computation efficiency of our algorithm.

\begin{figure}[t]  
	\centering  
	\includegraphics[width=0.49\columnwidth]{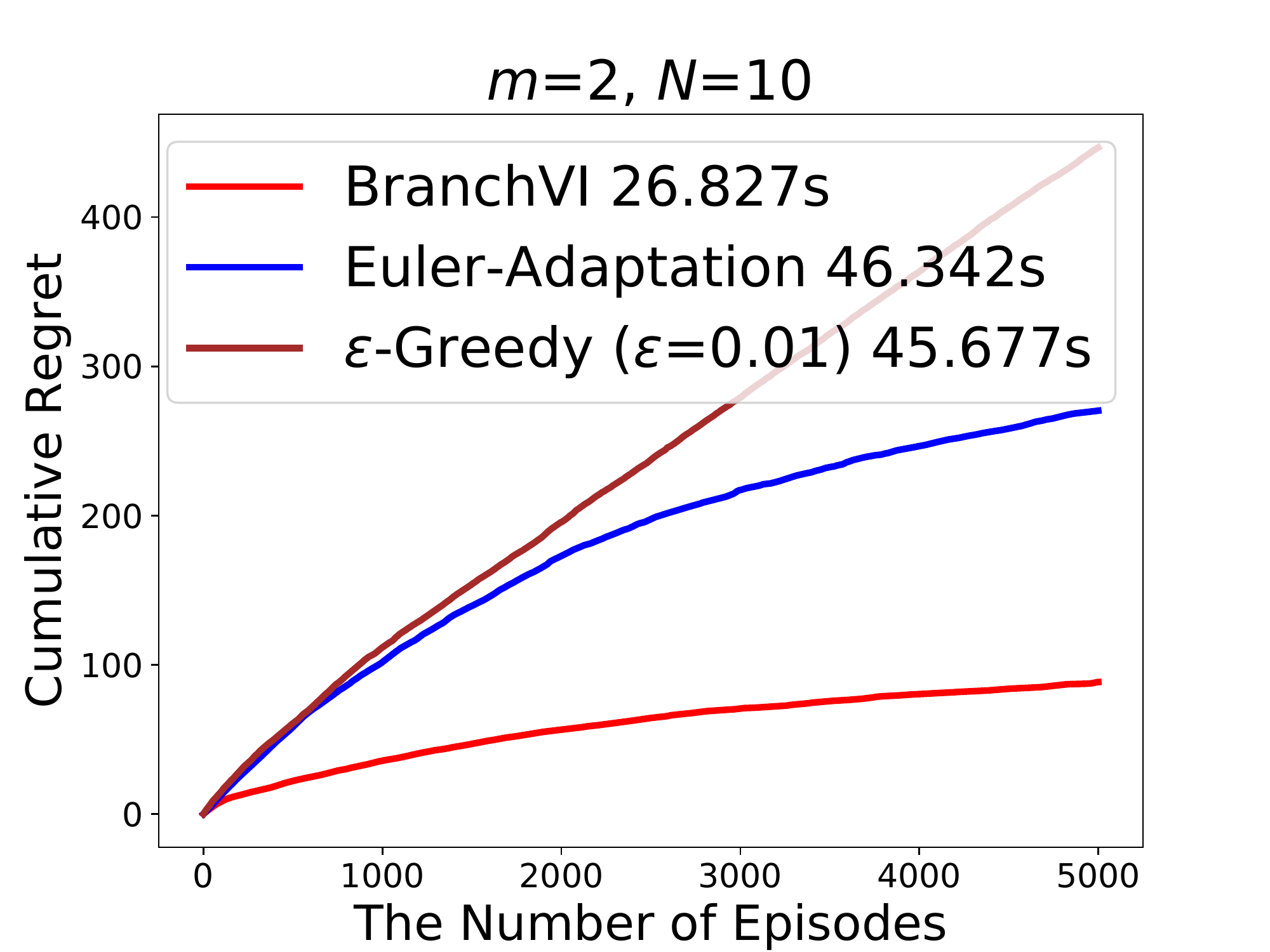}    
	\includegraphics[width=0.49\columnwidth]{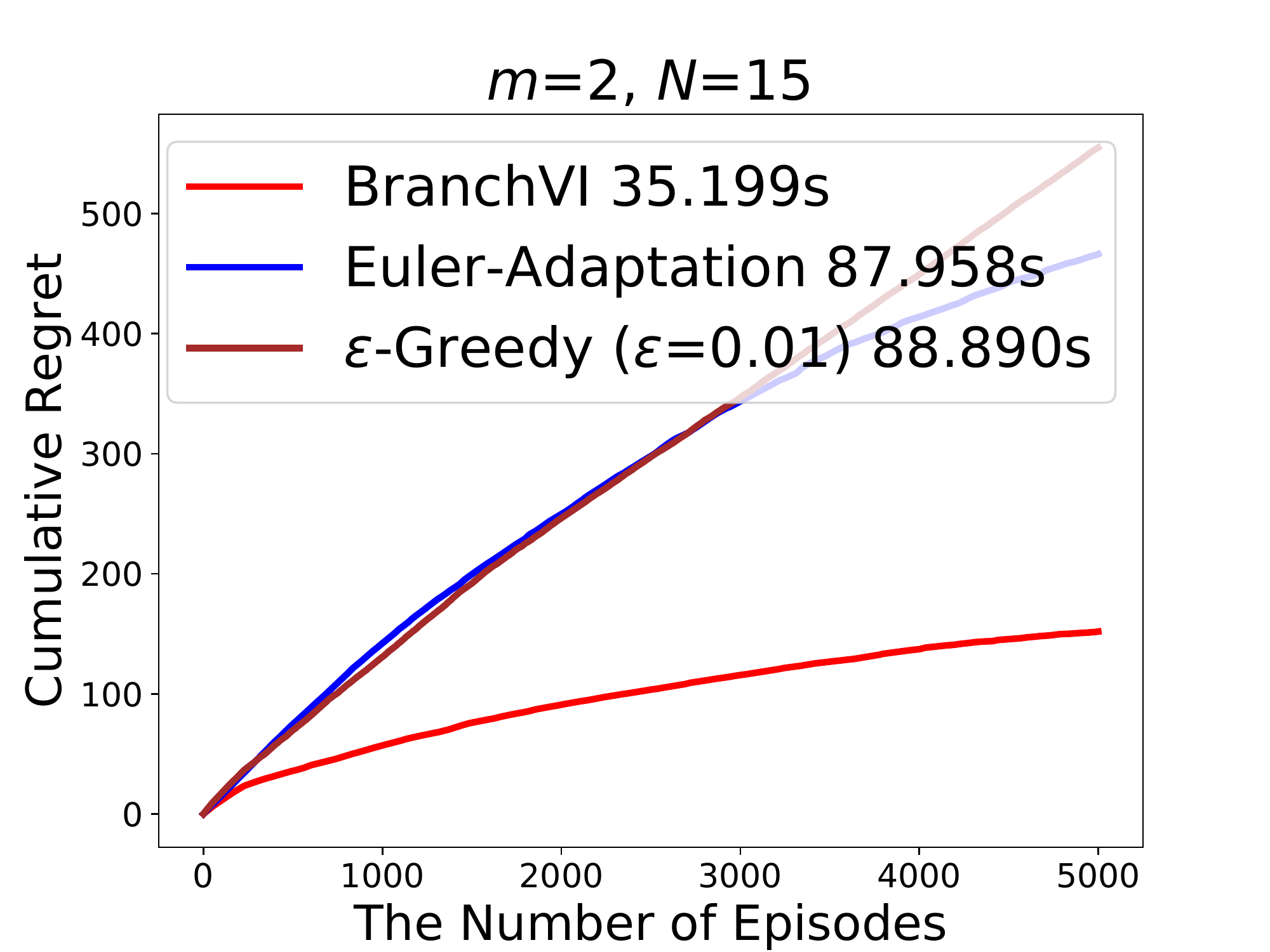} 
	\caption{Experiments for branching RL.}
	\label{fig:experiments}
\end{figure}

\section{Conclusion and Future Work}
In this paper, we formulate a novel branching RL model, and consider both regret minimization and reward-free exploration metrics. Different from standard RL where each episode is a single $H$-step path, branching RL is a tree-structured forward model which allows multiple base actions in a state and multiple successor states.
For branching RL, we build novel fundamental analytical tools and carefully bound the overall variance. We design efficient algorithms and provide near-optimal upper and lower bounds.

There are many interesting directions to explore. One direction is to improve the dependency on $H$ in our regret upper bound (Theorem~\ref{thm:regret_ub}) for small $K$ and close the gap on factors $H,S$ between sample complexity upper and lower bounds (Theorems~\ref{thm:sample_complexity_ub},\ref{thm:sample_complexity_lb}). 
Another direction is to extend branching RL from the tabular setting to function approximation, e.g., representing the value function in a linear form with respect to the feature vectors of state-action pairs~\cite{jin2020provably,zhou2021provably}.

\bibliographystyle{icml2022}
\bibliography{tree-structured_RL_ref}

\OnlyInFull{
\input{tree-structured_RL_supp}

}

\end{document}

%% file: tree-structured_RL_supp.tex
\clearpage
\appendix
\onecolumn
\section*{Appendix}

%

\section{Discussion for Algorithm $\algtteuler$} \label{apx:discussion_bonus}

We remark that it is necessary to separately address triggered reward and triggered transition, and consider trigger and transition as a whole distribution in bonus design (Line~\ref{line:rm_bonus_qpV_ub} in Algorithm~\ref{alg:tt_euler}) and regret analysis (Eq.~\eqref{eq:regret_decomposition}). A counter example is discussed below to support this statement.

If one naively adapts standard RL algorithms~\cite{azar2017minimax,zanette2019tighter}, she/he may directly add bonuses on trigger and transition probabilities, respectively, without separating triggered reward and triggered transition. In this case, $f^k_h(s,a)$ (Line~\ref{line:rm_component_function}) becomes:
\begin{align*}
f^k_h(s,a) \leftarrow & \sbr{\hat{q}^k(s,a)+b^{k,q}(s,a)} \sbr{r(s,a) +  \hat{p}^k(\cdot|s,a)^\top \bar{V}^k_{h+1} + b^{k,pV}(s,a)}
\end{align*}
Then, in regret decomposition, we will obtain
\begin{align*}
& \regret(K) \overset{\textup{(a)}}{\approx} O(1) \sum_{\sigma=\emptyset}^{m^{\oplus(H-1)}} \sum_{\ell=1}^{m} \sum_{(s,a),s\neq s_{\perp}} w_{k\sigma\ell}(s,a) 
\Big[ b^{k,q}(s,a) \Big( r(s,a) + \underbrace{p(\cdot|s,a)^\top \bar{V}^k_{|\sigma \oplus \ell|+1}}_{\mathtt{Term\ \Lambda}} \Big) + b^{k,p\bar{V}}(s,a) q(s,a) \Big],
\end{align*}
where (a) omits second order terms. 
Since $\mathtt{Term\ \Lambda}$ already reaches $\Theta(h)$ order, further summing over all episodes $k$ and steps $h$, we will suffer an extra $\sqrt{H}$ factor in the final result.

One can see from this counter example that, our bonus design and analysis, which separately handle triggered reward and triggered transition and consider trigger and transition as a whole distribution, are sharp and enable a near-optimal regret. 

\section{Proofs for Properties of Branching MDP}

In this section, we present the proofs for structural properties of branching MDP, including branching value difference lemma (Lemma~\ref{lemma:value_diff_main_text}), branching law of total variance (Lemma~\ref{lemma:LTV_main_text}) and the upper bounds of the number of triggered states  (Lemma~\ref{lemma:triggered_nodes_main_text}).

\subsection{Proof of Lemma~\ref{lemma:value_diff_main_text}}

\begin{proof}[Proof of Lemma~\ref{lemma:value_diff_main_text}]
	This proof adapts the analysis of Lemma E.15 in \cite{dann2017unifying} to branching RL.
	According to branching Bellman equations,
	\begin{align*}
	& V'^{\pi}_h(s_{\sigma})-V''^{\pi}_h(s_{\sigma})
	\\
	\overset{\textup{(a)}}{=} & \sum_{\ell=1}^{m} \Big( q'(s_{\sigma},a_{\sigma \oplus \ell}) r(s_{\sigma},a_{\sigma \oplus \ell}) - q''(s_{\sigma},a_{\sigma \oplus \ell}) r(s_{\sigma},a_{\sigma \oplus \ell}) \\& +  q'(s_{\sigma},a_{\sigma \oplus \ell}) p'(s_{\sigma},a_{\sigma \oplus \ell})^\top V'^{\pi}_{h+1} - q''(s_{\sigma},a_{\sigma \oplus \ell}) p''(s_{\sigma},a_{\sigma \oplus \ell})^\top V''^{\pi}_{h+1} \Big)
	\\
	= & \sum_{\ell=1}^{m} \Big( q'(s_{\sigma},a_{\sigma \oplus \ell}) r(s_{\sigma},a_{\sigma \oplus \ell}) - q''(s_{\sigma},a_{\sigma \oplus \ell}) r(s_{\sigma},a_{\sigma \oplus \ell}) \\& +  \sbr{q'(s_{\sigma},a_{\sigma \oplus \ell}) p'(s_{\sigma},a_{\sigma \oplus \ell}) - q''(s_{\sigma},a_{\sigma \oplus \ell}) p''(s_{\sigma},a_{\sigma \oplus \ell}) }^\top V'^{\pi}_{h+1} \\ & + q''(s_{\sigma},a_{\sigma \oplus \ell}) p''(s_{\sigma},a_{\sigma \oplus \ell})^\top \sbr{V'^{\pi}_{h+1} - V''^{\pi}_{h+1}} \Big)
	\\
	= & \sum_{\ell=1}^{m} \Big(q'(s_{\sigma},a_{\sigma \oplus \ell}) r(s_{\sigma},a_{\sigma \oplus \ell}) - q''(s_{\sigma},a_{\sigma \oplus \ell}) r(s_{\sigma},a_{\sigma \oplus \ell}) \\& + \sbr{q'(s_{\sigma},a_{\sigma \oplus \ell}) p'(s_{\sigma},a_{\sigma \oplus \ell}) - q''(s_{\sigma},a_{\sigma \oplus \ell}) p''(s_{\sigma},a_{\sigma \oplus \ell}) }^\top V'^{\pi}_{h+1} \Big) \\ & + \sum_{\ell=1}^{m} \ex_{q'',p'',\pi} \mbr{ V'^{\pi}_{h+1}(s_{\sigma \oplus \ell}) - V''^{\pi}_{h+1}(s_{\sigma \oplus \ell})  | s_h=s} 
	\\
	= & \sum_{\sigma'=\emptyset}^{m} \sum_{\ell=1}^{m} \ex_{q'',p'',\pi} \Big[ q'(s_{\sigma \oplus \sigma'},a_{\sigma \oplus \sigma' \oplus \ell}) r(s_{t\ell},a_{\sigma \oplus \sigma' \oplus \ell}) - q''(s_{t\ell},a_{\sigma \oplus \sigma' \oplus \ell}) r(s_{t\ell},a_{\sigma \oplus \sigma' \oplus \ell}) \\& +  \sbr{q'(s_{t\ell},a_{\sigma \oplus \sigma' \oplus \ell}) p'(s_{t\ell},a_{\sigma \oplus \sigma' \oplus \ell}) - q''(s_{t\ell},a_{\sigma \oplus \sigma' \oplus \ell}) p''(s_{t\ell},a_{\sigma \oplus \sigma' \oplus \ell}) }^\top V'^{\pi}_{h+1} \Big] \\ & + \sum_{\sigma'=1^{\oplus 2}}^{m^{\oplus 2}} \ex_{q'',p'',\pi} \mbr{ V'^{\pi}_{h+2,\ell}(s_{\sigma'}) - V''^{\pi}_{h+2,\ell}(s_{\sigma'})  | s_h} 
	\\
	= & \sum_{\sigma'=\emptyset}^{m^{\oplus(H-h)}} \sum_{\ell=1}^{m} \ex_{q'',p'',\pi} \Big[ q'(s_{\sigma \oplus \sigma'},a_{\sigma \oplus \sigma' \oplus \ell}) r(s_{\sigma \oplus \sigma'},a_{\sigma \oplus \sigma' \oplus \ell}) - q''(s_{\sigma \oplus \sigma'},a_{\sigma \oplus \sigma' \oplus \ell}) r(s_{\sigma \oplus \sigma'},a_{\sigma \oplus \sigma' \oplus \ell}) \\& +  \sbr{q'(s_{\sigma \oplus \sigma'},a_{\sigma \oplus \sigma' \oplus \ell}) p'(s_{\sigma \oplus \sigma'},a_{\sigma \oplus \sigma' \oplus \ell}) - q''(s_{\sigma \oplus \sigma'},a_{\sigma \oplus \sigma' \oplus \ell}) p''(s_{\sigma \oplus \sigma'},a_{\sigma \oplus \sigma' \oplus \ell}) }^\top V'^{\pi}_{|\sigma \oplus \sigma' \oplus \ell|+1} \Big] 
	\end{align*}
\end{proof}

\subsection{Proof of Lemma~\ref{lemma:LTV_main_text}}
\begin{proof}[Proof of Lemma~\ref{lemma:LTV_main_text}]
	First, we prove the equality. This proof adapts the analysis of standard law of total variance in \cite{jin2018q,zanette2019tighter} to branching RL.
	\begin{align*}
	& \ex_{q,p,\pi} \mbr{ \sum_{\sigma=\emptyset}^{m^{\oplus(H-1)}} \sum_{\ell=1}^{m} \var_{s_{\sigma \oplus \ell} \sim q,p}\sbr{V^{\pi}_{|\sigma \oplus \ell|+1}(s_{\sigma \oplus \ell}) | s_{\sigma},a_{\sigma \oplus \ell} } } 
	\\
	= & \ex_{q,p,\pi} \mbr{ \sum_{\sigma=\emptyset}^{m^{\oplus(H-1)}} \sum_{\ell=1}^{m} \sbr{ V^{\pi}_{|\sigma \oplus \ell|+1}(s_{\sigma \oplus \ell}) - q(s_{\sigma},a_{\sigma \oplus \ell}) p(s_{\sigma},a_{\sigma \oplus \ell})^\top V^{\pi}_{|\sigma \oplus \ell|+1} }^2 }
	\\
	= & \ex_{q,p,\pi} \Bigg[ \sum_{\sigma=\emptyset}^{m^{\oplus(H-1)}} \sum_{\ell=1}^{m} \Bigg( q(s_{\sigma},a_{\sigma \oplus \ell}) r(s_{\sigma},a_{\sigma \oplus \ell}) + V^{\pi}_{|\sigma \oplus \ell|+1}(s_{\sigma \oplus \ell}) \\& -  \sbr{q(s_{\sigma},a_{\sigma \oplus \ell}) r(s_{\sigma},a_{\sigma \oplus \ell}) + q(s_{\sigma},a_{\sigma \oplus \ell}) p(s_{\sigma},a_{\sigma \oplus \ell})^\top V^{\pi}_{|\sigma \oplus \ell|+1} } \Bigg)^2 \Bigg]
	\\
	\overset{\textup{(a)}}{=} & \ex_{q,p,\pi} \Bigg[ \Bigg( \sum_{\sigma=\emptyset}^{m^{\oplus(H-1)}} \sum_{\ell=1}^{m} \Big( q(s_{\sigma},a_{\sigma \oplus \ell}) r(s_{\sigma},a_{\sigma \oplus \ell}) + V^{\pi}_{|\sigma \oplus \ell|+1}(s_{\sigma \oplus \ell}) \\&-  \sbr{q(s_{\sigma},a_{\sigma \oplus \ell}) r(s_{\sigma},a_{\sigma \oplus \ell}) + q(s_{\sigma},a_{\sigma \oplus \ell}) p(s_{\sigma},a_{\sigma \oplus \ell})^\top V^{\pi}_{|\sigma \oplus \ell|+1} } \Big) \Bigg)^2 \Bigg]
	\\
	\overset{\textup{(b)}}{=} & \ex_{q,p,\pi} \mbr{ \Bigg( \sum_{\sigma=\emptyset}^{m^{\oplus(H-1)}} \sum_{\ell=1}^{m}  q(s_{\sigma},a_{\sigma \oplus \ell}) r(s_{\sigma},a_{\sigma \oplus \ell}) + \sum_{\sigma=\emptyset}^{m^{\oplus(H-1)}} \sum_{\ell=1}^{m} V^{\pi}_{|\sigma \oplus \ell|+1}(s_{\sigma \oplus \ell}) -  \sum_{\sigma=\emptyset}^{m^{\oplus(H-1)}}V^{\pi}_{|\sigma|+1}(s_{\sigma}) \Bigg)^2  }
	\\
	= & \ex_{q,p,\pi} \mbr{ \sbr{ \sum_{\sigma=\emptyset}^{m^{\oplus(H-1)}} \sum_{\ell=1}^{m}  q(s_{\sigma},a_{\sigma \oplus \ell}) r(s_{\sigma},a_{\sigma \oplus \ell}) -  V^{\pi}_{1}(s_{\emptyset}) }^2 } 
	\end{align*}
	Here (a) comes from the Markov property and that conditioning on the filtration of step $h$, the triggers and transitions of state-base action pairs at step $h+1$ are independent. (b) is due to that $V^{\pi}_{|\sigma|+1}(s_{\sigma})=\sum_{\ell=1}^{m} \sbr{q(s_{\sigma},a_{\sigma \oplus \ell}) r(s_{\sigma},a_{\sigma \oplus \ell}) + q(s_{\sigma},a_{\sigma \oplus \ell}) p(s_{\sigma},a_{\sigma \oplus \ell})^\top V^{\pi}_{|\sigma \oplus \ell|+1}}$ and we can merge the $m$ terms of state-base action value into $V^{\pi}_{|\sigma|+1}(s_{\sigma})$.
	
	Now, we prove the inequality.
	\begin{align*}
	& \ex_{q,p,\pi} \mbr{\sbr{ \sum_{\sigma=\emptyset}^{m^{\oplus(H-1)}} \sum_{\ell=1}^{m}  q(s_{\sigma},a_{\sigma \oplus \ell}) r(s_{\sigma},a_{\sigma \oplus \ell}) -  V^{\pi}_{1}(s_{\emptyset})  }^2} 
	\\
	\leq & \ex_{q,p,\pi} \mbr{\sbr{ \sum_{\sigma=\emptyset}^{m^{\oplus(H-1)}} \sum_{\ell=1}^{m}  q(s_{\sigma},a_{\sigma \oplus \ell}) }^2}
	\\
	= & \ex\mbr{ \sbr{ \sum_{\sigma=\emptyset}^{m^{\oplus(H-1)}} \sum_{\ell=1}^{m} \Big( q(s_{\sigma},a_{\sigma \oplus \ell}) \indicator{s_{\sigma} \neq s_{\perp}} + q(s_{\sigma},a_{\sigma \oplus \ell}) \indicator{s_{\sigma} = s_{\perp}} \Big) }^2}
	\\
	\overset{\textup{(a)}}{\leq} & \ex\mbr{\sbr{ \sum_{\sigma=\emptyset}^{m^{\oplus(H-1)}} \sum_{\ell=1}^{m} \frac{1}{m} \indicator{s_{\sigma} \neq s_{\perp}} }^2}
	\\
	= & \ex\mbr{\sbr{ \sum_{\sigma=\emptyset}^{m^{\oplus(H-1)}}\indicator{s_{\sigma} \neq s_{\perp}} }^2}
	\end{align*}
	where (a) is due to Assumption~\ref{assumption:tri_prob} and $q(s_{\perp}, a)=0$ for any $a \in A^{\univ}$.
	
\end{proof}

\subsection{Proof of Lemma~\ref{lemma:triggered_nodes_main_text}}
\begin{proof}[Proof of Lemma~\ref{lemma:triggered_nodes_main_text}]
	Under Assumption~\ref{assumption:tri_prob}, to bound the total number of triggered (regular) states for any branching MDP and policy $\pi$, it suffices to bound it under a relaxed model $\cM^*$ with $q(s,a)=q^*:=\frac{1}{m}$ for all $(s,a) \in \cS \setminus \{s_{\perp}\} \times A^{\univ}$. 
	Let $\omega_h$ denote the number of triggered states at each step $h$ under $\cM^*$, and $\omega:=\sum_{h=1}^{H} \omega_h$. 
	Below we prove that $\ex[\omega] \leq H$ and $\ex[\omega^2] \leq 3H^2$.

	For $h=1$, $\omega_h=1$ deterministically. 
	For $h \geq 2$, $\omega_h | \omega_{h-1} \sim \binomial(m \omega_{h-1} ,q^*)$. According to the properties of Binomial distribution and $q^*:=\frac{1}{m}$, for $h \geq 2$,
	\begin{align*}
	&\ex \mbr{\omega_h} \hspace*{0.5em} =  mq^* \ex \mbr{\omega_{h-1}} = 1 ,
	\\
	&\ex \mbr{(\omega_h)^2} \!=\!  m q^* (1-q^*) \ex \mbr{\omega_{h-1}}  + m^2 (q^*)^2 \ex \mbr{(\omega_{h-1})^2} 
	\\
	& \hspace*{3.85em} = (1-q^*) + \ex \mbr{(\omega_{h-1})^2} \leq h .
	\end{align*}
	\vspace*{-2em}
	
	Hence, we have that $\ex[\omega]=\sum_{h=1}^{H} \ex [\omega_h]=H$, and 
	\vspace*{-0.5em}
	\begin{align}
	\ex[\omega^2] = & \sum_{h=1}^{H} \ex [(\omega_h)^2] + 2 \sum_{1<i,j<H} \ex [\omega_i \omega_j] 
	\nonumber\\
	\leq & \frac{H(H+1)}{2} + 2 \sum_{1<i,j<H} \ex [\omega_i \omega_j] . \label{eq:apx_sec_moment_omega}
	\end{align}
	\vspace*{-1em}
	
	Now, the challenge falls on how to bound $\ex [\omega_i \omega_j]$ for any $1<i,j<H$.
	Let $W_{\sigma}$ be a Bernoulli random variable denoting whether state $s_{\sigma}$ is triggered for any index string $\sigma$. Then, we can write $\ex [\omega_i \omega_j]$ as 
	\begin{align}
	& \ex \Big[  \big( \sum_{\sigma=1^{\oplus(i-1)}}^{m^{\oplus(i-1)}} W_{\sigma} \big) \cdot  \big(\sum_{\sigma'=1^{\oplus(j-1)}}^{m^{\oplus(j-1)}} W_{\sigma'} \big) \Big]
	\nonumber\\
	\overset{\textup{(a)}}{=} & m^{i-1} \ex \Big[   W_{1^{\oplus(i-1)}}  \big(\sum_{\sigma'=1^{\oplus(j-1)}}^{m^{\oplus(j-1)}} W_{\sigma'} \big) \Big]
	\nonumber\\
	= & m^{i-1} \ex \Big[ W_{1^{\oplus(i-1)}} \big( \hspace*{-1.5em} \sum_{\begin{subarray}{c}
		\sigma'=1^{\oplus(j-1)}
		\\ \sigma' \textup{starts with } 1^{\oplus(i-1)}
		\end{subarray} }^{m^{\oplus(j-1)}} \hspace*{-2em} W_{\sigma'} + \hspace*{-1em} \sum_{\begin{subarray}{c}
		\sigma'=1^{\oplus(j-1)}
		\\ \sigma' \textup{does not start with } 1^{\oplus(i-1)}
		\end{subarray} }^{m^{\oplus(j-1)}} \hspace*{-3.5em} W_{\sigma'} \big) \Big] 
	\nonumber\\
	\overset{\textup{(b)}}{\leq} & m^{i-1}  \sbr{m^{j-i} (q^{*})^{i-1+j-i} + m^{j-1} (q^{*})^{i-1+j-1} }
	\nonumber\\
	= & 2 . \label{eq:apx_omega_ij}
	\end{align}
	\vspace*{-3em}
	
	Here (a) comes from the symmetry of trajectory tree. (b) is due to that at step $j$, the children states of $s_{1^{\oplus(i-1)}}$ have dependency on it and the other states are independent of it, and $\ex [W_{\sigma}W_{\sigma'}] =\Pr[W_{\sigma}=1, W_{\sigma'}=1]$ for any $\sigma,\sigma'$.
	By plugging Eq.~\eqref{eq:apx_omega_ij} into Eq.~\eqref{eq:apx_sec_moment_omega}, we have 
	$
	\ex[\omega^2]
	\leq \frac{H(H+1)}{2} + 4 \frac{H(H-1)}{2} \leq 3H^2
	$. 
	Thus, we obtain Lemma~\ref{lemma:triggered_nodes_main_text}.
\end{proof}

\section{Proofs for Branching RL with Regret Minimization}

In this section, we prove the regret upper bound for algorithm $\algtteuler$ (Theorem~\ref{thm:regret_ub}) and the regret lower bounds for Branching RL-RM in cases with Assumption~\ref{assumption:tri_prob} (Theorem~\ref{thm:regret_lb}) and without Assumption~\ref{assumption:tri_prob} (Theorem~\ref{thm:relax_assumption}).

We first introduce some notations.
Let Bernoulli random variable $X_{k\sigma\ell}(s,a)$ denote whether $(s,a)$ was visited at indices $\sigma$ and $\sigma \oplus \ell$, respectively, in episode $k$, and $w_{k\sigma\ell}(s,a):=\Pr[X_{k\sigma\ell}(s,a)]$.
Let $X_{k}(s,a):=\sum_{\sigma=\emptyset}^{m^{\oplus(H-1)}} \sum_{\ell=1}^{m} X_{k\sigma\ell}(s,a)$ denote the number of times that $(s,a)$ was visited in episode $k$, and $w_k(s,a):=\ex[X_{k}(s,a)] =\sum_{\sigma=\emptyset}^{m^{\oplus(H-1)}} \sum_{\ell=1}^{m} w_{k\sigma\ell}(s,a)$.

Let $n^k_{\sigma\ell}(s,a):=\sum_{k'<k} X_{k'\sigma\ell}(s,a)$ denote the cumulative number of times that $(s,a)$ was visited at indices $\sigma$ and $\sigma \oplus \ell$, respectively, up to episode $k$. Let $n^k(s,a):=\sum_{k'<k} \sum_{\sigma=\emptyset}^{m^{\oplus(H-1)}} \sum_{\ell=1}^{m} X_{k\sigma\ell}(s,a)$ denote the \emph{cumulative} number of times that $(s,a)$ was visited up to episode $k$.

In the following proofs, we make the convention that when $m=1$, $\frac{m^{H+1}-m}{m-1}:=H$. Then, we have $X_k(s,a)\leq m+m^2+\dots+m^H=\frac{m^{H+1}-m}{m-1}$ for any $m \geq 1$.

\subsection{Proof of Regret Upper Bound}

\subsubsection{Concentration}

In the following, we present several important concentration lemmas and define concentration events.

\begin{lemma}[Concentration of Trigger]\label{lemma:con_trigger}
	\begin{align*}
	\Pr\mbr{\abr{ \hat{q}^k(s,a)-q(s,a) } \leq 4 \sqrt{  \frac{\log\sbr{\frac{SNH(m^{H} \vee K)}{\delta'} } }{n^k(s,a)} } , \  \forall (s,a)\in \cS \setminus\{s_{\perp}\} \times A^{\univ}, \forall k \in [K]} \geq 1-2\delta'
	\end{align*}
\end{lemma}
\begin{proof}[Proof of Lemma~\ref{lemma:con_trigger}]
	Since $n^k(s,a) \leq \frac{m^{H+1}-m}{m-1} K$,
	using the Hoeffding inequality with a union bound over $(s,a)$ and $n^k(s,a)$, we have
	\begin{align*}
	\Pr\mbr{\abr{ \hat{q}^k(s,a)-q(s,a) } \leq 2 \sqrt{  \frac{\log\sbr{\frac{SN}{\delta'}\cdot \frac{m^{H+1}-m}{m-1} K} }{n^k(s,a)} } , \  \forall (s,a)\in \cS \setminus\{s_{\perp}\} \times A^{\univ}, \forall k \in [K]} \geq 1-2\delta'
	\end{align*}
	
	If $\frac{m^{H+1}-m}{m-1} \leq K$, then we have 
	$$
	\log\sbr{\frac{SN}{\delta'}\cdot \frac{m^{H+1}-m}{m-1} K} \leq 2\log\sbr{\frac{SNHK}{\delta'} }
	$$
	
	If $\frac{m^{H+1}-m}{m-1} \geq K$, then using $\frac{m^{H+1}-m}{m-1}\leq Hm^{2H}$, we have 
	\begin{align*}
	\log\sbr{\frac{SN}{\delta'}\cdot \frac{m^{H+1}-m}{m-1} K} \leq & 2\log\sbr{\frac{SNHm^{2H}}{\delta'} }
	\\
	\leq & 4\log\sbr{\frac{SNHm^{H}}{\delta'} }
	\end{align*}
	
	Combining the above two cases, we have 
	\begin{align*}
	\log\sbr{\frac{SN}{\delta'}\cdot \frac{m^{H+1}-m}{m-1} K} 
	\leq & 4\log\sbr{\frac{SNH(m^{H} \vee K)}{\delta'} }
	\end{align*}
	
	Therefore, we have
	\begin{align*}
	\Pr\mbr{\abr{ \hat{q}^k(s,a)-q(s,a) } \leq 4 \sqrt{  \frac{\log\sbr{\frac{SNH(m^{H} \vee K)}{\delta'} } }{n^k(s,a)} } , \  \forall (s,a)\in \cS \setminus\{s_{\perp}\} \times A^{\univ}, \forall k \in [K]} \geq 1-2\delta'
	\end{align*}
\end{proof}

\begin{lemma}[Concentration of Triggered Transition]\label{lemma:con_tri_transition}
	\begin{align*}
	 \Pr\Bigg[ &\abr{ \sbr{ \hat{q}^k(s,a) \hat{p}^k(\cdot|s,a) - q(s,a) p(\cdot|s,a) }^\top V^{*}_{h+1} } \leq  4\sqrt{\frac{\var_{s' \sim q,p}\sbr{ V^{*}_{h+1}(s')} \log\sbr{\frac{SNH(m^{H} \vee K)}{\delta'} } }{n^k(s,a)}} \\&+ \frac{ 4H \log\sbr{\frac{SNH(m^{H} \vee K)}{\delta'} } }{n^k(s,a)} , \  \forall (s,a)\in \cS \setminus\{s_{\perp}\} \times A^{\univ}, \forall k \in [K] \Bigg] \geq 1-2\delta'
	\end{align*}
\end{lemma}
\begin{proof}[Proof of Lemma~\ref{lemma:con_tri_transition}]
	Using the similar analytical procedure as  Lemma~\ref{lemma:con_trigger} and the Bernstein's inequality, we can obtain this lemma. 
\end{proof}

\begin{lemma}[Concentration of Variance]\label{lemma:con_variance}
	\begin{align}
	\Pr \Bigg[&
	\abr{ \sqrt{\var_{s' \sim \hat{q},\hat{p}}\sbr{\bar{V}^{k}_{h+1}(s')} } - \sqrt{ \var_{s' \sim q,p}\sbr{ V^{*}_{h+1}(s')} } } \leq \sqrt{\ex_{s' \sim \hat{q},\hat{p}} \mbr{ \sbr{\bar{V}^{k}_{h+1}(s')-V^{*}_{h+1}(s')}^2} } \nonumber\\& \qquad + 8H \sqrt{ \frac{ \log\sbr{\frac{SNH(m^{H} \vee K)}{\delta'} } }{n^k(s,a)} }  , \  \forall (s,a)\in \cS \setminus\{s_{\perp}\} \times A^{\univ}, \forall k \in [K] \Bigg] \geq 1-2\delta'
	\end{align} 
\end{lemma}
\begin{proof}[Proof of Lemma~\ref{lemma:con_variance}]
	Using the similar analytical procedure as  Lemma~\ref{lemma:con_trigger} and Proposition 2 (in particular, Eq.~(53)) in \cite{zanette2019tighter}, we can obtain 
	\begin{align}
		\Pr \Bigg[&
		\abr{ \sqrt{\var_{s' \sim \hat{q},\hat{p}}\sbr{V^{*}_{h+1}(s')} } - \sqrt{ \var_{s' \sim q,p}\sbr{ V^{*}_{h+1}(s')} } }  \leq 8 H \sqrt{ \frac{ \log\sbr{\frac{SNH(m^{H} \vee K)}{\delta'} } }{n^k(s,a)} } ,
		\nonumber\\& \  \forall (s,a)\in \cS \setminus\{s_{\perp}\} \times A^{\univ}, \forall k \in [K] \Bigg] \geq 1-2\delta' \label{eq:dis_hat_var_fix_V}
	\end{align}
	
	With probability $1-2\delta'$, for any $(s,a)\in \cS \setminus\{s_{\perp}\} \times A^{\univ}$ and $k \in [K]$, we have
	\begin{align*}
	& \abr{ \sqrt{\var_{s' \sim \hat{q},\hat{p}}\sbr{\bar{V}^{k}_{h+1}(s')} } - \sqrt{ \var_{s' \sim q,p}\sbr{ V^{*}_{h+1}(s')} } } 
	\\
	\leq & \abr{ \sqrt{\var_{s' \sim \hat{q},\hat{p}}\sbr{\bar{V}^{k}_{h+1}(s')} } - \sqrt{\var_{s' \sim \hat{q},\hat{p}}\sbr{V^{*}_{h+1}(s')} } } \\& + \abr{ \sqrt{\var_{s' \sim \hat{q},\hat{p}}\sbr{V^{*}_{h+1}(s')} } - \sqrt{ \var_{s' \sim q,p}\sbr{ V^{*}_{h+1}(s')} } } 
	\\
	\overset{(a)}{\leq} & \sqrt{\var_{s' \sim \hat{q},\hat{p}}\sbr{ \bar{V}^{k}_{h+1}(s') - V^{*}_{h+1}(s')} } + 8H \sqrt{ \frac{ \log\sbr{\frac{SNH(m^{H} \vee K)}{\delta'} } }{n^k(s,a)} }
	\\
	\leq & \sqrt{\ex_{s' \sim \hat{q},\hat{p}} \mbr{ \sbr{\bar{V}^{k}_{h+1}(s')-V^{*}_{h+1}(s')}^2} }  + 8H \sqrt{ \frac{ \log\sbr{\frac{SNH(m^{H} \vee K)}{\delta'} } }{n^k(s,a)} }
	\end{align*}
	where (a) uses Proposition 2 (in particular, Eqs.~(52)) in \cite{zanette2019tighter} and Eq.~\eqref{eq:dis_hat_var_fix_V}.
	
\end{proof}


To summarize the concentration lemmas used, we define the following concentration events:

\begin{align*}
	\cE_{\textup{tri}} := & \mbr{\abr{ \hat{q}^k(s,a)-q(s,a) } \leq 4 \sqrt{  \frac{\log\sbr{\frac{SNH(m^{H} \vee K)}{\delta'} }}{n^k(s,a)} } , \  \forall (s,a)\in \cS \setminus\{s_{\perp}\} \times A^{\univ}, \forall k \in [K] }
	\\
	\cE_{\textup{trans}} := & \Bigg[ \abr{ \sbr{ \hat{q}^k(s,a) \hat{p}^k(\cdot|s,a) - q(s,a) p(\cdot|s,a) }^\top V^{*}_{h+1} } \leq  4\sqrt{\frac{\var_{s' \sim q,p}\sbr{ V^{*}_{h+1}(s')} \log\sbr{\frac{SNH(m^{H} \vee K)}{\delta'} } }{n^k(s,a)}} \\&+ 4\frac{ H \log\sbr{\frac{SNH(m^{H} \vee K)}{\delta'} } }{n^k(s,a)} , \  \forall (s,a)\in \cS \setminus\{s_{\perp}\} \times A^{\univ}, \forall k \in [K] \Bigg]
	\\
	\cE_{\textup{var}} := & \Bigg[
	\abr{ \sqrt{\var_{s' \sim \hat{q},\hat{p}}\sbr{\bar{V}^{k}_{h+1}(s')} } - \sqrt{ \var_{s' \sim q,p}\sbr{ V^{*}_{h+1}(s')} } } \leq \sqrt{\ex_{s' \sim \hat{q},\hat{p}} \mbr{ \bar{V}^{k}_{h+1}(s')-V^{*}_{h+1}(s')}^2 } \nonumber\\& \qquad + 8H \sqrt{ \frac{ \log\sbr{\frac{SNH(m^{H} \vee K)}{\delta'} } }{n^k(s,a)} }  , \  \forall (s,a)\in \cS \setminus\{s_{\perp}\} \times A^{\univ}, \forall k \in [K] \Bigg] 
	\\
	\cE:=&
	\cE_{\textup{tri}} \cap \cE_{\textup{trans}} \cap \cE_{\textup{var}} 
\end{align*}  

\begin{lemma}\label{lemma:rm_combine_concentration}
	Letting $\delta':=\frac{\delta}{6}$, the concentration event $\cE$ satisfies that
	$$
	\Pr \mbr{\cE} \geq 1-6\delta'=1-\delta
	$$
\end{lemma}
\begin{proof}[Proof of Lemma~\ref{lemma:rm_combine_concentration}]
	We can obtain this lemma by combining Lemmas~\ref{lemma:con_trigger}-\ref{lemma:con_variance}.
\end{proof}


\subsubsection{Visitation} \label{apx:lemmas_visitation}

Below we present an important bound on visitation, which will be used in the proof of Theorem~\ref{thm:regret_ub}.

\begin{lemma}[Regret Bound of Visitation]\label{lemma:bound_visitation}
	Suppose that the concentration event $\cE$ holds. Then, it holds that
	$$
	\sum_{k=1}^{K} \sum_{\sigma=\emptyset}^{m^{\oplus(H-1)}} \sum_{\ell=1}^{m} \sum_{(s,a), s \neq s_{\perp}} \frac{w_{k\sigma\ell}(s,a)}{n^k(s,a)} \leq SN  \log \sbr{ KH } 
	$$
\end{lemma}
\begin{proof}[Proof of Lemma~\ref{lemma:bound_visitation}]
	\begin{align*}
	\sum_{k=1}^{K} \sum_{\sigma=\emptyset}^{m^{\oplus(H-1)}} \sum_{\ell=1}^{m} \sum_{(s,a), s \neq s_{\perp}} \frac{w_{k\sigma\ell}(s,a)}{n^k(s,a)}= & \sum_{k=1}^{K} \sum_{\sigma=\emptyset}^{m^{\oplus(H-1)}} \sum_{\ell=1}^{m} \ex_{(s_{\sigma},a_{\sigma \oplus \ell})\sim \pi^k} \mbr{\frac{1}{n^k(s_{\sigma},a_{\sigma \oplus \ell})} \cdot \indicator{s_{\sigma} \neq s_{\perp}} }
	\\
	= & \sum_{k=1}^{K} \ex_{X_k\sim \pi^k} \mbr{ \sum_{(s,a), s \neq s_{\perp}} 	X_k(s,a) \frac{1}{n^k(s,a)}}
	\\
	= & \ex_{X_k\sim \pi^k} \mbr{ \sum_{(s,a), s \neq s_{\perp}}  \sum_{k=1}^{K}	X_k(s,a) \frac{1}{n^k(s,a)}}	
	\\
	= & \ex_{X_k\sim \pi^k} \mbr{ \sum_{(s,a), s \neq s_{\perp}}  \sum_{k=1}^{K}	X_k(s,a) \frac{1}{ \sum_{k'<k} X_{k'}(s,a)}}
	\\
	\leq & \ex_{X_k\sim \pi^k} \mbr{ \sum_{(s,a), s \neq s_{\perp}}  \log \sbr{ \sum_{k=1}^{K} X_{k}(s,a) } }  
	\\
	\overset{\textup{(a)}}{\leq} & \sum_{(s,a), s \neq s_{\perp}}  \log \sbr{ \ex_{X_k\sim \pi^k} \mbr{\sum_{k=1}^{K} X_{k}(s,a) }}   
	\\
	\overset{\textup{(b)}}{\leq} & SN  \log \sbr{ KH }  ,
	\end{align*}	
	where (a) uses Jensen's inequality, and (b) is due to that for a fixed $(s,a)$ such that $s \neq s_{\perp}$, $ \ex[X_{k}(s,a)] \leq \ex[ \sum_{\sigma=\emptyset}^{m^{\oplus(H-1)}}  \indicator{ s_{\sigma} \neq s_{\perp}} ] \leq H$, since a single base action cannot be chosen twice in a state.

\end{proof}


\subsubsection{Optimism and Pessimism}
Next, we prove the optimism and pessimism of the constructed value functions in algorithm $\algtteuler$, and bound the gap between optimistic and pessimistic value functions.
Recall that $L:=\log\sbr{\frac{SNH(m^{H} \vee K)}{\delta'} }$. 

\begin{lemma}[Optimism]\label{lemma:optimism}
	Suppose that the concentration event $\cE$ holds. Then,
	\begin{align*}
	\underline{V}^{k}_{h}(s) \leq V^{*}_{h}(s) \leq \bar{V}^{k}_{h}(s), \ \forall s \in \cS, h \in [H], k \in [K]
	\end{align*} 
\end{lemma}
\begin{proof}[Proof of Lemma~\ref{lemma:optimism}]
	We prove this lemma by induction. Since $\underline{V}^{k}_{H+1}(s) = V^{*}_{H+1}(s) = \bar{V}^{k}_{H+1}(s)=0, \forall s \in \cS$, it suffices to prove that if $\underline{V}^{k}_{h+1}(s) \leq V^{*}_{h}(s) \leq \bar{V}^{k}_{h+1}(s), \forall s \in \cS$, then $\underline{V}^{k}_{h}(s) \leq V^{*}_{h}(s) \leq \bar{V}^{k}_{h}(s), \forall s \in \cS$.
	
	First, we prove the optimistic direction, i.e., $\bar{V}^{k}_{h}(s) \geq V^{*}_{h}(s) , \ \forall s \in \cS$. 
	In the following, we prove $\bar{Q}^{\pi^k}_h(s,A) \geq Q^{*}_h(s,A)$ for any $s \in \cS$ and $A \in \cA$.
	If $\bar{Q}^{\pi^k}_h(s,A)=H$, then $\bar{Q}^{\pi^k}_h(s,A) = H \geq Q^{*}_h(s,A)$ trivially holds. Otherwise,
	
	\begin{align*}
		\bar{Q}^{\pi^k}_h(s,A) \geq & \sum_{a \in A} \sbr{ \sbr{\hat{q}^k(s,a)+b_k^q(s,a)} r(s,a) + \hat{q}^k(s,a) \hat{p}^k(\cdot|s,a)^\top \bar{V}^{k}_{h+1} + b_k^{qpV}(s,a) }
		\\
		= & \sum_{a \in A} \Bigg( \sbr{\hat{q}^k(s,a)+4\sqrt{\frac{L}{n^k(s,a)}}} r(s,a) + \hat{q}^k(s,a) \hat{p}^k(\cdot|s,a)^\top \bar{V}^{k}_{h+1} + 4\sqrt{\frac{\var_{s' \sim \hat{q},\hat{p}}\sbr{\bar{V}^{k}_{h+1}(s')} L }{n^k(s,a)}} \\& + 4\sqrt{\frac{\ex_{s' \sim \hat{q},\hat{p}} \mbr{\sbr{ \bar{V}^{k}_{h+1}(s')-\underline{V}^{k}_{h+1}(s')}^2} L }{n^k(s,a)}} + \frac{36 H L}{n^k(s,a)} \Bigg)
		\\
		\geq & \sum_{a \in A} \Bigg( q(s,a) r(s,a)  + \hat{q}^k(s,a) \hat{p}^k(\cdot|s,a)^\top \bar{V}^{k}_{h+1} + 4\sqrt{\frac{L }{n^k(s,a)}} \Big( \sqrt{\var_{s' \sim \hat{q},\hat{p}}\sbr{\bar{V}^{k}_{h+1}(s')} } \\& + \sqrt{\ex_{s' \sim \hat{q},\hat{p}} \mbr{\sbr{ \bar{V}^{k}_{h+1}(s')-\underline{V}^{k}_{h+1}(s')}^2}  } + 8H \sqrt{\frac{ L}{n^k(s,a)}} \Big) + \frac{ 4H L}{n^k(s,a)} \Bigg)
		\\
		\overset{\textup{(a)}}{\geq} & \sum_{a \in A} \sbr{ q(s,a) r(s,a) + \hat{q}^k(s,a) \hat{p}^k(\cdot|s,a)^\top V^{*}_{h+1} + 4 \sqrt{\frac{ \var_{s' \sim q,p} \sbr{ V^{*}_{h+1}(s')} L }{n^k(s,a)}} + \frac{4H L}{n^k(s,a)} }
		\\
		\overset{\textup{(b)}}{\geq} & \sum_{a \in A} \sbr{ q(s,a) r(s,a) + q(s,a) p(\cdot|s,a)^\top V^{*}_{h+1}  }
		\\
		= & Q^{*}_h(s,A),
	\end{align*}
	where (a) is due to the induction hypothesis and Lemma~\ref{lemma:con_variance}, and (b) comes from Lemma~\ref{lemma:con_tri_transition}.
	
	Then, we have
	\begin{align*}
		\bar{V}^{k}_{h}(s)=\bar{Q}^{k}_h(s,\pi^k(s)) \geq \bar{Q}^{k}_h(s,\pi^*(s)) \geq Q^{*}_h(s,\pi^*(s)) = V^{*}_h(s)
	\end{align*}
	
	Now, we prove the pessimistic direction, i.e., $\underline{V}^{k}_{h}(s) \leq V^{*}_{h}(s) , \ \forall s \in \cS$. If $\underline{V}^{k}_{h}(s)=0$, then $\underline{V}^{k}_{h}(s)=0 \leq V^{*}_{h}(s)$ trivially holds. Otherwise,
	\begin{align*}
	\underline{V}^{k}_{h}(s) = & \sum_{a \in A} \sbr{ \sbr{\hat{q}^k(s,a)-b_k^q(s,a)} r(s,a) + \hat{q}^k(s,a) \hat{p}^k(\cdot|s,a)^\top \underline{V}^{k}_{h+1} - b_k^{qpV}(s,a) }
	\\
	= & \sum_{a \in \pi^k(s)} \Bigg( \sbr{\hat{q}^k(s,a)-4\sqrt{\frac{L}{n^k(s,a)}}} r(s,a) + \hat{q}^k(s,a) \hat{p}^k(\cdot|s,a)^\top \underline{V}^{k}_{h+1} - 4\sqrt{\frac{\var_{s' \sim \hat{q},\hat{p}}\sbr{ \bar{V}^{k}_{h+1}(s')} L }{n^k(s,a)}} \\& - 4\sqrt{\frac{\ex_{s' \sim \hat{q},\hat{p}} \mbr{\sbr{ \bar{V}^{k}_{h+1}(s')-\underline{V}^{k}_{h+1}(s')}^2} L }{n^k(s,a)}} - \frac{36 H L}{n^k(s,a)} \Bigg)
	\\
	\leq & \sum_{a \in \pi^k(s)} \Bigg( q(s,a) r(s,a)  + \hat{q}^k(s,a) \hat{p}^k(\cdot|s,a)^\top \underline{V}^{k}_{h+1} - 4\sqrt{\frac{L }{n^k(s,a)}} \Big( \sqrt{\var_{s' \sim \hat{q},\hat{p}}\sbr{ \bar{V}^{k}_{h+1}(s')} } \\& + \sqrt{\ex_{s' \sim \hat{q},\hat{p}} \mbr{\sbr{ \bar{V}^{k}_{h+1}(s')-\underline{V}^{k}_{h+1}(s')}^2}  } + 8H \sqrt{\frac{ L}{n^k(s,a)}} \Big) - \frac{ 4H L}{n^k(s,a)} \Bigg)
	\\
	\leq & \sum_{a \in \pi^k(s)} \sbr{ q(s,a) r(s,a) + \hat{q}^k(s,a) \hat{p}^k(\cdot|s,a)^\top V^{*}_{h+1} - 4\sqrt{\frac{ \var_{s' \sim q,p} \sbr{ V^{*}_{h+1}(s')} L }{n^k(s,a)}} - \frac{4H L}{n^k(s,a)} }
	\\
	\leq & \sum_{a \in \pi^k(s)} \sbr{ q(s,a) r(s,a) + q(s,a) p(\cdot|s,a)^\top V^{*}_{h+1}  }
	\\
	= & Q^{*}_h(s,\pi^k(s))
	\\
	\leq & Q^{*}_h(s,\pi^*(s))
	\\
	= & V^{*}_h(s)
	\end{align*}
	
\end{proof}

\begin{lemma}[Gap between Optimism and Pessimism]\label{lemma:V_gap}
	Suppose that the concentration event $\cE$ holds. Then, it holds that
	\begin{align*}
		\bar{V}^{\pi^k}_h(s) \leq  \sum_{\sigma'=\emptyset}^{m^{\oplus(H-h)}} \sum_{\ell=1}^{m}  \ex \mbr{170HL\sqrt{\frac{S}{n^k(s_{\sigma \oplus \sigma'},a_{\sigma \oplus \sigma' \oplus \ell})}} \cdot \indicator{ s_{\sigma \oplus \sigma'} \neq s_{\perp} }  \Big| s_{\sigma}=s, \pi^k} .
	\end{align*}
	In particular,
	\begin{align*}
		\bar{V}^{\pi^k}_1(s) - \underline{V}^{\pi^k}_1(s) \leq \sum_{\sigma=\emptyset}^{m^{\oplus(H-1)}} \sum_{\ell=1}^{m} \ex \mbr{170HL\sqrt{\frac{S}{n^k(s_{\sigma},a_{\sigma \oplus \ell})}} \cdot \indicator{ s_{\sigma} \neq s_{\perp} }  \Big| s_{\emptyset}=s, \pi^k}
	\end{align*}
\end{lemma}
\begin{proof}[Proof of Lemma~\ref{lemma:V_gap}]
	According to the construction of optimistic and pessimistic value functions, we have
	\begin{equation}
	\left\{
	\begin{aligned}
	\bar{V}^{\pi^k}_h(s) \leq & \sum_{a \in A} \sbr{ \sbr{\hat{q}^k(s,a)+b_k^q(s,a)} r(s,a) + \hat{q}^k(s,a) \hat{p}^k(\cdot|s,a)^\top \bar{V}^{k}_{h+1} + b_k^{qpV}(s,a) }
	\\
	\underline{V}^{\pi^k}_h(s) \geq & \sum_{a \in A} \sbr{ \sbr{\hat{q}^k(s,a)-b_k^q(s,a)} r(s,a) + \hat{q}^k(s,a) \hat{p}^k(\cdot|s,a)^\top \underline{V}^{k}_{h+1} - b_k^{qpV}(s,a) }
	\end{aligned}
	\right.
	\end{equation}
	
	Then,
	\begin{align*}
		\bar{V}^{\pi^k}_h(s) - \underline{V}^{\pi^k}_h(s) \leq & \sum_{a \in A} \sbr{ 2 b_k^q(s,a) r(s,a) + \hat{q}^k(s,a) \hat{p}^k(\cdot|s,a)^\top \sbr{\bar{V}^{k}_{h+1} - \underline{V}^{k}_{h+1}} + 2b_k^{qpV}(s,a)  }
		\\
		\leq & \sum_{a \in A} \sbr{ 2 b_k^q(s,a) r(s,a) + 2b_k^{qpV}(s,a)  + \hat{q}^k(s,a) \hat{p}^k(\cdot|s,a)^\top \sbr{\bar{V}^{k}_{h+1} - \underline{V}^{k}_{h+1}} }
		\\
		\leq & \sum_{a \in A} \Bigg( 8 \sqrt{\frac{L}{n^k(s,a)}} r(s,a) + 8 \sqrt{\frac{\var_{s' \sim \hat{q},\hat{p}}\sbr{\bar{V}^{k}_{h+1}(s')} L }{n^k(s,a)}}  \\ & \!\!+ 8\sqrt{\frac{\ex_{s' \sim \hat{q},\hat{p}} \mbr{\sbr{ \bar{V}^{k}_{h+1}(s')-\underline{V}^{k}_{h+1}(s')}^2} L }{n^k(s,a)}} \!+\! \frac{72 H L}{n^k(s,a)}  \!+\! \hat{q}^k(s,a) \hat{p}^k(\cdot|s,a)^\top \sbr{\bar{V}^{k}_{h+1} - \underline{V}^{k}_{h+1}} \Bigg)
		\\
		\leq & \sum_{a \in A} \Bigg( 8 \sqrt{\frac{L}{n^k(s,a)}} r(s,a) + 8 \sqrt{\frac{\var_{s' \sim q,p}\sbr{ V^{*}_{h+1}(s')} L }{n^k(s,a)}}  \\ & + 16 \sqrt{\frac{\ex_{s' \sim \hat{q},\hat{p}} \mbr{\sbr{ \bar{V}^{k}_{h+1}(s')-\underline{V}^{k}_{h+1}(s')}^2} L }{n^k(s,a)}} + \frac{136 H L}{n^k(s,a)}  + q(s,a) p(\cdot|s,a)^\top \sbr{\bar{V}^{k}_{h+1} - \underline{V}^{k}_{h+1}} \\& + \sbr{\hat{q}^k(s,a) \hat{p}^k(\cdot|s,a) - q(s,a) p(\cdot|s,a) }^\top \sbr{\bar{V}^{k}_{h+1} - \underline{V}^{k}_{h+1}} \Bigg)
		\\
		\leq & \sum_{a \in A} \Bigg( 8 \sqrt{\frac{L}{n^k(s,a)}} r(s,a) + 8 \sqrt{\frac{\var_{s' \sim q,p}\sbr{ V^{*}_{h+1}(s')} L }{n^k(s,a)}}  \\ & + 16 \sqrt{\frac{\ex_{s' \sim \hat{q},\hat{p}} \mbr{\sbr{ \bar{V}^{k}_{h+1}(s')-\underline{V}^{k}_{h+1}(s')}^2} L }{n^k(s,a)}} + \frac{136 H L}{n^k(s,a)}  + q(s,a) p(\cdot|s,a)^\top \sbr{\bar{V}^{k}_{h+1} - \underline{V}^{k}_{h+1}} \\& + \| \hat{q}^k(s,a) \hat{p}^k(\cdot|s,a) - q(s,a) p(\cdot|s,a) \|_1 \| \bar{V}^{k}_{h+1} - \underline{V}^{k}_{h+1} \|_{\infty} \Bigg)
		\\
		\leq & \sum_{a \in A} \Bigg( 8 \sqrt{\frac{L}{n^k(s,a)}} r(s,a) + 8 \sqrt{\frac{\var_{s' \sim q,p}\sbr{ V^{*}_{h+1}(s')} L }{n^k(s,a)}}  \\ & + 16 \sqrt{\frac{\ex_{s' \sim \hat{q},\hat{p}} \mbr{\sbr{ \bar{V}^{k}_{h+1}(s')-\underline{V}^{k}_{h+1}(s')}^2} L }{n^k(s,a)}} + \frac{136 H L}{n^k(s,a)}  + q(s,a) p(\cdot|s,a)^\top \sbr{\bar{V}^{k}_{h+1} - \underline{V}^{k}_{h+1}} \\& + H\sqrt{\frac{2SL}{n^k(s,a)}} \Bigg)
		\\
		\leq & \sum_{a \in A} \Bigg( 170HL\sqrt{\frac{S}{n^k(s,a)}} + q(s,a) p(\cdot|s,a)^\top \sbr{\bar{V}^{k}_{h+1} - \underline{V}^{k}_{h+1}} \Bigg)
		\\
		\leq & \sum_{\sigma'=\emptyset}^{m^{\oplus(H-h)}} \sum_{\ell=1}^{m}  \ex \mbr{170HL\sqrt{\frac{S}{n^k(s_{\sigma \oplus \sigma'},a_{\sigma \oplus \sigma' \oplus \ell})}} \cdot \indicator{ s_{\sigma \oplus \sigma'} \neq s_{\perp} } \Big| s_{\sigma}=s, \pi^k} 
	\end{align*}
	Thus,
	\begin{align*}
		\bar{V}^{\pi^k}_1(s) - \underline{V}^{\pi^k}_1(s) \leq \sum_{\sigma=\emptyset}^{m^{\oplus(H-1)}} \sum_{\ell=1}^{m} \ex \mbr{170HL\sqrt{\frac{S}{n^k(s_{\sigma},a_{\sigma \oplus \ell})}} \cdot \indicator{ s_{\sigma} \neq s_{\perp} }  \Big| s_{\emptyset}=s, \pi^k}
	\end{align*}
\end{proof}

\begin{lemma}[Cumulative Gap between Optimism and Pessimism]\label{lemma:cumulative_V_gap}
	Suppose that the concentration event $\cE$ holds. Then, it holds that
	\begin{align*}
		\sum_{k=1}^{K} \sum_{\sigma=\emptyset}^{m^{\oplus(H-1)}} \sum_{\ell=1}^{m} \sum_{(s,a),s\neq s_{\perp}} w_{k\sigma\ell}(s,a) \sum_{s'} q(s,a) p(s'|s,a)  \sbr{\bar{V}^{k}_{h+1}(s')-\underline{V}^{k}_{h+1}(s')}^2  \leq 28900 m H^4L^3 S^3 N^2    
	\end{align*}
\end{lemma}
\begin{proof}[Proof of Lemma~\ref{lemma:cumulative_V_gap}]
	\begin{align*}
		& \sum_{k=1}^{K} \sum_{\sigma=\emptyset}^{m^{\oplus(H-1)}} \sum_{\ell=1}^{m} \sum_{(s,a),s\neq s_{\perp}} w_{k\sigma\ell}(s,a) \sum_{s' \neq s_{\perp}} q(s,a) p(s'|s,a)  \sbr{\bar{V}^{k}_{h+1}(s')-\underline{V}^{k}_{h+1}(s')}^2 
		\\
		= & \sum_{k=1}^{K} \sum_{\sigma=\emptyset}^{m^{\oplus(H-1)}} \sum_{\ell=1}^{m} \sum_{(s,a),s\neq s_{\perp}} \sum_{s' \neq s_{\perp}} w_{k\sigma\ell}(s,a)  q(s,a) p(s'|s,a)  \sbr{\bar{V}^{k}_{h+1}(s')-\underline{V}^{k}_{h+1}(s')}^2
		\\
		= & \sum_{k=1}^{K} \sum_{\sigma=\emptyset}^{m^{\oplus(H-1)}} \sum_{\ell=1}^{m} \sum_{(s,a),s\neq s_{\perp}} \sum_{s' \neq s_{\perp}} w_{k\sigma\ell}(s',s,a)  \sbr{\bar{V}^{k}_{h+1}(s')-\underline{V}^{k}_{h+1}(s')}^2
		\\
		= & \sum_{k=1}^{K} \sum_{\sigma=\emptyset}^{m^{\oplus(H-1)}} \sum_{\ell=1}^{m} \sum_{s' \neq s_{\perp}} \tilde{w}_{k\sigma\ell}(s')  \sbr{\bar{V}^{k}_{h+1}(s')-\underline{V}^{k}_{h+1}(s')}^2
		\\
		= & \sum_{k=1}^{K} \sum_{\sigma=\emptyset}^{m^{\oplus(H-1)}} \sum_{\ell=1}^{m} \ex_{s_{\sigma \oplus \ell} \sim \pi^k} \mbr{\sbr{ \bar{V}^{k}_{h+1}(s_{\sigma \oplus \ell})-\underline{V}^{k}_{h+1}(s_{\sigma \oplus \ell}) }^2}
		\\
		\leq & \sum_{k=1}^{K} \sum_{\sigma=\emptyset}^{m^{\oplus(H-1)}} \ex_{s_{\sigma} \sim \pi^k} \mbr{\sbr{ \bar{V}^{k}_{h}(s_{\sigma})-\underline{V}^{k}_{h}(s_{\sigma}) }^2}
		\\
		= & \sum_{k=1}^{K} \sum_{\sigma=\emptyset}^{m^{\oplus(H-1)}} \ex_{s_{\sigma} \sim \pi^k} \mbr{\sbr{ \sbr{\bar{V}^{k}_{h}(s_{\sigma})-\underline{V}^{k}_{h}(s_{\sigma})} \cdot \indicator{s_{\sigma} \neq s_{\perp}} }^2}
		\\
		\overset{\textup{(a)}}{\leq} & \sum_{k=1}^{K} \!\! \sum_{\sigma=\emptyset}^{m^{\oplus(H-1)}} \!\! \sum_{\ell=1}^{m} \ex_{s_{\sigma} \sim \pi^k} \mbr{ \sbr{\indicator{s_{\sigma} \neq s_{\perp}} \cdot \!\!\!\!\!\! \sum_{\sigma'=\emptyset}^{m^{\oplus(H-h)}} \!\! \sum_{\ell=1}^{m}  \ex \mbr{170HL\sqrt{\frac{S}{n^k(s_{\sigma \oplus \sigma'},a_{\sigma \oplus \sigma' \oplus \ell})}} \cdot \indicator{s_{\sigma \oplus \sigma'} \neq s_{\perp}} \Big| s_{\sigma}, \pi^k}  }^2}
		\\
		\leq & 28900 H^2L^2 \sum_{k=1}^{K} \sum_{\sigma=\emptyset}^{m^{\oplus(H-1)}} \sum_{\ell=1}^{m} \ex_{s_{\sigma} \sim \pi^k} \mbr{\sbr{ \indicator{s_{\sigma} \neq s_{\perp}} \cdot \ex \mbr{ \sum_{(s,a),s\neq s_{\perp}} X_k(s,a) \sqrt{\frac{S}{n^k(s,a)}} }  }^2}
		\\
		= & 28900 H^2L^2 \sum_{k=1}^{K} \sum_{\sigma=\emptyset}^{m^{\oplus(H-1)}} \sum_{\ell=1}^{m} \ex_{s_{\sigma} \sim \pi^k} \mbr{ \sbr{\indicator{s_{\sigma} \neq s_{\perp}}}^2 \cdot \sbr{\ex \mbr{ \sum_{(s,a),s\neq s_{\perp}} X_k(s,a) \sqrt{\frac{S}{n^k(s,a)}} } }^2 }
		\\
		= & 28900 H^2L^2 \sum_{k=1}^{K} \sbr{\ex \mbr{ \sum_{(s,a),s\neq s_{\perp}} X_k(s,a) \sqrt{\frac{S}{n^k(s,a)}} }}^2 \sum_{\sigma=\emptyset}^{m^{\oplus(H-1)}} \sum_{\ell=1}^{m} \ex_{s_{\sigma} \sim \pi^k} \mbr{ \sbr{\indicator{s_{\sigma} \neq s_{\perp}} }^2}
		\\
		\leq & 28900 H^3L^2  \sum_{k=1}^{K} \sbr{\ex \mbr{ \sum_{(s,a),s\neq s_{\perp}} X_k(s,a)}}^2 \frac{S}{n^k(s,a)}
		\\
		\overset{\textup{(b)}}{\leq} & 28900 m H^4L^2  \sum_{k=1}^{K} \ex \mbr{ \sum_{(s,a),s\neq s_{\perp}} X_k(s,a)} \frac{S}{n^k(s,a)}
		\\
		= & 28900 m H^4L^2 S  \ex \mbr{ \sum_{(s,a),s\neq s_{\perp}} \sum_{k=1}^{K} X_k(s,a) \frac{1}{n^k(s,a)} } 
		\\
 		\overset{\textup{(c)}}{\leq} & 28900 m H^4L^3 S^2 N  
	\end{align*}
	where (a) uses Lemma~\ref{lemma:V_gap}, (b) is due to $\ex \mbr{ \sum_{(s,a),s\neq s_{\perp}} X_k(s,a)} \leq m \ex[\sum_{\sigma=\emptyset}^{m^{\oplus(H-1)}}  \indicator{ s_{\sigma} \neq s_{\perp}} ] \leq mH$ and (c) comes from Lemma~\ref{lemma:bound_visitation}.
	
	
\end{proof}

\subsubsection{Proof of Theorem~\ref{thm:regret_ub}}

Now we prove the regret upper bound (Theorem~\ref{thm:regret_ub}) for algorithm $\algtteuler$.

\begin{proof}[Proof of Theorem~\ref{thm:regret_ub}]
	Suppose that the concentration event $\cE$ holds. 
	
	For any $k \in [K]$ and $(s,a) \in \cS \setminus \{s_{\perp}\} \times A^{\univ}$, let $\tilde{q}^k(s,a):=\hat{q}^k(s,a)+b_k^q(s,a)$, $\tilde{q}^k(s,a)	\tilde{p}^k(\cdot|s,a)^\top \bar{V}^{k}_{h+1}:=\hat{q}^k(s,a)\hat{p}^k(\cdot|s,a)^\top \bar{V}^{k}_{h+1} + b_k^{qpV}(s,a)$.
	
	\paragraph{Step 1: Regret decomposition.}
	
	Using Lemma~\ref{lemma:value_diff_main_text}, we can decompose $\regret(K)$ as follows:
	\begin{align}
		\regret(K)= & \sum_{k=1}^{K} \sbr{ V^{*}_1(s)-V^{\pi^k}_1(s) }
		\nonumber\\
		= & \sum_{k=1}^{K} \sbr{ \bar{V}^{k}_1(s)-V^{\pi^k}_1(s) }
		\nonumber\\
		= & \sum_{k=1}^{K} \sum_{\sigma=\emptyset}^{m^{\oplus(H-1)}} \sum_{\ell=1}^{m} \ex \Big[ \sbr{ \tilde{q}^k(s_{\sigma}, a_{\sigma \oplus \ell}) - q(s_{\sigma}, a_{\sigma \oplus \ell}) } r(s_{\sigma}, a_{\sigma \oplus \ell}) \nonumber\\&+ \sbr{ \tilde{q}^k(s_{\sigma}, a_{\sigma \oplus \ell}) \tilde{p}^k(\cdot|s_{\sigma}, a_{\sigma \oplus \ell}) - q(s_{\sigma}, a_{\sigma \oplus \ell}) p(\cdot|s_{\sigma}, a_{\sigma \oplus \ell}) }^\top \bar{V}^{k}_{h+1}  \Big]
		\nonumber\\
		= & \sum_{k=1}^{K} \sum_{\sigma=\emptyset}^{m^{\oplus(H-1)}} \sum_{\ell=1}^{m} \sum_{(s,a),s\neq s_{\perp}} w_{k\sigma\ell}(s,a) \cdot \nonumber\\& \mbr{ \sbr{ \tilde{q}^k(s,a) - q(s,a) } r(s,a) + \sbr{ \tilde{q}^k(s,a) \tilde{p}^k(\cdot|s,a) - q(s,a) p(\cdot|s,a) }^\top \bar{V}^{k}_{h+1}  }
		\nonumber\\
		= & \sum_{k=1}^{K} \sum_{\sigma=\emptyset}^{m^{\oplus(H-1)}} \sum_{\ell=1}^{m} \sum_{(s,a),s\neq s_{\perp}} w_{k\sigma\ell}(s,a) \Big[ \underbrace{ \sbr{ \tilde{q}^k(s,a) - q(s,a) } r(s,a) }_{\term{ 1}} \nonumber\\ & + \underbrace{\sbr{ \tilde{q}^k(s,a) \tilde{p}^k(\cdot|s,a) - \hat{q}^k(s,a) \hat{p}^k(\cdot|s,a) }^\top \bar{V}^{k}_{h+1}}_{\term{ 2}}  +  \underbrace{\sbr{ \hat{q}^k(s,a) \hat{p}^k(\cdot|s,a) - q(s,a) p(\cdot|s,a) }^\top V^{*}_{h+1}}_{\term{ 3}} \nonumber\\ & + \underbrace{ \sbr{ \hat{q}^k(s,a) \hat{p}^k(\cdot|s,a) - q(s,a) p(\cdot|s,a) }^\top \sbr{\bar{V}^{k}_{h+1}-V^{*}_{h+1}} }_{\term{ 4}}  \Big] \label{eq:regret_decomposition}
	\end{align}
	
	\paragraph{Step 2: Bound the bonus term for triggered rewards -- $\term{ 1}$.}
	
	\begin{align}
		\term{ 1} = & \sum_{k=1}^{K} \sum_{\sigma=\emptyset}^{m^{\oplus(H-1)}} \sum_{\ell=1}^{m} \sum_{(s,a),s\neq s_{\perp}} w_{k\sigma\ell}(s,a) \sbr{ \sbr{ \tilde{q}^k(s,a) - q(s,a) } r(s,a) }
		\nonumber\\
		= & \sum_{k=1}^{K} \sum_{\sigma=\emptyset}^{m^{\oplus(H-1)}} \sum_{\ell=1}^{m} \sum_{(s,a),s\neq s_{\perp}} w_{k\sigma\ell}(s,a) \sbr{ \sbr{ \hat{q}^k(s,a) + 4 \sqrt{\frac{L}{n^k(s,a)}} - q(s,a) } r(s,a) }
		\nonumber\\
		= & 8 \sum_{k=1}^{K} \sum_{\sigma=\emptyset}^{m^{\oplus(H-1)}} \sum_{\ell=1}^{m} \sum_{(s,a),s\neq s_{\perp}} w_{k\sigma\ell}(s,a)   \sqrt{\frac{L}{n^k(s,a)}} r(s,a) 
		\nonumber\\
		\leq & 8  \sqrt{L} \sqrt{ \sum_{k=1}^{K} \sum_{\sigma=\emptyset}^{m^{\oplus(H-1)}} \sum_{\ell=1}^{m} \sum_{(s,a),s\neq s_{\perp}} \frac{ w_{k\sigma\ell}(s,a) }{n^k(s,a)} }  \sqrt{ \sum_{k=1}^{K} \sum_{\sigma=\emptyset}^{m^{\oplus(H-1)}} \sum_{\ell=1}^{m} \sum_{(s,a),s\neq s_{\perp}} w_{k\sigma\ell}(s,a)  r^2(s,a) }
		\nonumber\\
		\leq & 8  L  \sqrt{ SN \sum_{k=1}^{K} \sum_{\sigma=\emptyset}^{m^{\oplus(H-1)}} \sum_{\ell=1}^{m} \ex \mbr{\indicator{s^k_{\sigma\ell} \neq s_{\perp}} } }
		\nonumber\\
		\overset{(a)}{\leq} & 8 L \sqrt{m SNHK } , \label{eq:term1}
	\end{align}
	where (a) uses Lemma~\ref{lemma:triggered_nodes_main_text}.
	
	\paragraph{Step 3: Bound the bonus term for triggered future values -- $\term{ 2}$.}
	
	\begin{align}
	\term{ 2} = &  \sum_{k=1}^{K} \sum_{\sigma=\emptyset}^{m^{\oplus(H-1)}} \sum_{\ell=1}^{m} \sum_{(s,a),s\neq s_{\perp}} w_{k\sigma\ell}(s,a) \sbr{ \tilde{q}^k(s,a) \tilde{p}^k(\cdot|s,a) - \hat{q}^k(s,a) \hat{p}^k(\cdot|s,a) }^\top \bar{V}^{k}_{h+1}
	\nonumber\\
	= & \! \sum_{k=1}^{K} \!\! \sum_{\sigma=\emptyset}^{m^{\oplus(H-1)}}\!\! \sum_{\ell=1}^{m} \sum_{(s,a),s\neq s_{\perp}} w_{k\sigma\ell}(s,a) \Bigg(  4\sqrt{\frac{\var_{s' \sim \hat{q},\hat{p}}\sbr{\bar{V}^{k}_{h+1}(s')} L }{n^k(s,a)}} + 4\sqrt{\frac{\ex_{s' \sim \hat{q},\hat{p}} \mbr{\sbr{ \bar{V}^{k}_{h+1}(s')-\underline{V}^{k}_{h+1}(s')}^2} L }{n^k(s,a)}} \nonumber\\ & + \frac{36 H L}{n^k(s,a)} \Bigg)
	\nonumber\\
	\overset{\textup{(a)}}{\leq} &  \sum_{k=1}^{K} \!\! \sum_{\sigma=\emptyset}^{m^{\oplus(H-1)}} \!\! \sum_{\ell=1}^{m} \sum_{(s,a),s\neq s_{\perp}} w_{k\sigma\ell}(s,a) \Bigg(  4\sqrt{\frac{ \var_{s' \sim q,p}\sbr{ V^{*}_{h+1}(s')} L }{n^k(s,a)}} + 8 \sqrt{\frac{\ex_{s' \sim \hat{q},\hat{p}} \mbr{\sbr{ \bar{V}^{k}_{h+1}(s')-\underline{V}^{k}_{h+1}(s')}^2} L }{n^k(s,a)}} \nonumber\\ & + \frac{68 H L}{n^k(s,a)} \Bigg)
	\nonumber\\
	= & 4 \sum_{k=1}^{K} \!\! \sum_{\sigma=\emptyset}^{m^{\oplus(H-1)}}\!\!\! \sum_{\ell=1}^{m} \sum_{(s,a),s\neq s_{\perp}} w_{k\sigma\ell}(s,a)  \sqrt{\frac{ \var_{s' \sim q,p}\sbr{ V^{*}_{h+1}(s')} L }{n^k(s,a)}}  + 68 H L \sum_{k=1}^{K} \sum_{\sigma=\emptyset}^{m^{\oplus(H-1)}} \sum_{\ell=1}^{m} \sum_{(s,a),s\neq s_{\perp}} \frac{w_{k\sigma\ell}(s,a)}{n^k(s,a)}  \nonumber\\ & + 8 \sum_{k=1}^{K} \sum_{\sigma=\emptyset}^{m^{\oplus(H-1)}} \sum_{\ell=1}^{m} \sum_{(s,a),s\neq s_{\perp}} w_{k\sigma\ell}(s,a) \sqrt{\frac{\ex_{s' \sim \hat{q},\hat{p}} \mbr{\sbr{ \bar{V}^{k}_{h+1}(s')-\underline{V}^{k}_{h+1}(s')}^2} L }{n^k(s,a)}}
	\nonumber\\
	\leq & 4 \sqrt{L} \sqrt{ \sum_{k=1}^{K} \sum_{\sigma=\emptyset}^{m^{\oplus(H-1)}} \sum_{\ell=1}^{m} \sum_{(s,a),s\neq s_{\perp}} \frac{w_{k\sigma\ell}(s,a)}{n^k(s,a)} } \sqrt{ \sum_{k=1}^{K} \sum_{\sigma=\emptyset}^{m^{\oplus(H-1)}} \sum_{\ell=1}^{m} \sum_{(s,a),s\neq s_{\perp}} w_{k\sigma\ell}(s,a)  \var_{s' \sim q,p}\sbr{ V^{*}_{h+1}(s')}  } \nonumber\\& + 68 H L \sum_{k=1}^{K} \sum_{\sigma=\emptyset}^{m^{\oplus(H-1)}} \sum_{\ell=1}^{m} \sum_{(s,a),s\neq s_{\perp}} \frac{w_{k\sigma\ell}(s,a)}{n^k(s,a)}  \nonumber\\& + 8 \sum_{k=1}^{K} \sum_{\sigma=\emptyset}^{m^{\oplus(H-1)}} \sum_{\ell=1}^{m} \sum_{(s,a),s\neq s_{\perp}} w_{k\sigma\ell}(s,a) \sqrt{\frac{\ex_{s' \sim \hat{q},\hat{p}} \mbr{\sbr{ \bar{V}^{k}_{h+1}(s')-\underline{V}^{k}_{h+1}(s')}^2} L }{n^k(s,a)}}
	\nonumber\\
	\overset{(b)}{\leq} &  4 L \sqrt{SN \sum_{k=1}^{K} \ex_{\pi^k} \mbr{\sum_{h=1}^{H} \sum_{\ell=1}^{m^h} \var_{s' \sim q,p}\sbr{ V^{*}_{h+1}(s')} } } + 68 SNH L^2 \nonumber\\ & + 8 \sum_{k=1}^{K} \sum_{\sigma=\emptyset}^{m^{\oplus(H-1)}} \sum_{\ell=1}^{m} \sum_{(s,a),s\neq s_{\perp}} w_{k\sigma\ell}(s,a) \sqrt{\frac{\ex_{s' \sim \hat{q},\hat{p}} \mbr{\sbr{ \bar{V}^{k}_{h+1}(s')-\underline{V}^{k}_{h+1}(s')}^2} L }{n^k(s,a)}}
	\nonumber\\
	\overset{(c)}{\leq} &  4 L \sqrt{ 3SN KH^2} + 68 SNH L^2 + 8 \underbrace{\sum_{k=1}^{K} \!\!\sum_{\sigma=\emptyset}^{m^{\oplus(H-1)}} \!\! \sum_{\ell=1}^{m} \! \sum_{(s,a),s\neq s_{\perp}} \!\!\!\!\! w_{k\sigma\ell}(s,a) \sqrt{\frac{\ex_{s' \sim \hat{q},\hat{p}} \mbr{\sbr{ \bar{V}^{k}_{h+1}(s')-\underline{V}^{k}_{h+1}(s')}^2} L }{n^k(s,a)}}}_{\term{ 2.1}} , \label{eq:term2_incomplete}
	\end{align}
	where (a) uses Lemma~\ref{lemma:con_variance}, (b) comes from Lemma~\ref{lemma:bound_visitation} and (c) is due to Lemma~\ref{lemma:LTV_main_text}.
	
	Then, we bound $\term{2.1}$ as follows:
	\begin{align}
		&\term{ 2.1} \\= & \sum_{k=1}^{K} \sum_{\sigma=\emptyset}^{m^{\oplus(H-1)}} \sum_{\ell=1}^{m} \sum_{(s,a),s\neq s_{\perp}} w_{k\sigma\ell}(s,a) \sqrt{\frac{\ex_{s' \sim \hat{q},\hat{p}} \mbr{\sbr{ \bar{V}^{k}_{h+1}(s')-\underline{V}^{k}_{h+1}(s')}^2} L }{n^k(s,a)}}
		\nonumber\\
		\leq & \sqrt{L} \! \sqrt{ \sum_{k=1}^{K} \!\! \sum_{\sigma=\emptyset}^{m^{\oplus(H-1)}} \!\! \sum_{\ell=1}^{m} \sum_{(s,a),s\neq s_{\perp}} \frac{w_{k\sigma\ell}(s,a)}{n^k(s,a)} } \cdot 
		\nonumber\\& \sqrt{ \sum_{k=1}^{K} \sum_{\sigma=\emptyset}^{m^{\oplus(H-1)}} \sum_{\ell=1}^{m} \sum_{(s,a),s\neq s_{\perp}} w_{k\sigma\ell}(s,a) \hat{q}^k(s,a) \hat{p}^k(\cdot|s,a)^\top \sbr{\bar{V}^{k}_{h+1}(s')-\underline{V}^{k}_{h+1}(s')}^2  }
		\nonumber\\
		\leq &  \sqrt{L} \! \sqrt{ \sum_{k=1}^{K} \!\! \sum_{\sigma=\emptyset}^{m^{\oplus(H-1)}} \!\! \sum_{\ell=1}^{m} \sum_{(s,a),s\neq s_{\perp}} \frac{w_{k\sigma\ell}(s,a)}{n^k(s,a)} } \cdot
		\nonumber\\& \Bigg( \sqrt{ \sum_{k=1}^{K} \sum_{\sigma=\emptyset}^{m^{\oplus(H-1)}} \sum_{\ell=1}^{m} \sum_{(s,a),s\neq s_{\perp}} w_{k\sigma\ell}(s,a)  q(s,a) p(\cdot|s,a)^\top \!\! \sbr{\bar{V}^{k}_{h+1}(s')-\underline{V}^{k}_{h+1}(s')}^2 } \!+\! \nonumber\\& \sqrt{ \sum_{k=1}^{K} \sum_{\sigma=\emptyset}^{m^{\oplus(H-1)}} \sum_{\ell=1}^{m} \sum_{(s,a),s\neq s_{\perp}}  w_{k\sigma\ell}(s,a)  \sbr{ \hat{q}^k(s,a) \hat{p}^k(\cdot|s,a) - q(s,a) p(\cdot|s,a)}^\top \sbr{\bar{V}^{k}_{h+1}(s')-\underline{V}^{k}_{h+1}(s')}^2 } \Bigg)
		\nonumber\\
		\leq & \! \sqrt{L} \! \sqrt{ \sum_{k=1}^{K} \!\! \sum_{\sigma=\emptyset}^{m^{\oplus(H-1)}} \!\! \sum_{\ell=1}^{m} \sum_{(s,a),s\neq s_{\perp}} \frac{w_{k\sigma\ell}(s,a)}{n^k(s,a)} } \cdot 
		\nonumber\\&
		\Bigg( \sqrt{ \sum_{k=1}^{K} \!\!\sum_{\sigma=\emptyset}^{m^{\oplus(H-1)}}\!\! \sum_{\ell=1}^{m} \sum_{(s,a),s\neq s_{\perp}} w_{k\sigma\ell}(s,a)  q(s,a) p(\cdot|s,a)^\top \!\! \sbr{\bar{V}^{k}_{h+1}(s')-\underline{V}^{k}_{h+1}(s')}^2 } \!+\! \nonumber\\& \sqrt{ \sum_{k=1}^{K} \sum_{\sigma=\emptyset}^{m^{\oplus(H-1)}} \sum_{\ell=1}^{m} \sum_{(s,a),s\neq s_{\perp}}  w_{k\sigma\ell}(s,a) \abr{ \sbr{ \hat{q}^k(s,a) \hat{p}^k(\cdot|s,a) - q(s,a) p(\cdot|s,a)}^\top \sbr{\bar{V}^{k}_{h+1}(s')-\underline{V}^{k}_{h+1}(s')}^2} } \Bigg)
		\nonumber\\
		\leq & \! \sqrt{L} \! \sqrt{ \sum_{k=1}^{K} \!\! \sum_{\sigma=\emptyset}^{m^{\oplus(H-1)}} \!\! \sum_{\ell=1}^{m} \sum_{(s,a),s\neq s_{\perp}} \frac{w_{k\sigma\ell}(s,a)}{n^k(s,a)} } \cdot
		\nonumber\\&
		 \Bigg( \sqrt{ \sum_{k=1}^{K} \!\!\sum_{\sigma=\emptyset}^{m^{\oplus(H-1)}}\!\! \sum_{\ell=1}^{m} \sum_{(s,a),s\neq s_{\perp}} w_{k\sigma\ell}(s,a)  q(s,a) p(\cdot|s,a)^\top \!\! \sbr{\bar{V}^{k}_{h+1}(s')-\underline{V}^{k}_{h+1}(s')}^2  } \!+\! \nonumber\\& \sqrt{ H} \sqrt{ \sum_{k=1}^{K} \sum_{\sigma=\emptyset}^{m^{\oplus(H-1)}} \sum_{\ell=1}^{m} \sum_{(s,a),s\neq s_{\perp}}  w_{k\sigma\ell}(s,a) \abr{ \sbr{ \hat{q}^k(s,a) \hat{p}^k(\cdot|s,a) - q(s,a) p(\cdot|s,a)}^\top \sbr{\bar{V}^{k}_{h+1}(s')-\underline{V}^{k}_{h+1}(s')}}  } \Bigg)
		\nonumber\\
		\overset{\textup{(a)}}{\leq} &  \sqrt{L} \sqrt{SNL} \sbr{ \sqrt{28900 m H^4L^3 S^2 N} + \sqrt{H} \sqrt{170 S^2 N H^2 L^2 \sqrt{ m L} } } 
		\nonumber\\
		\leq &  184 S N H^2 L^2 \sqrt{m S L} \label{eq:term2.1}
	\end{align}
	where (a) comes from Lemmas~\ref{lemma:bound_visitation},\ref{lemma:cumulative_V_gap} and Eq.~\eqref{eq:term4} (the upper bound of $\term{4}$).
	
	Plugging Eq.~\eqref{eq:term2.1} into Eq.~\eqref{eq:term2_incomplete}, we obtain
	\begin{align}
		\term{ 2} = & 4 L \sqrt{ 3SN KH^2} \!+\! 68 SNH L^2 \!+\! 8 \underbrace{\sum_{k=1}^{K} \! \sum_{\sigma=\emptyset}^{m^{\oplus(H-1)}} \!\! \sum_{\ell=1}^{m} \! \sum_{(s,a),s\neq s_{\perp}} \!\!\!\!\! w_{k\sigma\ell}(s,a) \sqrt{\frac{\ex_{s' \sim \hat{q},\hat{p}} \mbr{\sbr{ \bar{V}^{k}_{h+1}(s')-\underline{V}^{k}_{h+1}(s')}^2} L }{n^k(s,a)}}}_{\term{ 2.1}}
		\nonumber\\
		\leq & 8 HL \sqrt{SNK} + 68 SNH L^2 + 1472  S N H^2 L^2 \sqrt{m S L}
		\nonumber\\
		\leq & 8 HL \sqrt{SNK} + 1540  S N H^2 L^2 \sqrt{m S L} \label{eq:term2}
	\end{align}
	
	\paragraph{Step 4: Bound the estimate deviation term for triggered future values -- $\term{ 3}$.}
	
	\begin{align}
	\term{ 3} = &  \sum_{k=1}^{K} \sum_{\sigma=\emptyset}^{m^{\oplus(H-1)}} \sum_{\ell=1}^{m} \sum_{(s,a),s\neq s_{\perp}} w_{k\sigma\ell}(s,a) \sbr{ \hat{q}^k(s,a) \hat{p}^k(\cdot|s,a) - q(s,a) p(\cdot|s,a) }^\top V^{*}_{h+1}
	\nonumber\\
	\leq & \sum_{k=1}^{K} \sum_{\sigma=\emptyset}^{m^{\oplus(H-1)}} \sum_{\ell=1}^{m} \sum_{(s,a),s\neq s_{\perp}} w_{k\sigma\ell}(s,a) \sbr{ 4\sqrt{\frac{\var_{s' \sim q,p}\sbr{ V^{*}_{h+1}(s')} L }{n^k(s,a)}} + \frac{ 4H L}{n^k(s,a)} }
	\nonumber\\
	\leq & 4 \sqrt{L} \sqrt{ \sum_{k=1}^{K} \sum_{\sigma=\emptyset}^{m^{\oplus(H-1)}} \sum_{\ell=1}^{m} \sum_{(s,a),s\neq s_{\perp}} \frac{w_{k\sigma\ell}(s,a)}{n^k(s,a)} } \sqrt{ \sum_{k=1}^{K} \sum_{\sigma=\emptyset}^{m^{\oplus(H-1)}} \sum_{\ell=1}^{m} \sum_{(s,a),s\neq s_{\perp}} w_{k\sigma\ell}(s,a)  \var_{s' \sim q,p}\sbr{ V^{*}_{h+1}(s')}  } \nonumber\\&+ 4H L \sum_{k=1}^{K} \sum_{\sigma=\emptyset}^{m^{\oplus(H-1)}} \sum_{\ell=1}^{m} \sum_{(s,a),s\neq s_{\perp}}  \frac{w_{k\sigma\ell}(s,a)}{n^k(s,a)} 
	\nonumber\\
	\leq & 4 L \sqrt{ SN \sum_{k=1}^{K} \ex_{\pi^k}\mbr{ \sum_{h=1}^{H} \sum_{\ell=1}^{m^h} \var_{s' \sim q,p}\sbr{ V^{*}_{h+1}(s')} } } + 4SN H L^2
	\nonumber\\
	\overset{\textup{(a)}}{\leq} & 4 L \sqrt{ 3 SN K H^2  } + 4SN H L^2
	\nonumber\\
	\leq & 8 HL \sqrt{ SN K }  + 4SN H L^2  , \label{eq:term3}
	\end{align}
	where (a) comes from Lemma~\ref{lemma:LTV_main_text}.
	
	\paragraph{Step 5: Bound the second order term. -- $\term{ 4}$.}
	
	\begin{align}
	\term{ 4} = &  \sum_{k=1}^{K} \sum_{\sigma=\emptyset}^{m^{\oplus(H-1)}} \sum_{\ell=1}^{m} \sum_{(s,a),s\neq s_{\perp}} w_{k\sigma\ell}(s,a) \sbr{ \hat{q}^k(s,a) \hat{p}^k(\cdot|s,a) - q(s,a) p(\cdot|s,a) }^\top \sbr{\bar{V}^{k}_{h+1}-V^{*}_{h+1}}
	\nonumber\\
	= &  \sum_{k=1}^{K} \sum_{\sigma=\emptyset}^{m^{\oplus(H-1)}} \sum_{\ell=1}^{m} \sum_{(s,a),s\neq s_{\perp}} w_{k\sigma\ell}(s,a) \sum_{s'} \sbr{ \hat{q}^k(s,a) \hat{p}^k(s'|s,a) - q(s,a) p(s'|s,a) } \sbr{\bar{V}^{k}_{h+1}(s')-V^{*}_{h+1}(s')}
	\nonumber\\
	\leq &  \sum_{k=1}^{K} \sum_{\sigma=\emptyset}^{m^{\oplus(H-1)}} \sum_{\ell=1}^{m} \sum_{(s,a),s\neq s_{\perp}} w_{k\sigma\ell}(s,a) \sum_{s'} \abr{ \hat{q}^k(s,a) \hat{p}^k(s'|s,a) - q(s,a) p(s'|s,a) } \cdot \sbr{\bar{V}^{k}_{h+1}(s')-V^{*}_{h+1}(s')}
	\nonumber\\
	\leq &  \sum_{k=1}^{K} \sum_{\sigma=\emptyset}^{m^{\oplus(H-1)}} \sum_{\ell=1}^{m} \sum_{(s,a),s\neq s_{\perp}} w_{k\sigma\ell}(s,a) \cdot \nonumber\\& \sum_{s'} \sbr{ \sqrt{\frac{q(s,a) p(s'|s,a) \sbr{ 1-q(s,a) p(s'|s,a) } L }{n^k(s,a)} } + \frac{L}{n^k(s,a)}  } \sbr{\bar{V}^{k}_{h+1}(s')-\underline{V}^{k}_{h+1}(s')}
	\nonumber\\
	\leq &  \sum_{k=1}^{K} \sum_{\sigma=\emptyset}^{m^{\oplus(H-1)}} \sum_{\ell=1}^{m} \sum_{(s,a),s\neq s_{\perp}} w_{k\sigma\ell}(s,a) \sum_{s'}  \sqrt{\frac{q(s,a) p(s'|s,a) L }{n^k(s,a)} } \sbr{\bar{V}^{k}_{h+1}(s')-\underline{V}^{k}_{h+1}(s')} \nonumber\\& + \sum_{k=1}^{K} \sum_{\sigma=\emptyset}^{m^{\oplus(H-1)}} \sum_{\ell=1}^{m} \sum_{(s,a),s\neq s_{\perp}} w_{k\sigma\ell}(s,a) \sum_{s'} \frac{HL}{n^k(s,a)}  
	\nonumber\\
	\leq &  \sum_{k=1}^{K} \sum_{\sigma=\emptyset}^{m^{\oplus(H-1)}} \sum_{\ell=1}^{m} \sum_{(s,a),s\neq s_{\perp}} w_{k\sigma\ell}(s,a)   \sqrt{SL \cdot \sum_{s'} \frac{q(s,a) p(s'|s,a)  }{n^k(s,a)}  \sbr{\bar{V}^{k}_{h+1}(s')-\underline{V}^{k}_{h+1}(s')}^2 } \nonumber\\& + \sum_{k=1}^{K} \sum_{\sigma=\emptyset}^{m^{\oplus(H-1)}} \sum_{\ell=1}^{m} \sum_{(s,a),s\neq s_{\perp}} w_{k\sigma\ell}(s,a) \frac{SHL}{n^k(s,a)}
	\nonumber\\
	\leq &  \sqrt{SL} \sqrt{ \sum_{k=1}^{K} \sum_{\sigma=\emptyset}^{m^{\oplus(H-1)}} \sum_{\ell=1}^{m} \sum_{(s,a),s\neq s_{\perp}} \frac{w_{k\sigma\ell}(s,a)}{n^k(s,a)} } \cdot \nonumber\\& \sqrt{ \underbrace{ \sum_{k=1}^{K} \sum_{\sigma=\emptyset}^{m^{\oplus(H-1)}} \sum_{\ell=1}^{m} \sum_{(s,a),s\neq s_{\perp}} w_{k\sigma\ell}(s,a) \sum_{s'} q(s,a) p(s'|s,a)  \sbr{\bar{V}^{k}_{h+1}(s')-\underline{V}^{k}_{h+1}(s')}^2 }_{\term{ 4.1}} } \nonumber\\& + SHL \sum_{k=1}^{K} \sum_{\sigma=\emptyset}^{m^{\oplus(H-1)}} \sum_{\ell=1}^{m} \sum_{(s,a),s\neq s_{\perp}} \frac{w_{k\sigma\ell}(s,a)}{n^k(s,a)} 
	\nonumber\\
	\overset{\textup{(a)}}{\leq} &  \sqrt{SL}  \sqrt{S N L} \sqrt{ 28900 m H^4L^3 S^2 N  } + S^2NHL^2 
	\nonumber\\
	\leq &  170 S^2 N H^2 L^2 \sqrt{ m L}  \label{eq:term4}
	\end{align}
	where (a) is due to Lemmas~\ref{lemma:bound_visitation},\ref{lemma:cumulative_V_gap}.

	Finally, we combine the upper bounds of $\term{1}$, $\term{2}$, $\term{3}$, $\term{4}$ and the minimal regret contribution to bound the total regret. 
	
	Plugging Eqs.~\eqref{eq:term1},\eqref{eq:term2},\eqref{eq:term3} and \eqref{eq:term4} into Eq.~\eqref{eq:regret_decomposition}, we have
	\begin{align*}
		&\regret(K)
		\\
		\leq & \sum_{k=1}^{K} \sum_{\sigma=\emptyset}^{m^{\oplus(H-1)}} \sum_{\ell=1}^{m} \sum_{(s,a), s \neq s_{\perp}} w_{k\sigma\ell}(s,a) \Big[ \underbrace{ \sbr{ \tilde{q}^k(s,a) - q(s,a) } r(s,a) }_{\term{ 1}} + \underbrace{\sbr{ \tilde{q}^k(s,a) \tilde{p}^k(\cdot|s,a) - \hat{q}^k(s,a) \hat{p}^k(\cdot|s,a) }^\top \bar{V}^{k}_{h+1}}_{\term{ 2}} \nonumber\\ & +  \underbrace{\sbr{ \hat{q}^k(s,a) \hat{p}^k(\cdot|s,a) - q(s,a) p(\cdot|s,a) }^\top V^{*}_{h+1}}_{\term{ 3}} + \underbrace{ \sbr{ \hat{q}^k(s,a) \hat{p}^k(\cdot|s,a) - q(s,a) p(\cdot|s,a) }^\top \sbr{\bar{V}^{k}_{h+1}-V^{*}_{h+1}} }_{\term{ 4}}  \Big]
		\\
		\leq & 8 L \sqrt{m SN HK } + 8 HL \sqrt{SNK} + 1540  S N H^2 L^2 \sqrt{m S L} + 8 HL \sqrt{ SN K }  + 4 SN H L^2 \\&+ 170 S^2 N H^2 L^2 \sqrt{ m L} 
		\\
		= & O \sbr{HL \sqrt{SNK}}
	\end{align*}
	
\end{proof}

\subsection{Regret Lower Bound} \label{apx:lb_regret}

In this subsection, we prove the regret lower bounds for branching RL-RM in cases with Assumption~\ref{assumption:tri_prob} (Theorem~\ref{thm:regret_lb}) and without Assumption~\ref{assumption:tri_prob} (Theorem~\ref{thm:relax_assumption}). 


\subsubsection{Proof of Theorem~\ref{thm:regret_lb}}
\begin{proof}[Proof of Theorem~\ref{thm:regret_lb}]
	
	\begin{figure}[t] 
		\centering    
		\includegraphics[width=0.7\columnwidth]{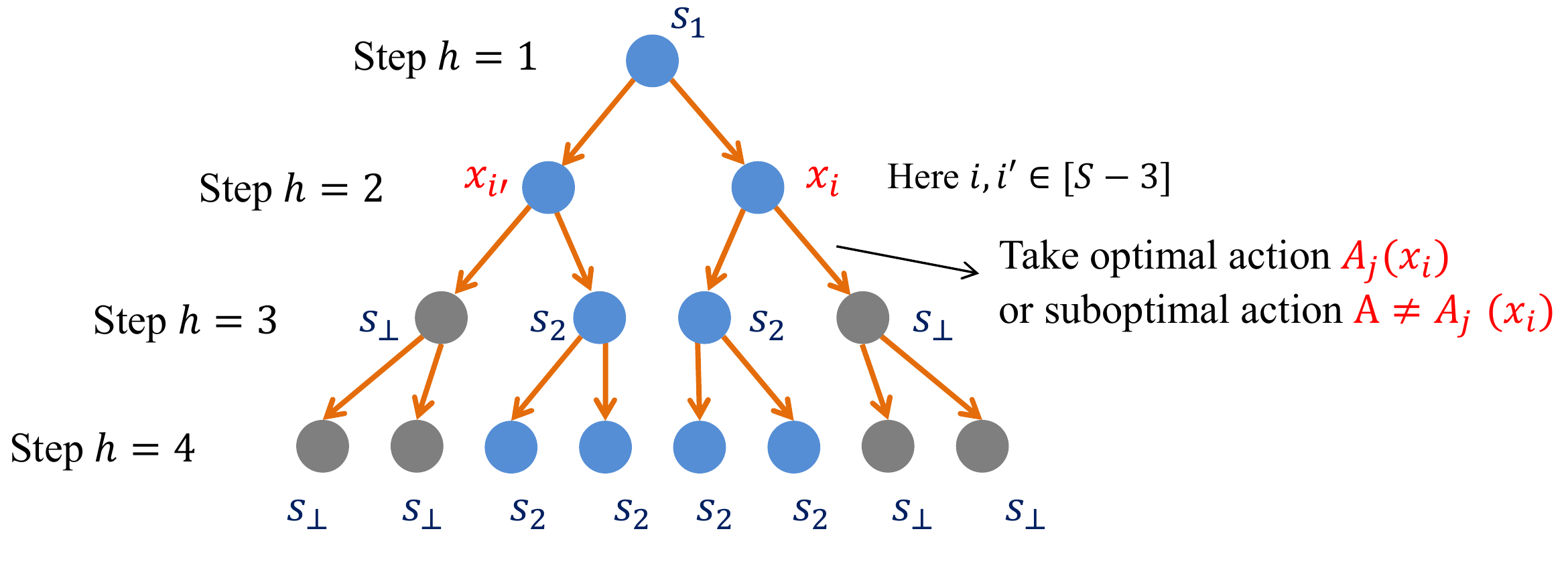} 
		\caption{The constructed instance with $m=2$ in regret lower bound analysis.} 
		\label{fig:regret_lb}     
	\end{figure}
	
	As shown in Figure~\ref{fig:regret_lb}, consider a random instance as follows: There are $N$ base actions, i.e., $A^{\univ}=\{a_1,\dots,a_{N}\}$, and $d:=\frac{N}{m}$ super actions, i.e., $A_1=\{a_1,\dots,a_m\}, A_2=\{a_{m+1},\dots,a_{2m}\},\dots,A_{d}=\{a_{m(d-1)+1},\dots,a_{md}\}$. The action set is $\cA=\{A_1,\dots,A_d\}$. 
	The state set is $\cS=\{s_{\perp},s_1,s_2,x_1,\dots,x_{S-3}\}$. 
	
	The trigger probabilities are as follows: $q(s_1,a)=q(s_2,a)=\alpha+\eta$ and $q(s_{\perp},a)=0$ for any $a \in A^{\univ}$. 
	
	The transition distributions are as follows: $s_1$ is the initial state. $p(x_i|s_1,a)=\frac{1}{S-3}$ for any $a \in A^{\univ}$. $p(s_2|x_i,a)=1$ for any $a \in A^{\univ}, i \in [S-3]$. $p(s_2|s_2,a)=1$ and $p(s_{\perp}|s_{\perp},a)=1$ for any $a \in A^{\univ}$. 
	
	The reward function is only dependent on the current state. $r(s,a)=1$ for any $s \in \cS \setminus \{s_{\perp}\}, a \in A^{\univ}$, and $r(s_{\perp},a)=0$ for any $a \in A^{\univ}$.
	
	The randomness of this instance is as follows: for each $x_i \in \{x_1,\dots,x_{S-3}\}$, we uniformly choose an action $A_j(x_i)$ from $A_1,\dots,A_d$ as the optimal action. Let $q(x_i,a)=\alpha+\eta$ for all $a \in A_j(x_i)$, and  $q(x_i,a)=\alpha$ for all $a \notin A_j(x_i)$. Let $\alpha+\eta=\frac{1}{m}$.
	
	In words, in each episode, at step $1$, an agent starts from state $s_1$ and takes an action that contains $m$ base actions, where each base action has trigger probability $\alpha+\eta$. For each state-base action pair at step $1$, if triggered successfully,  it transitions to $x_i \in \{x_1,\dots,x_{S-3}\}$ with probability $\frac{1}{S-3}$; Otherwise, if triggered successfully, it transitions to ending state $s_{\perp}$.  At step $1$, in a bandit state $x_i$, the agent takes an action $A \in \cA$ that contains $m$ base actions, where each base action has trigger probability $\alpha+\eta$ if $A=A_j(x_i)$, and has trigger probability only $\alpha$ if $A \neq A_j(x_i)$. At step $2$, for each state-base action pair, if triggered successfully, it transitions to $s_2$; Otherwise, it transitions to $s_{\perp}$. At step $3$, starting from state $s_2$, the agent takes an action where each contained base action has trigger probability $\alpha+\eta$. For each state-base action pair at step $3$, if triggered successfully, it still transitions back to $s_2$; Otherwise, it transitions to $s_{\perp}$. The following steps $4,\dots,H$ are similar to step $3$, where the agent starts from $s_2$ and transitions back to $s_2$ or transitions to $s_{\perp}$.
	
	The optimal policy $\pi_*$ is to take action $A_j(x_i)$ at state $x_i$, and we have
	\begin{align}
		\ex[\reward^{\pi_*}] = \ex \mbr{\sum_{k=1}^{K} V^{*}_1(s_1)} = K \sbr{ m (\alpha+\eta) + m^2 (\alpha+\eta)^2 + \dots + m^H (\alpha+\eta)^H } =HK \label{eq:lb_reward_pi_star}
	\end{align}
	
	Fix an algorithm $\A$. Let $\pi^k$ denote the policy taken by $\A$ in episode $k$. For each $x_i \in \{x_1,\dots,x_{S-3}\}$, let $T_{x_i,A_j(x_i)}:= \sum_{k=1}^{K} \indicator{\pi^k(x_i) = A_j(x_i) } $ denote the number of episodes where $\A$ chooses $A_j(x_i)$ in state $x_i$. Then, the number of episodes where $\A$ chooses suboptimal actions in state $x_i$ is $K-T_{x_i,A_j(x_i)}$.
	\begin{align}
		\ex[\reward^{\A}] = & \ex \mbr{\sum_{k=1}^{K} V^{\pi^k}_1(s_1) }
		\nonumber\\
		= & K m (\alpha+\eta) + \frac{1}{(S-3)d} \sum_{i=1}^{S-3} \sum_{j=1}^{d} \ex \Big[  T_{x_i,A_j(x_i)} \sbr{  m^2 (\alpha+\eta)^2 + \dots + m^H (\alpha+\eta)^H } \nonumber\\& + \sbr{K-T_{x_i,A_j(x_i)}} \sbr{ m^2 (\alpha+\eta)\alpha + m^3 (\alpha+\eta)\alpha (\alpha+\eta) + \dots + m^H (\alpha+\eta) \alpha (\alpha+\eta)^{H-2} } \Big]
		\nonumber\\
		= & K m (\alpha+\eta) + \frac{1}{(S-3)d} \sum_{i=1}^{S-3} \sum_{j=1}^{d} \ex \Big[ T_{x_i,A_j(x_i)} \sbr{  m^2 (\alpha+\eta)^2 + \dots + m^H (\alpha+\eta)^H } \nonumber\\& + \sbr{K-T_{x_i,A_j(x_i)}} \sbr{ m^2 (\alpha+\eta)\alpha + m^3 (\alpha+\eta)^2 \alpha + \dots + m^H (\alpha+\eta)^{H-1} \alpha  } \Big] \label{eq:lb_reward_A}
	\end{align}
	
	Subtracting Eq.~\eqref{eq:lb_reward_A} by Eq.~\eqref{eq:lb_reward_pi_star}, we have
	\begin{align}
		\ex[\regret^{\A}] = & \ex[\reward^{\pi_*}] - \ex[\reward^{\A}]
		\nonumber\\
		= & \ex \mbr{\sum_{k=1}^{K} \sbr{V^{*}_1(s_1) - V^{\pi^k}_1(s_1)}}
		\nonumber\\
		= & \frac{1}{(S-3)d}  \sum_{i=1}^{S-3} \sum_{j=1}^{d} \ex \Big[ \!\! \sbr{K-T_{x_i,A_j(x_i)}} \! \sbr{ m (\alpha+\eta) m\eta + m^2 (\alpha+\eta)^2 m\eta + \dots + m^{H-1} (\alpha+\eta)^{H-1} m\eta  } \!\! \Big] 
		\nonumber\\
		= & \frac{1}{(S-3)d} \sum_{i=1}^{S-3} \sum_{j=1}^{d} \ex \Big[ \sbr{K-T_{x_i,A_j(x_i)}} (H-1) m\eta \Big] 
		\nonumber\\
		= &  m\eta(H-1) \sbr{ K- \frac{1}{(S-3)d} \sum_{i=1}^{S-3} \sum_{j=1}^{d} \ex \mbr{T_{x_i,A_j(x_i)}} }  \label{eq:lb_regret}
	\end{align}
	
	Let $\ex_{A_j(x_i)}[\cdot]$ denote the expectation operator under the instance $\cI_{A_j(x_i)}$ where the optimal action of state $x_i$ is $A_j(x_i)$. 
	
	Let $\ex_{x_i,\unif}[\cdot]$ denote the expectation operator under the instance where all actions $A \in \cA$ at state $x_i$ have the same trigger probability, i.e., $q(x_i,a)=\alpha$ for any $a \in A^{\univ}$, and other distribution settings are the same as $\cI_{A_j(x_i)}$.
	
	Note that the KL-divergence between the above two instances is $m \cdot \kl\sbr{\bernoulli(\alpha) || \bernoulli(\alpha+\eta)} \leq m \cdot \frac{\eta^2}{(\alpha+\eta)(1-(\alpha+\eta))} \leq m\cdot \frac{\eta^2}{c_1}$ with $(\alpha+\eta)(1-(\alpha+\eta))\geq c_1$ for some absolute positive constant $c_1$.
	After a pull of $(x_i,A_j(x_i))$, we receive an observation of such difference between the two instances.  
	
	Using Lemma~A.1 in \cite{auer2002nonstochastic}, we have
	\begin{align*}
		\ex_{A_j(x_i)} \mbr{ T_{x_i,A_j(x_i)} } \leq & \ex_{x_i,\unif} \mbr{ T_{x_i,A_j(x_i)} } \\& + \frac{K}{2} \sqrt{ \frac{1}{2} \cdot \frac{1}{S-3} \ex_{x_i,\unif} \mbr{   T_{x_i,A_j(x_i)}} \cdot m \kl\sbr{\bernoulli(\alpha) || \bernoulli(\alpha+\eta)}  }
		\\
		\overset{\textup{(a)}}{\leq} & \ex_{x_i,\unif} \mbr{ T_{x_i,A_j(x_i)} } + \frac{K}{2} \sqrt{ \frac{m}{2 c_1 (S-3) }  \ex_{x_i,\unif} \mbr{ T_{x_i,A_j(x_i)} } \eta^2  }
		\\
		= & \ex_{x_i,\unif} \mbr{ T_{x_i,A_j(x_i)} } + \frac{K \eta}{2} \sqrt{ \frac{m}{2 c_1 (S-3) } \ex_{x_i,\unif} \mbr{ T_{x_i,A_j(x_i)} } }
	\end{align*}
	
	Since $\sum_{j=1}^{d} \ex_{x_i,\unif}  \mbr{ T_{x_i,A_j(x_i)} } = \sum_{j=1}^{d}\sum_{k=1}^{K} \ex_{x_i,\unif}  \mbr{  \pi^k(A_j(x_i)|x_i) } =K$,
	we have
	\begin{align*}
		\sum_{j=1}^{d} \ex_{A_j(x_i)} \mbr{ T_{x_i,A_j(x_i)} } \leq & \sum_{j=1}^{d} \ex_{x_i,\unif} \mbr{ T_{x_i,A_j(x_i)} } +  \frac{K \eta}{2} \sum_{j=1}^{d} \sqrt{ \frac{m}{2 c_1 (S-3) } \ex_{x_i,\unif} \mbr{ T_{x_i,A_j(x_i)} } }
		\\
		\leq & \sum_{j=1}^{d} \ex_{x_i,\unif} \mbr{ T_{x_i,A_j(x_i)} } +  \frac{K \eta}{2}  \sqrt{ \frac{dm}{2 c_1 (S-3) } \sum_{j=1}^{d} \ex_{x_i,\unif} \mbr{ T_{x_i,A_j(x_i)} } }
		\\
		= & K +  \frac{K \eta}{2}  \sqrt{   \frac{dmK}{2 c_1 (S-3) }  }
	\end{align*}
	and thus
	\begin{align}
		\frac{1}{(S-3)d} \sum_{i=1}^{S-1} \sum_{j=1}^{d} \ex_{A_j(x_i)} \mbr{ T_{x_i,A_j(x_i)} } \leq & \frac{1}{d} \sbr{K +  \frac{K \eta}{2}  \sqrt{  \frac{dm K}{2 c_1 (S-3) }  }}
		\nonumber\\
		= & K \sbr{\frac{1}{d} + \frac{ \eta }{2}  \sqrt{    \frac{mK}{2c_1 d(S-3) } }} \label{eq:lb_N_A_j}
	\end{align}
	
	Plugging Eq.~\eqref{eq:lb_N_A_j} into Eq.~\eqref{eq:lb_regret}, we obtain
	\begin{align}
	\ex[\regret^{\A}] \geq & m\eta(H-1) \sbr{ K- \frac{1}{(S-3)d} \sum_{i=1}^{S-3} \sum_{j=1}^{d} \ex \mbr{T_{x_i,A_j(x_i)}} }
	\nonumber\\
	\geq & m\eta(H-1) K \sbr{1-\frac{1}{d} - \frac{ \eta }{2} \sqrt{ \frac{mK}{2 c_1 d(S-3) } }} \label{eq:lb_regret_plug_eta}
	\end{align}
	
	Let $\eta=c_2 \sqrt{ \frac{d(S-3)}{m K} }$ for some small enough constant $c_2$, we have 
	\begin{align*}
	\ex[\regret^{\A}] = & \Omega \sbr{ (H-1) \sqrt{(S-3)dmK} }
	\\
	= & \Omega \sbr{ H \sqrt{SNK} }
	\end{align*}
\end{proof}

\subsubsection{Proof of Theorem~\ref{thm:relax_assumption}} \label{apx:proof_lb_relax_assump}


\begin{proof}[Proof of Theorem~\ref{thm:relax_assumption}]
	
This proof uses the same instance and analytical procedure as the proof of Theorem~\ref{thm:regret_lb}, except that we set $\alpha+\eta=\bar{q}$ for some trigger probability threshold $\bar{q}>\frac{1}{m}$.

Then, Eq.~\eqref{eq:lb_regret} becomes
\begin{align}
\ex[\regret^{\A}] = & \ex[\reward^{\pi_*}] - \ex[\reward^{\A}]
\nonumber\\
= & \ex \mbr{\sum_{k=1}^{K} \sbr{V^{*}_1(s_1) - V^{\pi^k}_1(s_1)}}
\nonumber\\
= & \frac{1}{(S-3)d} \! \sum_{i=1}^{S-3} \sum_{j=1}^{d} \! \ex \Big[\!\! \sbr{K-T_{x_i,A_j(x_i)}} \! \sbr{ m \bar{q} \cdot m\eta + m^2 \bar{q}^2 \cdot m\eta + \dots + m^{H-1} \bar{q}^{H-1} \cdot m\eta  } \!\!\Big] 
\nonumber\\
= & \frac{1}{(S-3)d} \sum_{i=1}^{S-3} \sum_{j=1}^{d} \ex \Big[ \sbr{K-T_{x_i,A_j(x_i)}} \frac{m \bar{q} \sbr{ (m \bar{q})^{H-1}-1 } }{m \bar{q}-1} \cdot m\eta \Big]  
\nonumber\\
= & \frac{m \bar{q} \sbr{ (m \bar{q})^{H-1}-1 } }{m \bar{q}-1} \cdot m\eta \sbr{ K- \frac{1}{(S-3)d} \sum_{i=1}^{S-3} \sum_{j=1}^{d} \ex \mbr{T_{x_i,A_j(x_i)}} } 
\end{align}
and Eq~\eqref{eq:lb_regret_plug_eta} becomes
\begin{align}
\ex[\regret^{\A}] \geq & \frac{m \bar{q} \sbr{ (m \bar{q})^{H-1}-1 } }{m \bar{q}-1} \cdot m\eta \sbr{ K- \frac{1}{(S-3)d} \sum_{i=1}^{S-3} \sum_{j=1}^{d} \ex \mbr{T_{x_i,A_j(x_i)}} }
\\
\geq & \frac{m \bar{q} \sbr{ (m \bar{q})^{H-1}-1 } }{m \bar{q}-1} \cdot m\eta K \sbr{1-\frac{1}{d} - \frac{ \eta }{2} \sqrt{ \frac{mK}{2 c_1 d(S-3) } }} 
\end{align}

Let $\eta=c_2 \sqrt{ \frac{d(S-3)}{m K} }$ for some small enough constant $c_2$, we have 
\begin{align*}
\ex[\regret^{\A}] = & \Omega \sbr{ \frac{m \bar{q} \sbr{ (m \bar{q})^{H-1}-1 } }{m \bar{q}-1}  \sqrt{(S-3)dmK} }
\\
= & \Omega \sbr{ \frac{m \bar{q} \sbr{ (m \bar{q})^{H-1}-1 } }{m \bar{q}-1} \sqrt{SNK} }
\end{align*}

Therefore, when relaxing the trigger probability threshold in Assumption~\ref{assumption:tri_prob} to some $\bar{q}>\frac{1}{m}$, any algorithm for branching RL-RM must suffer an exponential regret.

\end{proof}

\section{Proofs for Branching RL with Reward-Free Exploration}

In this section, we prove sample complexity upper and lower bounds (Theorems \ref{thm:sample_complexity_ub},\ref{thm:sample_complexity_lb}) for branching RL-RFE.

\subsection{Proof for Sample Complexity Upper Bound}

\subsubsection{Augmented Transition Distribution}

First, we introduce an augmented transition distribution $p^{\aug}(\cdot|s,a)$ and connect it with trigger distribution $q(s,a)$ and transition distribution $p(\cdot|s,a)$.

For any $(s,a)\in \cS \times A^{\univ}$, let $p^{\aug}(\cdot|s,a)$ denote the augmented transition distribution on $\cS$, which satisfies that
\begin{align*}
	p^{\aug}(s_{\perp}|s,a)&=1-q(s,a), 
	\\
	p^{\aug}(s'|s,a)&=q(s,a) p(s'|s,a), \ \forall s' \in \cS \setminus \{s_{\perp}\}.
\end{align*}

For any episode $k$, we can also define the empirical augmented transition distribution as 
\begin{align*}
\hat{p}^{\aug,k}(s_{\perp}|s,a)&=1-\hat{q}^k(s,a), 
\\
\hat{p}^{\aug,k}(s'|s,a)&=\hat{q}^k(s,a) \hat{p}^k(s'|s,a), \ \forall s' \in \cS \setminus \{s_{\perp}\}.
\end{align*}

\begin{lemma}\label{lemma:p_aug_equivalent}
For any function $f(\cdot)$ defined on $\cS$ such that $f(s_{\perp})=0$ (e.g., $f$ can be the value function $V^{\pi}_{h}(\cdot)$, $\forall h \in [H], \pi$), it holds that
\begin{align}
	p^{\aug}(\cdot|s,a)^\top f &=  q(s,a) p(\cdot|s,a)^\top f , \label{eq:p_aug_equivalent}
	\\
	\hat{p}^{\aug,k}(\cdot|s,a)^\top V_{h+1} &=  \hat{q}^k(s,a) \hat{p}^k(\cdot|s,a)^\top  f . \label{eq:hat_p_aug_equivalent}
\end{align}
\end{lemma}
\begin{proof}[Proof of Lemma~\ref{lemma:p_aug_equivalent}]
	We prove Eq.~\eqref{eq:p_aug_equivalent} as follows:
	\begin{align*}
		p^{\aug}(\cdot|s,a)^\top f &=  \sum_{s' \in \cS} p^{\aug}(s'|s,a) f(s')
		\\
		&=  \sum_{s' \in \cS \setminus\{s_{\perp}\}} q(s,a) p(s'|s,a) f(s')
		\\
		&=  q(s,a) p(\cdot|s,a)^\top f 
	\end{align*}
	
	Eq.~\eqref{eq:hat_p_aug_equivalent} can be proved in a similar manner.
\end{proof}

\subsubsection{Concentration}
In the following, we introduce a concentration lemma.

\begin{lemma}[KL-divergence Based Concentration of Triggered Transition]\label{lemma:kl_con_trigger_transition}
	Defining event
	\begin{align*}
	\cG := \lbr{\kl\sbr{ \hat{p}^{\aug,k}(\cdot|s,a) \| p^{\aug}(\cdot|s,a) } \leq   \frac{\log\sbr{\frac{SN}{\delta'}}+S\log \sbr{8e (n^k(s,a)+1)} }{n^k(s,a)}  , \  \forall (s,a)\in \cS \times A^{\univ}, \forall k  } ,
	\end{align*} 
	it holds that
	\begin{align*}
		\Pr \mbr{\cG} \geq 1-\delta .
	\end{align*}
\end{lemma}
\begin{proof}[Proof of Lemma~\ref{lemma:kl_con_trigger_transition}]
	Using Theorem 3 and Lemma 3 in \cite{menard2021fast},
	we can obtain this lemma.	
\end{proof}


%

\subsubsection{KL Divergence-based Technical Tools}

Below, we present several useful KL divergence-based technical tools.

\begin{lemma}[Lemma 10 in \cite{menard2021fast}] \label{lemma:kl_deviation_value}
	Let $p_1$ and $p_2$ be two distributions on $\cS$ such that $\kl(p_1,p_2)\leq \alpha$. Let $f$ be a function defined on $\cS$ such that for any $s \in \cS$, $0\leq f(s) \leq b$. Then,
	$$
	\abr{ p_1 f- p_2f } \leq \sqrt{2 \var_{p_2}(f)\alpha } + \frac{2}{3} b \alpha
	$$
\end{lemma}

\begin{lemma}[Lemma 11 in \cite{menard2021fast}] \label{lemma:kl_change_transition}
	Let $p_1$ and $p_2$ be two distributions on $\cS$ such that $\kl(p_1,p_2)\leq \alpha$. Let $f$ be a function defined on $\cS$ such that for any $s \in \cS$, $0\leq f(s) \leq b$. Then,
	\begin{align*}
	\var_{p_2}(f) \leq 2 \var_{p_1}(f) + 4b^2\alpha
	\\
	\var_{p_1}(f) \leq 2 \var_{p_2}(f) + 4b^2\alpha
	\end{align*}
\end{lemma}

\begin{lemma}[Lemma 12 in \cite{menard2021fast}] \label{lemma:kl_change_value}
	Let $p_1$ and $p_2$ be two distributions on $\cS$ such that $\kl(p_1,p_2)\leq \alpha$. Let $f,g$ be two functions defined on $\cS$ such that for any $s \in \cS$, $0\leq f(s),g(s) \leq b$. Then,
	\begin{align*}
	\var_{p_1}(f) \leq & 2 \var_{p_1}(g) + 2bp_1|f-g|
	\\
	\var_{p_2}(f) \leq & \var_{p_1}(f) + 3b^2\alpha\|p_1-p_2\|_1
	\end{align*}
\end{lemma}


\subsubsection{Estimation Error}

Next, we state an important lemma on estimation error.

\begin{lemma}[Estimation Error]\label{lemma:estimation_error}
	Suppose that the concentration event $\cG$ holds. Then, for any episode $k$, policy $\pi$ and reward function $r$,
	\begin{align*}
	\abr{\hat{V}^{k,\pi}_1(s;r)-V^{\pi}_1(s;r)}  \leq 4e \sqrt{ B^{k}_1(s) } + B^{k}_1(s) .
	\end{align*}
\end{lemma}

\begin{proof}[Proof of Lemma~\ref{lemma:estimation_error}]

For any $t \in \N, \kappa \in (0,1)$, let $\beta(t,\kappa):=\log(SN/\kappa)+S\log(8e(t+1))$. 
Then, for any $\pi$, $r$ and $(s,A) \in \cS \setminus \{s_{\perp}\} \times \cA$,
\begin{align*}
	& \abr{\hat{Q}^{k,\pi}_h(s,A;r)-Q^{\pi}_h(s,A;r)} 
	\\
	= & \sum_{a \in A} \abr{ \sbr{ \hat{q}^k(s,a) - q(s,a) } r(s,a) +  \hat{q}^k(s,a) \hat{p}^k(\cdot|s,a)^\top \hat{V}^{k,\pi}_{h+1} - q(s,a) p(\cdot|s,a)^\top V^{\pi}_{h+1}  }
	\\
	\leq & \sum_{a \in A} \Big( \abr{ \hat{q}^k(s,a) - q(s,a) } r(s,a) + \abr{\hat{q}^k(s,a) \hat{p}^k(\cdot|s,a) - q(s,a) p(\cdot|s,a)}^\top V^{\pi}_{h+1} \\& + \hat{q}^k(s,a) \hat{p}^k(\cdot|s,a)^\top \abr{\hat{V}^{k,\pi}_{h+1} - V^{\pi}_{h+1}} \Big)
	\\
	\overset{\textup{(a)}}{\leq} & \sum_{a \in A} \Bigg( 2 \sqrt{\frac{\beta(n^k(s,a),\delta')}{n^k(s,a)}} r(s,a) + \sqrt{\frac{2 \var_{s' \sim q,p}\sbr{ V^{*}_{h+1}(s')} \beta(n^k(s,a),\delta') }{n^k(s,a)}} + \frac{2}{3}H \frac{\beta(n^k(s,a),\delta')}{n^k(s,a)} \\& + \hat{q}^k(s,a) \hat{p}^k(\cdot|s,a)^\top \abr{\hat{V}^{k,\pi}_{h+1} - V^{\pi}_{h+1}} \Bigg)
	\\
	\overset{\textup{(b)}}{\leq} & \sum_{a \in A} \Bigg( 2 \sqrt{\frac{\beta(n^k(s,a),\delta')}{n^k(s,a)}} r(s,a) + \sqrt{\frac{4 \var_{s' \sim \hat{q},\hat{p}}\sbr{V^{*}_{h+1}(s')} \beta(n^k(s,a),\delta') }{n^k(s,a)}+8H^2 \sbr{\frac{\beta(n^k(s,a),\delta')}{n^k(s,a)}}^2 } \\& + \frac{2}{3}H \frac{\beta(n^k(s,a),\delta')}{n^k(s,a)} + \hat{q}^k(s,a) \hat{p}^k(\cdot|s,a)^\top \abr{\hat{V}^{k,\pi}_{h+1} - V^{\pi}_{h+1}} \Bigg)
	\\
	\overset{\textup{(c)}}{\leq} & \sum_{a \in A} \Bigg( 2 \sqrt{\frac{\beta(n^k(s,a),\delta')}{n^k(s,a)}} r(s,a) \\& \!+\! \sqrt{\frac{8 \var_{s'} (\hat{V}^{k,\pi}_{h+1}(s')) \! \beta(n^k(s,a),\delta') }{n^k(s,a)} \!+\! 8H \hat{q}^k(s,a) \hat{p}^k(\cdot|s,a)^\top \!\! \abr{ V^{\pi}_{h+1} \!-\! \hat{V}^{k,\pi}_{h+1}} \!\! \frac{\beta(n^k(s,a),\delta')}{n^k(s,a)}  \!+\! 8H^2 \!\! \sbr{\!\!\frac{\beta(n^k(s,a),\delta')}{n^k(s,a)} \!\!}^{\!\!2} } \\& + \frac{2}{3}H \frac{\beta(n^k(s,a),\delta')}{n^k(s,a)} + \hat{q}^k(s,a) \hat{p}^k(\cdot|s,a)^\top \abr{\hat{V}^{k,\pi}_{h+1} - V^{\pi}_{h+1}} \Bigg)
	\\
	\overset{\textup{(d)}}{\leq} & \sum_{a \in A} \Bigg( 2\sqrt{\frac{\beta(n^k(s,a),\delta')}{n^k(s,a)}} r(s,a) + \sqrt{\frac{8 \var_{s' \sim \hat{q},\hat{p}}\sbr{\hat{V}^{k,\pi}_{h+1}(s')} \beta(n^k(s,a),\delta') }{n^k(s,a)}} + \sqrt{8H^2 \sbr{\frac{\beta(n^k(s,a),\delta')}{n^k(s,a)}}^2 } \\& + \sqrt{ \frac{1}{H} \hat{q}^k(s,a) \hat{p}^k(\cdot|s,a)^\top  \abr{ V^{\pi}_{h+1} - \hat{V}^{k,\pi}_{h+1}} 8H^2 \frac{\beta(n^k(s,a),\delta')}{n^k(s,a)} }   + \frac{2}{3}H \frac{\beta(n^k(s,a),\delta')}{n^k(s,a)} \\&+ \hat{q}^k(s,a) \hat{p}^k(\cdot|s,a)^\top \abr{\hat{V}^{k,\pi}_{h+1} - V^{\pi}_{h+1}} \Bigg)
	\\
	\overset{\textup{(e)}}{\leq} & \sum_{a \in A} \Bigg( 2\sqrt{\frac{\beta(n^k(s,a),\delta')}{n^k(s,a)}} r(s,a) + \sqrt{\frac{8 \var_{s' \sim \hat{q},\hat{p}}\sbr{\hat{V}^{k,\pi}_{h+1}(s')} \beta(n^k(s,a),\delta') }{n^k(s,a)}} \\& + \frac{1}{H} \hat{q}^k(s,a) \hat{p}^k(\cdot|s,a)^\top  \abr{ V^{\pi}_{h+1} - \hat{V}^{k,\pi}_{h+1}} + 12H^2 \frac{\beta(n^k(s,a),\delta')}{n^k(s,a)} + \hat{q}^k(s,a) \hat{p}^k(\cdot|s,a)^\top \abr{\hat{V}^{k,\pi}_{h+1} - V^{\pi}_{h+1}} \Bigg)
	\\
	= & \sum_{a \in A} \Bigg( 2\sqrt{\frac{\beta(n^k(s,a),\delta')}{n^k(s,a)}} r(s,a) + \sqrt{\frac{8 \var_{s' \sim \hat{q},\hat{p}}\sbr{\hat{V}^{k,\pi}_{h+1}(s')} \beta(n^k(s,a),\delta') }{n^k(s,a)}}  + 12H^2 \frac{\beta(n^k(s,a),\delta')}{n^k(s,a)} \\&+ \sbr{1+\frac{1}{H}} \hat{q}^k(s,a) \hat{p}^k(\cdot|s,a)^\top \abr{\hat{V}^{k,\pi}_{h+1} - V^{\pi}_{h+1}} \Bigg) ,
\end{align*} 
Here (a)(b)(c) use Lemmas~\ref{lemma:kl_deviation_value},\ref{lemma:kl_change_transition},\ref{lemma:kl_change_value}, respectively. (d) is due to that $\sqrt{x+y+z}\leq \sqrt{x}+\sqrt{y}+\sqrt{z}$ for $x,y,z\geq 0$. (e) comes from that $\sqrt{xy}\leq x+y$ for $x,y\geq 0$.

Then, unfolding $\abr{\hat{Q}^{k,\pi}_1(s,a;r)-Q^{\pi}_1(s,a;r)}$, we have
\begin{align}
	& \abr{\hat{Q}^{k,\pi}_1(s,A;r)-Q^{\pi}_1(s,A;r)} 
	\nonumber\\
	\leq & \sum_{\sigma=\emptyset}^{m^{\oplus(H-1)}} \sum_{\ell=1}^{m} \sbr{1+\frac{1}{H}}^{h-1}  \ex_{\hat{q},\hat{p},\pi} \Bigg[ \Bigg( 2\sqrt{\frac{\beta(n^k(s_{\sigma},a_{\sigma \oplus \ell}),\delta')}{n^k(s_{\sigma},a_{\sigma \oplus \ell})}} r(s_{\sigma},a_{\sigma \oplus \ell}) \nonumber\\& + \sqrt{\frac{8 \var_{s' \sim \hat{q},\hat{p}} \sbr{ \hat{V}^{k,\pi}_{h+1}(s')} \beta(n^k(s_{\sigma},a_{\sigma \oplus \ell}),\delta') }{n^k(s_{\sigma},a_{\sigma \oplus \ell})}}  + 12H^2 \frac{\beta(n^k(s_{\sigma},a_{\sigma \oplus \ell}),\delta')}{n^k(s_{\sigma},a_{\sigma \oplus \ell})}\Bigg) \cdot \indicator{ s_{\sigma}\neq s_{\perp} }\Bigg] 
	\nonumber\\
	\leq & e\sum_{\sigma=\emptyset}^{m^{\oplus(H-1)}} \sum_{\ell=1}^{m} \sum_{(s,a), s \neq s_{\perp}} \hat{w}^{k,\pi}_{\sigma\ell}(s,a) \sbr{  5 \sqrt{ \frac{ \var_{s' \sim \hat{q},\hat{p}}\sbr{\hat{V}^{k,\pi}_{h+1}(s')}  }{ H^2 } \frac{ H^2 \beta(n^k(s,a),\delta') }{n^k(s,a)} } + 12H^2 \frac{\beta(n^k(s,a),\delta')}{n^k(s,a)} }
	\nonumber\\
	\leq & 5e \sqrt{ \sum_{\sigma=\emptyset}^{m^{\oplus(H-1)}} \sum_{\ell=1}^{m} \sum_{(s,a), s \neq s_{\perp}} \hat{w}^{k,\pi}_{\sigma\ell}(s,a)   \frac{ \var_{s' \sim \hat{q},\hat{p}}\sbr{\hat{V}^{k,\pi}_{h+1}(s')}  }{ H^2 } } \sqrt{ H^2 \sum_{\sigma=\emptyset}^{m^{\oplus(H-1)}} \sum_{\ell=1}^{m} \sum_{(s,a), s \neq s_{\perp}} \hat{w}^{k,\pi}_{\sigma\ell}(s,a) \frac{ \beta(n^k(s,a),\delta') }{n^k(s,a)} } \nonumber\\& + 12 e H^2 \sum_{\sigma=\emptyset}^{m^{\oplus(H-1)}} \sum_{\ell=1}^{m} \sum_{(s,a), s \neq s_{\perp}} \hat{w}^{k,\pi}_{\sigma\ell}(s,a)  \frac{\beta(n^k(s,a),\delta')}{n^k(s,a)} 
	\nonumber\\
	= & 5e \sqrt{ \frac{1}{H^2} \ex_{\hat{q},\hat{p},\pi} \mbr{ \sum_{\sigma=\emptyset}^{m^{\oplus(H-1)}} \sum_{\ell=1}^{m}  \var_{s' \sim \hat{q},\hat{p}} \sbr{ \hat{V}^{k,\pi}_{h+1}(s')} } } \sqrt{ H^2 \sum_{\sigma=\emptyset}^{m^{\oplus(H-1)}} \sum_{\ell=1}^{m} \sum_{(s,a), s \neq s_{\perp}} \hat{w}^{k,\pi}_{\sigma\ell}(s,a) \frac{ \beta(n^k(s,a),\delta') }{n^k(s,a)} } \nonumber\\& + 12 e H^2 \sum_{\sigma=\emptyset}^{m^{\oplus(H-1)}} \sum_{\ell=1}^{m} \sum_{(s,a), s \neq s_{\perp}} \hat{w}^{k,\pi}_{\sigma\ell}(s,a)  \frac{\beta(n^k(s,a),\delta')}{n^k(s,a)}
	\nonumber\\
	\overset{\textup{(a)}}{=} & 10e \sqrt{ H^2 \sum_{\sigma=\emptyset}^{m^{\oplus(H-1)}} \sum_{\ell=1}^{m} \sum_{(s,a), s \neq s_{\perp}} \hat{w}^{k,\pi}_{\sigma\ell}(s,a) \frac{ \beta(n^k(s,a),\delta') }{n^k(s,a)} } + 12 e H^2 \sum_{\sigma=\emptyset}^{m^{\oplus(H-1)}} \sum_{\ell=1}^{m} \sum_{(s,a), s \neq s_{\perp}} \hat{w}^{k,\pi}_{\sigma\ell}(s,a)  \frac{\beta(n^k(s,a),\delta')}{n^k(s,a)} , \label{eq:unfolding_Q_1} 
\end{align}
where (a) uses branching law of total variance (Lemma~\ref{lemma:LTV_main_text}), which also holds for the estimated model $(\hat{q}^k,\hat{p}^k)$ if adding a clip operation $\hat{q}^k(s,a) \leftarrow \min\{\hat{q}^k(s,a),\frac{1}{m}\}$ in algorithm $\algbonusleader$ to guarantee $\hat{q}^k \leq \frac{1}{m}$.

Define 
\begin{align*}
	B^{k,\pi}_h(s,A) := & \min \lbr{\sum_{a \in A} \sbr{12 H^2 \frac{\beta(n^k(s,a),\delta')}{n^k(s,a)} + \sbr{1+\frac{1}{H}} \hat{q}^k(s,a) \hat{p}^k(\cdot|s,a)^\top B^{k,\pi}_{h+1} }, H} 
	\\
	B^{k,\pi}_h(s) := & B^{k,\pi}_h(s,\pi(s)) .
	\\
	B^{k}_h(s,A) := & \min \lbr{\sum_{a \in A} \sbr{12 H^2 \frac{\beta(n^k(s,a),\delta')}{n^k(s,a)} + \sbr{1+\frac{1}{H}} \hat{q}^k(s,a) \hat{p}^k(\cdot|s,a)^\top B^{k}_h(s) }, H}
	\\
	B^{k}_h(s) := & \max_{A \in \cA} B^{k}_h(s,A)
\end{align*}

In the following, we show 
\begin{align}
	\abr{\hat{Q}^{k,\pi}_1(s,A;r)-Q^{\pi}_1(s,A;r)} \leq 4e \sqrt{ B^{k,\pi}_1(s,A) } + B^{k,\pi}_1(s,A) \label{eq:Q_diff_leq_sqrt_B}
\end{align}

If $B^{k,\pi}_1(s,A)=H$, Eq.~\eqref{eq:Q_diff_leq_sqrt_B} holds trivially. 
Otherwise, unfolding $B^{k,\pi}_1(s,A)$, we have
\begin{align*}
	B^{k,\pi}_1(s,A) = 12 e H^2 \sum_{\sigma=\emptyset}^{m^{\oplus(H-1)}} \sum_{\ell=1}^{m} \sum_{(s,a), s \neq s_{\perp}} \hat{w}^{k,\pi}_{\sigma\ell}(s,a)  \frac{\beta(n^k(s,a),\delta')}{n^k(s,a)} ,
\end{align*}
and using Eq.~\eqref{eq:unfolding_Q_1}, we obtain
\begin{align*}
	\abr{\hat{Q}^{k,\pi}_1(s,A;r)-Q^{\pi}_1(s,A;r)} \leq 4e \sqrt{ B^{k,\pi}_1(s,A) } + B^{k,\pi}_1(s,A) .
\end{align*}

Thus, by the definitions of $B^{k}_h(s,A)$ and $B^{k}_h(s)$, we have
\begin{align*}
	\abr{\hat{V}^{k,\pi}_1(s;r)-V^{\pi}_1(s;r)} \leq 4e \sqrt{ B^{k,\pi}_1(s) } + B^{k,\pi}_1(s) \leq 4e \sqrt{ B^{k}_1(s) } + B^{k}_1(s) .
\end{align*}

\end{proof}

\subsubsection{Proof of Theorem~\ref{thm:sample_complexity_ub}}

Now, we prove the sample complexity upper bound for algorithm $\algbonusleader$ (Theorem~\ref{thm:sample_complexity_ub}).

\begin{proof}[Proof of Theorem~\ref{thm:sample_complexity_ub}]
First, we prove the correctness.

Let $K$ denote the number of episodes that algorithm $\algbonusleader$ costs.
According to the stopping rule (Line~\ref{line:rfe_stop_rule} in Algorithm~$\algbonusleader$) and Lemma~\ref{lemma:estimation_error}, when algorithm $\algbonusleader$ stops in episode $K$, we have that for any $\pi,r$,
\begin{align*}
	\abr{\hat{V}^{K,\pi}_1(s_1;r)-V^{\pi}_1(s_1;r)} \leq  4e \sqrt{ B^{K}_1(s) } + B^{K}_1(s) \leq \frac{\varepsilon}{2}
\end{align*}

Then, we have that for any $r$,
\begin{align*}
	V^{*}_1(s_1;r)-V^{\hat{\pi}^*}_1(s_1;r) = & V^{*}_1(s_1;r)-\hat{V}^{K,\pi^*}_1(s_1;r)+\hat{V}^{K,\pi^*}_1(s_1;r)-\hat{V}^{K,\hat{\pi}^*}_1(s_1;r)+\hat{V}^{K,\hat{\pi}^*}_1(s_1;r)-V^{\hat{\pi}^*}_1(s_1;r)
	\\
	\leq & \abr{V^{*}_1(s_1;r)-\hat{V}^{K,\pi^*}_1(s_1;r)}+\abr{\hat{V}^{K,\hat{\pi}^*}_1(s_1;r)-V^{\hat{\pi}^*}_1(s_1;r)}
	\\
	\leq & \frac{\varepsilon}{2} + \frac{\varepsilon}{2}
	\\
	= & \varepsilon
\end{align*}

Now, we prove the sample complexity.


\begin{align*}
	B^{k}_h(s,A) \leq & \sum_{a \in A} \sbr{12 H^2 \frac{\beta(n^k(s,a),\delta')}{n^k(s,a)} + \sbr{1+\frac{1}{H}} \hat{q}^k(s,a) \hat{p}^k(\cdot|s,a)^\top B^{k}_{h+1} }
	\\
	= & \sum_{a \in A} \Bigg( 12 H^2 \frac{\beta(n^k(s,a),\delta')}{n^k(s,a)} + \sbr{1+\frac{1}{H}} q(s,a) p(\cdot|s,a)^\top B^{k}_{h+1} \\& + \sbr{1+\frac{1}{H}} \sbr{\hat{q}^k(s,a) \hat{p}^k(\cdot|s,a) - q(s,a) p(\cdot|s,a) }^\top B^{k}_{h+1} \Bigg)
	\\
	\overset{\textup{(a)}}{\leq} & \sum_{a \in A} \Bigg( 12 H^2 \frac{\beta(n^k(s,a),\delta')}{n^k(s,a)} + \sbr{1+\frac{1}{H}} q(s,a) p(\cdot|s,a)^\top B^{k}_{h+1} \\& + \sbr{1+\frac{1}{H}} \Bigg( \sqrt{2 \var_{s' \sim q,p} \sbr{ B^{k}_{h+1}(s')} \frac{\beta(n^k(s,a),\delta') }{n^k(s,a)}} + \frac{2}{3}H \frac{\beta(n^k(s,a),\delta')}{n^k(s,a)} \Bigg) \Bigg)
	\\
	\leq & \sum_{a \in A} \Bigg( 12 H^2 \frac{\beta(n^k(s,a),\delta')}{n^k(s,a)} + \sbr{1+\frac{1}{H}} q(s,a) p(\cdot|s,a)^\top B^{k}_{h+1} \\& + \sbr{1+\frac{1}{H}} \Bigg( \sqrt{ 2H q(s,a) p(\cdot|s,a)^\top B^{k}_{h+1} \frac{\beta(n^k(s,a),\delta') }{n^k(s,a)}} + \frac{2}{3}H \frac{\beta(n^k(s,a),\delta')}{n^k(s,a)} \Bigg) \Bigg)
	\\
	= & \sum_{a \in A} \Bigg( 12 H^2 \frac{\beta(n^k(s,a),\delta')}{n^k(s,a)} + \sbr{1+\frac{1}{H}} q(s,a) p(\cdot|s,a)^\top B^{k}_{h+1} \\& + \sbr{1+\frac{1}{H}} \Bigg( \sqrt{ \frac{1}{H} q(s,a) p(\cdot|s,a)^\top B^{k}_{h+1} 2 H^2 \frac{\beta(n^k(s,a),\delta') }{n^k(s,a)}} + \frac{2}{3}H \frac{\beta(n^k(s,a),\delta')}{n^k(s,a)} \Bigg) \Bigg)
	\\
	\leq & \sum_{a \in A} \Bigg( 12 H^2 \frac{\beta(n^k(s,a),\delta')}{n^k(s,a)} + \sbr{1+\frac{1}{H}} q(s,a) p(\cdot|s,a)^\top B^{k}_{h+1} \\& + \sbr{1+\frac{1}{H}} \Bigg(  \frac{1}{H} q(s,a) p(\cdot|s,a)^\top B^{k}_{h+1} + 2 H^2 \frac{\beta(n^k(s,a),\delta') }{n^k(s,a)} + \frac{2}{3}H \frac{\beta(n^k(s,a),\delta')}{n^k(s,a)} \Bigg) \Bigg)
	\\
	\leq & \sum_{a \in A} \Bigg( 12 H^2 \frac{\beta(n^k(s,a),\delta')}{n^k(s,a)} + \sbr{1+\frac{1}{H}} q(s,a) p(\cdot|s,a)^\top B^{k}_{h+1} + \frac{2}{H} q(s,a) p(\cdot|s,a)^\top B^{k}_{h+1} \\& + 6 H^2 \frac{\beta(n^k(s,a),\delta') }{n^k(s,a)} \Bigg)
	\\
	= & \sum_{a \in A} \Bigg( 18 H^2 \frac{\beta(n^k(s,a),\delta')}{n^k(s,a)} + \sbr{1+\frac{3}{H}} q(s,a) p(\cdot|s,a)^\top B^{k}_{h+1} \Bigg),
\end{align*}
where (a) uses Lemma~\ref{lemma:kl_deviation_value}.


Then, unfolding $B^{k}_1(s)=B^{k}_1(s,\pi^k(s))$ and summing over $k=1,\dots,K-1$, we have

\begin{align}
\sum_{k=1}^{K-1} B^{k}_1(s) \leq & \sum_{k=1}^{K-1} \sum_{\sigma=\emptyset}^{m^{\oplus(H-1)}} \sum_{\ell=1}^{m}  \ex_{q,p,\pi^k} \mbr{ 18 e^3 H^2 \frac{\beta(n^k(s_{\sigma},a_{\sigma \oplus \ell}),\delta')}{n^k(s_{\sigma},a_{\sigma \oplus \ell})} \cdot \indicator{ s_{\sigma} \neq s_{\perp} } }
\nonumber\\
= & 18 e^3 H^2 \ex_{q,p,\pi^k} \mbr{ \sum_{k=1}^{K-1} \sum_{(s,a), s \neq s_{\perp}} X_k(s,a) \frac{\beta(n^k(s,a),\delta')}{n^k(s,a)}  }
\nonumber\\
\leq & 18 e^3 H^2 \ex_{q,p,\pi^k} \mbr{ \sum_{(s,a), s \neq s_{\perp}} \sum_{k=1}^{K-1}  X_k(s,a) \frac{\beta(n^k(s,a),\delta')}{n^k(s,a)} }
\nonumber\\
\leq & 18 e^3 H^2 \cdot \beta \sbr{(K-1) \frac{m^{H+1}-m}{m-1},\delta'} \ex_{q,p,\pi^k} \mbr{  \sum_{(s,a), s \neq s_{\perp}}  \log \sbr{ n^{K-1}(s,a) } }
\nonumber\\
\leq & 18 e^3 H^2 \cdot \beta \sbr{(K-1) \frac{m^{H+1}-m}{m-1},\delta'}  \sum_{(s,a), s \neq s_{\perp}}  \log \sbr{ \ex_{q,p,\pi^k} \mbr{  n^{K-1}(s,a) } }
\nonumber\\
\leq & 18 e^3 H^2 SN \cdot \beta \sbr{(K-1) \frac{m^{H+1}-m}{m-1},\delta'} \log \sbr{ H(K-1) }
\nonumber\\
\overset{\textup{(a)}}{\leq} & 18 e^3 H^2 SN \cdot \sbr{ \log \sbr{\frac{SN}{\delta'}} + S \log \sbr{8e Hm^{H+1} (K-1) } }  \log \sbr{ H(K-1) } ,
\end{align}
where (a) comes from $\beta(t,\kappa):=\log(SN/\kappa)+S\log(8e(t+1))$
and $\frac{m^{H+1}-m}{m-1}\leq Hm^{H+1}$.

According to the stopping rule, we have $\varepsilon \leq 4e \sqrt{ B^{k}_1(s) } + B^{k}_1(s)$ for $k=1,\dots,K-1$. 
Then, summing over $k=1,\dots,K-1$ for both sides, we obtain
\begin{align*}
(K-1) \varepsilon \leq & 4e \sum_{k=1}^{K-1} \sqrt{ B^{k}_1(s) } + \sum_{k=1}^{K-1} B^{k}_1(s) 
\\
\leq & 4e \sqrt{ (K-1) \sum_{k=1}^{K-1}    B^{k}_1(s) }  + \sum_{k=1}^{K-1} B^{k}_1(s) 
\end{align*}
and thus
\begin{align*}
(K-1) 
\leq & \frac{4e}{\varepsilon} \sqrt{ (K-1) \sum_{k=1}^{K-1}    B^{k}_1(s) }  + \frac{1}{\varepsilon} \sum_{k=1}^{K-1} B^{k}_1(s) 
\\
\leq & \frac{4e}{\varepsilon} \sqrt{(K-1) \cdot 18 e^3 H^2 SN \cdot \sbr{ \log \sbr{\frac{SN}{\delta'}} + S \log \sbr{8e Hm^{H+1} (K-1) } }  \log \sbr{ H(K-1) } }
\\& 
+\frac{18 e^3 H^2 SN }{\varepsilon} \cdot \sbr{ \log \sbr{\frac{SN}{\delta'}} + S \log \sbr{8e Hm^{H+1} (K-1) } }  \log \sbr{ H(K-1) }
\\
\leq & \frac{4e\sqrt{18 e^3 H^2 SN }}{\varepsilon} \sqrt{(K-1) \cdot  \log \sbr{\frac{SN}{\delta'}}  \log \sbr{ H(K-1) } + S \log\sbr{8e m^{H+1}} \log^2 \sbr{H(K-1)} }
\\& 
+\frac{18 e^3 H^2 SN }{\varepsilon} \cdot \sbr{ \log \sbr{\frac{SN}{\delta'}}  \log \sbr{ H(K-1) } + S \log\sbr{8e m^{H+1}} \log^2 \sbr{H(K-1)} } 
\end{align*}

Using Lemma 13 in \cite{menard2021fast} with $\tau=K-1$, $C=\frac{4e\sqrt{18 e^3 H^2 SN }}{\varepsilon}$, $A=\log \sbr{\frac{SN}{\delta'}}$, $\alpha= H$, $B=E=S \log\sbr{8e m^{H+1}}$ and $D=\frac{18 e^3 H^2 SN }{\varepsilon}$,
we obtain
\begin{align*}
K-1 = & \tilde{O} \sbr{C^2(A+B)C_1^2}
\\
= & O \sbr{ \frac{H^2 SN}{\varepsilon^2} \sbr{ \log \sbr{\frac{SN}{\delta}} + S \log\sbr{e\cdot m^{H}} } C_1^2},
\end{align*}
where 
$C_1=\log(\alpha(A+E)(C+D))$.

Thus, we have
\begin{align*}
K = O \sbr{ \frac{H^2 SN}{\varepsilon^2} \sbr{ \log \sbr{\frac{SN}{\delta}} + S \log\sbr{e\cdot m^{H}} } C_1^2} ,
\end{align*}
where 
\begin{align*}
	C_1=\log\sbr{ \sbr{\log \sbr{\frac{SN}{\delta}} + S \log\sbr{e\cdot m^{H}}} \cdot \frac{H SN }{\varepsilon} } .
\end{align*}

\end{proof}

\subsection{Sample Complexity Lower Bound} \label{apx:lb_rfe}
In this subsection, we prove the sample complexity lower bound (Theorem~\ref{thm:sample_complexity_lb}) for branching RL-RFE.

\begin{proof}[Proof of Theorem~\ref{thm:sample_complexity_lb}]
	This lower bound analysis follows the proof procedure of Theorem 2 in \cite{dann2015sample}.
	
	We consider the same instance as the proof of regret minimization lower bound in Section~\ref{apx:lb_regret}.
	
	The optimal policy $\pi_*$ is to take action $A_j(x_i)$ at state $x_i$, and we have
	\begin{align}
	 V^{*}_1(s_1) = m (\alpha+\eta) + m^2 (\alpha+\eta)^2 + \dots + m^H (\alpha+\eta)^H = H \label{eq:lb_rfe_V_star}
	\end{align}
	
	Fix a policy $\pi$. For each $i \in [S-3]$, let $G_i:=\{\pi(x_i)=A_j(x_i)\}$ denotes the event that policy $\pi$ chooses the optimal action $A_j(x_i)$ in state $x_i$. Then, we have
	\begin{align}
	V^{\pi}_1(s_1) = & m (\alpha+\eta) + \frac{1}{(S-3)} \sum_{i=1}^{S-3} \ex \Big[  \indicator{G_i} \sbr{ m^2 (\alpha+\eta)^2 + \dots + m^H (\alpha+\eta)^H } \nonumber\\& + \sbr{1-\indicator{G_i}} \sbr{ m^2 (\alpha+\eta)\alpha + m^3 (\alpha+\eta)\alpha (\alpha+\eta) + \dots + m^H (\alpha+\eta) \alpha (\alpha+\eta)^{H-2} } \Big]
	\nonumber\\
	= &  m (\alpha+\eta) + \frac{1}{(S-3)} \sum_{i=1}^{S-3} \ex \Big[  \indicator{G_i} \sbr{ m^2 (\alpha+\eta)^2 + \dots + m^H (\alpha+\eta)^H } \nonumber\\& + \sbr{1-\indicator{G_i}} \sbr{ m^2 (\alpha+\eta)\alpha + m^3 (\alpha+\eta)^2 \alpha + \dots + m^H (\alpha+\eta)^{H-1} \alpha  } \Big] \label{eq:lb_rfe_V_pi}
	\end{align}
	
	Subtracting Eq.~\eqref{eq:lb_rfe_V_pi} by Eq.~\eqref{eq:lb_rfe_V_star}, we have
	\begin{align*}
		V^{*}_1(s_1) - V^{\pi}_1(s_1) = & \frac{1}{(S-3)} \sum_{i=1}^{S-3} \ex \Big[ \sbr{1-\indicator{G_i}} \sbr{ m (\alpha+\eta) m\eta + m^2 (\alpha+\eta)^2 m\eta + \dots + m^{H-1} (\alpha+\eta)^{H-1} m\eta  } \Big]
		\\
		= & \frac{1}{(S-3)} \sum_{i=1}^{S-3} \ex \Big[ \sbr{1-\indicator{G_i}}  m\eta (H-1) \Big]
		\\
		= &  m\eta (H-1) \sbr{1-\frac{1}{(S-3)} \sum_{i=1}^{S-3} \ex \Big[\indicator{G_i}\Big]}  
	\end{align*}
	The following analysis follows the proof procedure of Theorem 2 in \cite{dann2015sample}.
	For $\pi$ to be $\varepsilon$-optimal, we need
	\begin{align*}
		\Pr \mbr{ m\eta (H-1) \sbr{1-\frac{1}{(S-3)} \sum_{i=1}^{S-3} \ex \Big[\indicator{G_i}\Big]} \leq \varepsilon }\geq 1-\delta
	\end{align*}
	Let $\eta=\frac{8e^2\varepsilon}{cm(H-1)}$, where $c$ is an absolute constant that we specify later. Then, we have
	\begin{align*}
	\Pr \mbr{ \frac{1}{(S-3)} \sum_{i=1}^{S-3} \ex \mbr {\indicator{G_i}} \leq 1-\frac{c}{8e^4} } \geq 1-\delta
	\end{align*}
	Using Markov's inequality, we have
	\begin{align*}
	1-\delta \leq \Pr \mbr{ \frac{1}{(S-3)} \sum_{i=1}^{S-3} \ex \mbr {\indicator{G_i}} \leq 1-\frac{c}{8e^4} } \leq \frac{1}{(S-3)\sbr{1-\frac{c}{8e^4}}} \sum_{i=1}^{S-3} \Pr \mbr{G_i}
	\end{align*}
	
	Since all $G_i$ ($i \in [S-3]$) are independent of each other, there exist $\{\delta_i\}_{i \in [S-3]}$ such that $\Pr \mbr{\bar{G}_i} \leq \delta_i$ and 
	\begin{align*}
	\frac{1}{(S-3)\sbr{1-\frac{c}{8e^4}}} \sum_{i=1}^{S-3} \sbr{1 - \delta_i } \geq 1-\delta ,
	\end{align*}
	which is equivalent to
	\begin{align*}
	\sum_{i=1}^{S-3} \delta_i \leq (S-3) \sbr{1 + \delta \sbr{1-\frac{c}{8e^4}} - \sbr{1-\frac{c}{8e^4}}} .
	\end{align*}
	
	Let $\varepsilon$ be small enough such that $\eta=\frac{8e^2\varepsilon}{cm(H-1)} \leq \frac{1}{4}$, and let $\delta$ to be small enough $\delta \leq \frac{c}{8e^4}$.
	Since all $G_i$ are independent, we can use Theorem 1 in \cite{mannor2004sample} to obtain
	\begin{align*}
		\delta_i \leq & \frac{1}{c} \sbr{1 + \delta \sbr{1-\frac{c}{8e^4}} - \sbr{1-\frac{c}{8e^4}}}
		\\
		\leq & \frac{1}{c} \sbr{1 + \delta - \sbr{1-\frac{c}{8e^4}}}
		\\
		= & \frac{\delta}{c} + \frac{1}{8e^4}
		\\
		\leq & \frac{2}{8e^4}
	\end{align*}
	
	Let $n_i$ denote the number of observations on state $x_i$. The KL-divergence of the trigger distribution on state $x_i$ between our constructed instance and the uniform instance is $m \cdot \kl\sbr{\bernoulli(\alpha) || \bernoulli(\alpha+\eta)} \leq m \cdot  \frac{\eta^2}{(\alpha+\eta)(1-(\alpha+\eta))} = O(m \eta^2)$ with $(\alpha+\eta)(1-(\alpha+\eta))\geq c_1$ for some absolute positive constant $c_1$.
	Then, to ensure $\Pr \mbr{\bar{G}_i} \leq \delta_i$, we need
	\begin{align}
		\ex \mbr{n_i} \geq  \frac{c_1 d}{ m \eta^2} \log\sbr{\frac{c_2}{\delta_i}} \cdot \indicator{c \delta_i \leq \sbr{1 + \delta - \sbr{1-\frac{c}{8e^4}}}}, \label{eq:lb_rfe_n_i}
	\end{align}
	where $c_1$ and $c_2$ are appropriate absolute constant, e.g., $c_1=400$ and $c_2=4$.
	
	In the following, we compute the worst bound over all $\delta_1,\dots,\delta_{S-3}$ to ensure that $\sum_{i=1}^{S-3} \delta_i \leq (S-3) \sbr{1 + \delta \sbr{1-\frac{c}{8e^4}} - \sbr{1-\frac{c}{8e^4}}}$.
	\begin{align}
		 \min_{\delta_1,\dots,\delta_{S-3}} & \sum_{i=1}^{S-3} \log\sbr{\frac{1}{\delta_i}} \cdot \indicator{c \delta_i \leq \sbr{1 + \delta - \sbr{1-\frac{c}{8e^4}}}} \label{eq:lb_rfe_opt}
		\nonumber\\
		s.t. & \sum_{i=1}^{S-3} \delta_i \leq (S-3) \sbr{1 + \delta \sbr{1-\frac{c}{8e^4}} - \sbr{1-\frac{c}{8e^4}}}
	\end{align}
	
	Using Lemma D.1 in \cite{dann2015sample}, the optimal solution of this optimization is $\delta_1=\dots=\delta_{S-3}=z$, if $c \sbr{1-\log z} \leq 1$ with $z=1 + \delta \sbr{1-\frac{c}{8e^4}} - \sbr{1-\frac{c}{8e^4}}$.
	
	Since $z \geq 1 - \sbr{1-\frac{c}{8e^4}}=\frac{c}{8e^4}$ and $c \sbr{1-\log z}$ is decreasing wi respect to $z$, we can obtain a sufficient condition for $c \sbr{1-\log z} \leq 1$ as 
	\begin{align*}
		c \sbr{1-\log \sbr{ \frac{c}{8e^4} } } \leq 1
	\end{align*}
	
	Let $c=\frac{1}{10}$, which satisfies this condition.
	Thus, $\delta_1=\dots=\delta_{S-3}=z$ is the optimal solution to Eq.~\eqref{eq:lb_rfe_opt}.
	
	Since in each episode, we only observe a single state $x_i$, the number of required episodes is at least
	\begin{align*}
		K \geq & \sum_{i=1}^{S-3} \ex \mbr{n_i} 
		\\
		\geq & \frac{c_1 d (S-3)}{m \eta^2} \log\sbr{\frac{c_2}{1 + \delta \sbr{1-\frac{c}{8e^4}} - \sbr{1-\frac{c}{8e^4}}}} 
		\\
		\geq & \frac{c_1 c^2 d (S-3) m (H-1)^2 }{64 e^4 \varepsilon^2} \log\sbr{ \frac{c_2}{ \delta \sbr{1-\frac{c}{8e^4}} + \frac{c}{8e^4} } }  
		\\
		= & \Omega \sbr{ \frac{ S N H^2 }{\varepsilon^2} \log\sbr{ \frac{1}{ \delta } } } 
	\end{align*}
\end{proof}

%% file: tree-structured_RL_main.bbl
\begin{thebibliography}{22}
\providecommand{\natexlab}[1]{#1}
\providecommand{\url}[1]{\texttt{#1}}
\expandafter\ifx\csname urlstyle\endcsname\relax
  \providecommand{\doi}[1]{doi: #1}\else
  \providecommand{\doi}{doi: \begingroup \urlstyle{rm}\Url}\fi

\bibitem[Agrawal \& Jia(2017)Agrawal and Jia]{agrawal2017posterior}
Agrawal, S. and Jia, R.
\newblock Posterior sampling for reinforcement learning: worst-case regret
  bounds.
\newblock In \emph{Advances in Neural Information Processing Systems}, pp.\
  1184--1194, 2017.

\bibitem[Auer et~al.(2002)Auer, Cesa-Bianchi, Freund, and
  Schapire]{auer2002nonstochastic}
Auer, P., Cesa-Bianchi, N., Freund, Y., and Schapire, R.~E.
\newblock The nonstochastic multiarmed bandit problem.
\newblock \emph{SIAM Journal on Computing}, 32\penalty0 (1):\penalty0 48--77,
  2002.

\bibitem[Azar et~al.(2017)Azar, Osband, and Munos]{azar2017minimax}
Azar, M.~G., Osband, I., and Munos, R.
\newblock Minimax regret bounds for reinforcement learning.
\newblock In \emph{International Conference on Machine Learning}, pp.\
  263--272. PMLR, 2017.

\bibitem[Burnetas \& Katehakis(1997)Burnetas and
  Katehakis]{burnetas1997optimal}
Burnetas, A.~N. and Katehakis, M.~N.
\newblock Optimal adaptive policies for markov decision processes.
\newblock \emph{Mathematics of Operations Research}, 22\penalty0 (1):\penalty0
  222--255, 1997.

\bibitem[Dann \& Brunskill(2015)Dann and Brunskill]{dann2015sample}
Dann, C. and Brunskill, E.
\newblock Sample complexity of episodic fixed-horizon reinforcement learning.
\newblock In \emph{Advances in Neural Information Processing Systems}, pp.\
  2818--2826, 2015.

\bibitem[Dann et~al.(2017)Dann, Lattimore, and Brunskill]{dann2017unifying}
Dann, C., Lattimore, T., and Brunskill, E.
\newblock Unifying {PAC} and regret: uniform {PAC} bounds for episodic
  reinforcement learning.
\newblock In \emph{Advances in Neural Information Processing Systems}, pp.\
  5717--5727, 2017.

\bibitem[Domingues et~al.(2021)Domingues, M{\'e}nard, Kaufmann, and
  Valko]{domingues2021episodic}
Domingues, O.~D., M{\'e}nard, P., Kaufmann, E., and Valko, M.
\newblock Episodic reinforcement learning in finite {MDP}s: Minimax lower
  bounds revisited.
\newblock In \emph{Algorithmic Learning Theory}, pp.\  578--598. PMLR, 2021.

\bibitem[Fu et~al.(2021)Fu, Agrawal, Irissappane, Zhang, Huang, and
  Qu]{drl_recommendation}
Fu, M., Agrawal, A., Irissappane, A.~A., Zhang, J., Huang, L., and Qu, H.
\newblock Deep reinforcement learning framework for category-based item
  recommendation.
\newblock \emph{IEEE Transactions on Cybernetics}, 2021.

\bibitem[Jaksch et~al.(2010)Jaksch, Ortner, and Auer]{jaksch2010near}
Jaksch, T., Ortner, R., and Auer, P.
\newblock Near-optimal regret bounds for reinforcement learning.
\newblock \emph{Journal of Machine Learning Research}, 11\penalty0 (4), 2010.

\bibitem[Jin et~al.(2018)Jin, Allen-Zhu, Bubeck, and Jordan]{jin2018q}
Jin, C., Allen-Zhu, Z., Bubeck, S., and Jordan, M.~I.
\newblock Is {Q}-learning provably efficient?
\newblock In \emph{Advances in Neural Information Processing Systems}, pp.\
  4868--4878, 2018.

\bibitem[Jin et~al.(2020{\natexlab{a}})Jin, Krishnamurthy, Simchowitz, and
  Yu]{jin2020reward}
Jin, C., Krishnamurthy, A., Simchowitz, M., and Yu, T.
\newblock Reward-free exploration for reinforcement learning.
\newblock In \emph{International Conference on Machine Learning}, pp.\
  4870--4879. PMLR, 2020{\natexlab{a}}.

\bibitem[Jin et~al.(2020{\natexlab{b}})Jin, Yang, Wang, and
  Jordan]{jin2020provably}
Jin, C., Yang, Z., Wang, Z., and Jordan, M.~I.
\newblock Provably efficient reinforcement learning with linear function
  approximation.
\newblock In \emph{Conference on Learning Theory}, pp.\  2137--2143. PMLR,
  2020{\natexlab{b}}.

\bibitem[Kang et~al.(2020)Kang, Jeong, and Chung]{tree_based_recommendation}
Kang, S., Jeong, C., and Chung, K.
\newblock Tree-based real-time advertisement recommendation system in online
  broadcasting.
\newblock \emph{IEEE Access}, 8:\penalty0 192693--192702, 2020.

\bibitem[Kaufmann et~al.(2021)Kaufmann, M{\'e}nard, Domingues, Jonsson,
  Leurent, and Valko]{kaufmann2021adaptive}
Kaufmann, E., M{\'e}nard, P., Domingues, O.~D., Jonsson, A., Leurent, E., and
  Valko, M.
\newblock Adaptive reward-free exploration.
\newblock In \emph{International Conference on Algorithmic Learning Theory},
  pp.\  865--891. PMLR, 2021.

\bibitem[Mannor \& Tsitsiklis(2004)Mannor and Tsitsiklis]{mannor2004sample}
Mannor, S. and Tsitsiklis, J.~N.
\newblock The sample complexity of exploration in the multi-armed bandit
  problem.
\newblock \emph{Journal of Machine Learning Research}, 5:\penalty0 623--648,
  2004.

\bibitem[M{\'e}nard et~al.(2021)M{\'e}nard, Domingues, Jonsson, Kaufmann,
  Leurent, and Valko]{menard2021fast}
M{\'e}nard, P., Domingues, O.~D., Jonsson, A., Kaufmann, E., Leurent, E., and
  Valko, M.
\newblock Fast active learning for pure exploration in reinforcement learning.
\newblock In \emph{International Conference on Machine Learning}, pp.\
  7599--7608. PMLR, 2021.

\bibitem[Osband \& Van~Roy(2016)Osband and Van~Roy]{osband2016lower}
Osband, I. and Van~Roy, B.
\newblock On lower bounds for regret in reinforcement learning.
\newblock \emph{arXiv preprint arXiv:1608.02732}, 2016.

\bibitem[Sutton \& Barto(2018)Sutton and Barto]{sutton2018reinforcement}
Sutton, R.~S. and Barto, A.~G.
\newblock \emph{Reinforcement learning: An introduction}.
\newblock MIT press, 2018.

\bibitem[Zanette \& Brunskill(2019)Zanette and Brunskill]{zanette2019tighter}
Zanette, A. and Brunskill, E.
\newblock Tighter problem-dependent regret bounds in reinforcement learning
  without domain knowledge using value function bounds.
\newblock In \emph{International Conference on Machine Learning}, pp.\
  7304--7312. PMLR, 2019.

\bibitem[Zhang et~al.(2020)Zhang, Zhou, and Ji]{zhang2020almost}
Zhang, Z., Zhou, Y., and Ji, X.
\newblock Almost optimal model-free reinforcement learning via
  reference-advantage decomposition.
\newblock \emph{Advances in Neural Information Processing Systems}, 2020.

\bibitem[Zhang et~al.(2021)Zhang, Du, and Ji]{zhang2021nearly}
Zhang, Z., Du, S.~S., and Ji, X.
\newblock Nearly minimax optimal reward-free reinforcement learning.
\newblock \emph{International Conference on Machine Learning}, 2021.

\bibitem[Zhou et~al.(2021)Zhou, He, and Gu]{zhou2021provably}
Zhou, D., He, J., and Gu, Q.
\newblock Provably efficient reinforcement learning for discounted {MDP}s with
  feature mapping.
\newblock In \emph{International Conference on Machine Learning}, pp.\
  12793--12802. PMLR, 2021.

\end{thebibliography}
